\documentclass{article}
\usepackage{PRIMEarxiv}

\usepackage[utf8]{inputenc} 
\usepackage[T1]{fontenc}    
\usepackage{hyperref}       
\usepackage{nicefrac}       
\usepackage{microtype}      
\usepackage{lipsum}
\usepackage{fancyhdr}

\usepackage{graphicx}       
\graphicspath{{media/}}     

\usepackage{times}  
\usepackage{helvet}  
\usepackage{courier}  
\usepackage{natbib}  
\usepackage{caption} 
\usepackage{algorithm}
\usepackage[noend]{algpseudocode}
\usepackage{float}

\floatstyle{ruled}
\restylefloat{algorithm}

\usepackage{newfloat}
\usepackage{listings}
\DeclareCaptionStyle{ruled}{labelfont=normalfont,labelsep=colon,strut=off} 
\floatstyle{ruled}
\newfloat{listing}{tb}{lst}{}
\floatname{listing}{Listing}
\usepackage{booktabs}       
\usepackage{amsfonts}       
\usepackage{nicefrac}       
\usepackage{microtype}      
\usepackage{rotating}
\usepackage{tabularx}
\usepackage{pdflscape}
\usepackage{amsmath,amsthm,amssymb, bm}
\usepackage[capitalize,noabbrev]{cleveref}
\usepackage{blkarray}  

\usepackage{amsmath}

\usepackage{subcaption} 
\usepackage{dsfont}
\usepackage{float}
\usepackage{cancel}
\usepackage{enumerate, cases}
\usepackage{thmtools,thm-restate}
\usepackage{mathtools}
\usepackage[textsize=tiny]{todonotes}
\usepackage[none]{hyphenat}
\usepackage{multirow}
\usepackage{makecell}

\usepackage{dblfloatfix}

\usepackage{array}
\usepackage{enumitem}
\usepackage{tikz}
\usetikzlibrary{decorations.pathreplacing}
\usepackage{pgfkeys}
\usetikzlibrary{intersections}
\usepackage[textsize=tiny]{todonotes}
\usepackage{etoolbox} 
\usepackage{pdfpages}

\DeclareMathOperator*{\argmin}{arg\,min}
\DeclareMathOperator*{\argmax}{arg\,max}
\DeclarePairedDelimiter\abs{\lvert}{\rvert}%
\makeatletter
\let\oldabs\abs
\def\abs{\@ifstar{\oldabs}{\oldabs*}}

\usepackage{pifont}

\newcommand{\norm}[1]{\left\lVert #1 \right\rVert} 
  
\theoremstyle{plain}
\newtheorem{thm}{Theorem}
\newtheorem{lem}{Lemma}
\newtheorem{prop}{Proposition}
\newtheorem{cor}{Corollary}

\newtheorem{rem}{Remark}
\newtheorem{defi}{Definition}
\newtheorem{assu}{Assumption}

\newcounter{keylemma}

\newtheorem{keylem}[keylemma]{Lemma}

\newcounter{auxlemma}

\newcounter{techlemma}

\newtheorem{techlem}[techlemma]{Lemma}

\usepackage{newfloat}
\usepackage{listings}
\DeclareCaptionStyle{ruled}{labelfont=normalfont,labelsep=colon,strut=off} 
\floatstyle{ruled}
\newfloat{listing}{tb}{lst}{}
\floatname{listing}{Listing}
\newcommand{\Algo}{{Robust Fitted Learning with $\phi$-Divergence Uncertainty Set (RFL-$\phi$)}}
\newcommand{\Algoname}{RFL-$\phi$}

\newcommand{\AlgonameTV}{RFL-\text{TV}}

\newcommand{\RMDPf}{RMDP-$\phi$}
\newcommand{\RMDPchi}{$\chi^2$-RMDP}
\newcommand{\RMDPKL}{\text{KL}-RMDP}
\newcommand{\RMDPTV}{\text{TV}-RMDP}

\newcommand{\dimBE}{\operatorname{dim}_{\mathrm{BE}}}
\newcommand{\dimDE}{\operatorname{dim}_{\mathrm{DE}}}
\newcommand{\dimE}{\operatorname{dim}_{\mathrm{E}}}  
\newcommand{\dimrobBE}{\dimBE^{\mathrm{rob}}}

\newcommand{\bbR}{\mathbb R}
\newcommand{\bbE}{\mathbb E}
\newcommand{\bbN}{\mathbb N}

\newcommand{\inner}[2]{\left\langle #1,\, #2 \right\rangle}

\title{Online Robust Reinforcement Learning with General Function Approximation}

\author{Debamita Ghosh$^1$, George K. Atia$^{1,2}$, Yue Wang$^{1,2}$ \\
$^1$ Department of Electrical and Computer Engineering, University of Central Florida, Orlando, FL 32816, USA \\
$^2$ Department of Computer Science, University of Central Florida, Orlando, FL 32816, USA 
}

\begin{document}

\maketitle

\begin{abstract}
In many real-world settings, reinforcement learning systems suffer performance degradation when the environment encountered at deployment differs from that observed during training. Distributionally robust reinforcement learning (DR-RL) mitigates this issue by seeking policies that maximize performance under the most adverse transition dynamics within a prescribed uncertainty set. Most existing DR-RL approaches, however, rely on strong data availability assumptions—such as access to a generative model or large offline datasets—and are largely restricted to tabular settings.

In this work, we propose a fully online DR-RL algorithm with general function approximation that learns robust policies solely through interaction, without requiring prior knowledge or pre-collected data. Our approach is based on a dual-driven fitted robust Bellman procedure that simultaneously estimates the value function and the corresponding worst-case backup operator. We establish regret guarantees for online DR-RL characterized by an intrinsic complexity notion—the robust Bellman–Eluder dimension covering a broad class of $\phi$-divergence uncertainty sets. The resulting regret bounds are sublinear, do not scale with the size of the state or action spaces, and specialize to tight rates in structured problem classes, demonstrating the practicality and scalability of our framework.
\end{abstract}

\section{Introduction}
\label{sec:Introduction}
Reinforcement learning (RL) provides a general framework for sequential decision-making under uncertainty, enabling agents to learn effective policies through interaction with an environment. In the \emph{online} learning paradigm, an agent improves its behavior via trial-and-error interactions without relying on pre-collected datasets or access to high-fidelity simulators. This setting has underpinned a wide range of successes in complex domains, including strategic games and simulators \cite{silver2016mastering,zha2021douzero,berner2019dota,vinyals2017starcraft}, as well as more recent applications in generative and decision-driven AI systems \cite{NeurIPS2022_TrainingLLM_Ouyang,cao2023reinforcement,black2023training,uehara2024understanding,zhang2024large,du2023guiding,cao2024survey}.

Despite these advances, a central limitation of conventional online RL lies in its implicit reliance on the assumption that the environment dynamics remain stable between training and deployment. In real-world systems, this assumption is frequently violated due to non-stationarity, unmodeled disturbances, or structural mismatch between simulation and reality. Policies optimized purely for expected return under the training dynamics may therefore exhibit brittle behavior when exposed to even mild shifts, leading to significant performance degradation or unsafe decisions. Such failures have been documented in safety-critical applications including autonomous driving \cite{IEEE2021_DRLSurvey_Kiran} and healthcare \cite{ACM2018_SupervisedRLRNN_Wang}, where unanticipated changes—such as weather-induced friction loss or population-level distributional drift—can fundamentally alter system dynamics.

Distributionally robust reinforcement learning (DR-RL) offers a principled approach to addressing this challenge by explicitly optimizing for worst-case performance over a specified uncertainty set of transition models \cite{INFORM2005_RobusDP_Iyengar,PMLR2017_RobustAdvRL_Pinto,Arxiv2022_SimToRealTransferContinuousPartialObs_Hu}. Rather than assuming perfect model fidelity, DR-RL seeks policies that perform reliably under adverse but plausible perturbations of the dynamics, thereby improving robustness to deployment-time mismatch. This formulation has been shown to yield policies that are inherently more resilient to environmental variation, improving safety and reliability across a range of domains \cite{goodfellow2014explaining,vinitsky2020robust,abdullah2019wasserstein,hou2020robust,rajeswaran2017epopt,atkeson2003nonparametric,morimoto2005robust,huang2017adversarial,kos2017delving,lin2017tactics,pattanaik2018robust,mandlekar2017adversarially}.

In \emph{online} DR-RL \cite{ICML2025_OnlineDRMDPSampleComplexity_He,Arxiv2024_UpperLowerDRRL_Liu,PMLR2024_LinearFunctionDRMDP_Liu,Arxiv2024_DRORLwithInteractiveData_Lu,ghosh2025provably}, the agent must simultaneously learn from interaction with an unknown nominal environment while optimizing a worst-case objective defined over an uncertainty set centered at the interaction dynamics. This setting preserves the flexibility of online learning while explicitly accounting for robustness, but introduces substantial theoretical and algorithmic challenges absent in standard RL.

A first difficulty arises from the inherent \emph{off-dynamics} nature of the objective: data are collected under the nominal transition kernel, whereas policy evaluation and optimization are performed with respect to worst-case dynamics that may differ significantly from those encountered during interaction. This mismatch leads to an off-dynamics learning problem \cite{Arxiv2020_OffDynRL_Eysenbach,NeurIPS2024_MinimaxOptimalOfflineRL_Liu,holla2021off}, and can induce severe information bottlenecks—transitions that are critical under the worst-case model may be rarely, or never, observed under nominal exploration \cite{Arxiv2024_DRORLwithInteractiveData_Lu,ghosh2025provably}. A second challenge stems from the interactive nature of learning: because exploration occurs in the real environment, unsafe or overly aggressive exploration strategies are often unacceptable. The learner must therefore maintain satisfactory performance throughout training, including under worst-case dynamics, rather than merely converging asymptotically.

Due to these difficulties, most existing DR-RL methods rely on additional data access assumptions, such as a generative model capable of producing arbitrary samples \cite{PMLR2022_SampleComplexityRORL_Panaganti,PMLR2023_ImprovedSampleComplexityDRRL_Xu,NeuRIPS2023_CuriousPriceDRRLGenerativeMdel_Shi}, extensive offline datasets with strong coverage guarantees \cite{NeuRIPS2023_DoublePessimismDROfflineRL_Blanchet,JMLR2024_DROfflineRLNearOptimalSampelComplexity_Shi,Arxiv2024_OfflineRobustRLFDivg_Tang,NeuRIPS2024_UnifiedPessimismOfflineRL_Yue,NeurIPS2024_MinimaxOptimalOfflineRL_Liu,NeurIPS2022_RobustRLOffline_Panaganti,Arxiv2024_SampleComplexityOfflineLinearDRMDP_Wang}, or hybrid regimes that combine offline data with limited online interaction \cite{Arxiv2024_ModelFreeRobustRL_Panaganti}. In many practical applications, however, such assumptions are unrealistic, making the development of \emph{purely online} DR-RL methods essential.

Scalability presents an additional obstacle. While function approximation has enabled standard RL to scale to large or continuous state–action spaces \cite{mnih2013playing,silver2016mastering,Sage2013_RLSurvey_Kober,li2016deep}, extending these ideas to DR-RL is nontrivial. Robust value functions are defined with respect to worst-case dynamics that differ from the nominal data-generating process, and may fail to admit low-dimensional approximations even when the nominal value function does \cite{tamar2014scaling}. As a result, existing DR-RL approaches with function approximation typically impose strong structural assumptions—such as small discount factors \cite{xu2010distributionally,zhou2024natural,badrinath_robust_2021} or linear-MDP models \cite{Arxiv2022_OfflineDRRLLinearFunctionApprox_Ma,PMLR2024_LinearFunctionDRMDP_Liu,NeurIPS2024_MinimaxOptimalOfflineRL_Liu,Arxiv2024_UpperLowerDRRL_Liu,Arxiv2024_SampleComplexityOfflineLinearDRMDP_Wang}—that limit their applicability.

These considerations motivate the following question:
\begin{quote}
\emph{Can we design a sample-efficient, purely online DR-RL algorithm that scales to large problems while providing rigorous performance guarantees?}
\end{quote}

In this paper, we answer this question affirmatively by developing an online DR-RL framework with general function approximation and establishing its theoretical guarantees. Our main contributions are summarized below.

\paragraph{Sample-efficient purely online DR-RL with general function approximation.}
We propose {\Algoname}, a purely online DR-RL algorithm operating under $\phi$-divergence uncertainty sets with general function approximation. The method incorporates optimism into a fitted-learning framework through a functional reformulation of the robust Bellman operator. Rather than relying on per-state–action bonuses, {\Algoname} constructs a global uncertainty quantifier over the function class, enabling efficient exploration and scalability beyond tabular and offline or hybrid regimes.

\paragraph{Robust Bellman--Eluder dimension as an intrinsic complexity measure.}
We introduce the robust Bellman--Eluder (BE) dimension as the intrinsic notion governing learnability in online DR-RL with function approximation. Defined via the distributional Eluder dimension of the robust Bellman residual class under on-policy distributions, this measure captures the statistical complexity of robust value learning without requiring coverage or concentrability assumptions. We show that it is finite for broad problem classes and plays an analogous role to the BE dimension in non-robust RL \cite{NeurIPS2021_BellmanEluderDim_Jin}.

\paragraph{Dual-driven fitted learning and global confidence sets.}
Our approach leverages a dual formulation of the robust Bellman operator to construct global confidence sets over value functions. By optimizing a single least-squares objective on the dual residual, {\Algoname} simultaneously approximates the worst-case backup operator and quantifies uncertainty for exploration. This mechanism is fundamentally different from offline DR-RL methods, where the dual variables do not influence data collection.

\paragraph{Sharp regret guarantees governed by robust BE dimension.}
We establish regret bounds for purely online DR-RL with general function approximation that depend solely on the intrinsic robust BE dimension. The resulting bounds are sublinear in the number of episodes, independent of the sizes of the state and action spaces, and recover near-optimal rates in structured settings such as tabular and linear RMDPs.

\section{Related Work}
\label{sec:related_work}
We review prior work most closely related to ours, focusing on distributionally robust reinforcement learning (DR-RL) and its connections to non-robust RL with function approximation.

\textbf{Tabular Distributionally Robust Reinforcement Learning.}
The majority of existing DR-RL literature focuses on the tabular setting. A substantial body of work establishes minimax-optimal or near-optimal sample complexity guarantees under strong data-access assumptions, most commonly in the generative-model regime \cite{Arxiv2023_TowardsMinimaxOptimalityRobustRL_Clavier, liu2022distributionally, PMLR2022_SampleComplexityRORL_Panaganti, ramesh2023distributionallyrobustmodelbasedreinforcement, NeuRIPS2023_CuriousPriceDRRLGenerativeMdel_Shi, PMLR2023_SampleComplexityDRQLearning_Wang, wang2024foundationdistributionallyrobustreinforcement, JMLR2024_SampleComplexityVarianceReducedDRQLearning_Wang, PMLR2023_ImprovedSampleComplexityDRRL_Xu, AnnalsStat2022_TheoreticalUnderstandingRMDP_Yang, yang2023robust, badrinath_robust_2021, li2022first, liang2023single}. Related lines of work study offline DR-RL, where learning is performed from large datasets assumed to provide sufficient coverage of relevant dynamics \cite{NeuRIPS2023_DoublePessimismDROfflineRL_Blanchet, JMLR2024_DROfflineRLNearOptimalSampelComplexity_Shi, Arxiv2023_SoftRMDPRiskSensitive_Zhang, NeurIPS2024_MinimaxOptimalOfflineRL_Liu, NeuRIPS2024_UnifiedPessimismOfflineRL_Yue, Arxiv2024_SampleComplexityOfflineLinearDRMDP_Wang}.  

More recently, a small number of works have begun to address \emph{online} DR-RL, where the agent learns through direct interaction with the environment \cite{dong2022online, NeurIPS2021_OnlineRobustRLModelUncertainty_Wang, Arxiv2024_DRORLwithInteractiveData_Lu, ICML2025_OnlineDRMDPSampleComplexity_He, ghosh2025provably}. These approaches typically overcome the information bottleneck induced by worst-case dynamics via additional structural assumptions, such as bounded visitation ratios or specialized model access. However, they are largely confined to tabular or model-based/value-based frameworks and do not naturally extend to large or continuous state–action spaces.

\textbf{Distributionally Robust RL with Function Approximation.}
Theoretical treatments of DR-RL with function approximation have primarily focused on linear representations. A key challenge is that common function classes are not closed under robust Bellman operators, making approximation error difficult to control in general. To address this, prior work often imposes strong structural assumptions on the underlying RMDP, such as small discount factors \cite{xu2010distributionally,tamar2014scaling,zhou2024natural} or explicitly linear transition models \cite{Arxiv2022_OfflineDRRLLinearFunctionApprox_Ma,PMLR2024_LinearFunctionDRMDP_Liu,NeurIPS2024_MinimaxOptimalOfflineRL_Liu,Arxiv2024_UpperLowerDRRL_Liu,Arxiv2024_SampleComplexityOfflineLinearDRMDP_Wang}.  

More general function approximation in DR-RL has been explored in a limited number of works \cite{NeurIPS2022_RobustRLOffline_Panaganti,Arxiv2024_ModelFreeRobustRL_Panaganti}, which leverage functional optimization to compute robust backups. These methods, however, operate in offline or hybrid settings and rely on global coverage assumptions to avoid exploration challenges. In addition, \cite{Arxiv2024_ModelFreeRobustRL_Panaganti} studies regularized RMDPs, which differ fundamentally from the worst-case DR-RL objective considered here. In contrast, our work addresses fully online learning with general function approximation, without assuming linear structure, discounting constraints, or offline coverage.

\textbf{Non-Robust RL with Function Approximation.}
Function approximation has been extensively studied in standard (non-robust) reinforcement learning. While much of this literature focuses on offline RL with general function classes \cite{zhan2022offline, jiang2024offline, wang2020statistical}, our work is situated in the online setting, where the learner must explore while learning from interaction.

A central theme in online RL theory is the identification of structural complexity measures that characterize when learning with function approximation is statistically and computationally feasible. The Eluder dimension \cite{li2022understanding, russo2013eluder} was introduced to quantify the sequential complexity of a hypothesis class and underlies optimistic exploration strategies in bandits and RL \cite{wang_reinforcement_2020}. However, this notion captures only properties of the function class itself, without accounting for its interaction with the environment dynamics.

Subsequent work proposed complexity measures that explicitly incorporate the Bellman operator. The Bellman rank \cite{jiang2017contextual} and Witness rank \cite{sun2019model} formalize this interaction and were later unified under the Bellman--Eluder (BE) dimension framework \cite{NeurIPS2021_BellmanEluderDim_Jin}. The BE dimension measures the effective complexity of approximating Bellman residuals and provides a unifying explanation for tractability across many structured RL models.

A complementary line of research emphasizes coverage-based characterizations of learnability. In particular, \cite{xie2022role} introduced the notion of coverability and showed it to be necessary and sufficient for efficient exploration under function approximation, subsuming earlier concentrability-style assumptions. At the same time, hardness results \cite{foster2021statistical, du2021bilinear} demonstrate that without appropriate structural or complexity assumptions, online RL with rich observations can be exponentially difficult.

\textbf{Connections and Distinctions.}
Our work operates firmly within the online learning paradigm but departs from existing approaches in a fundamental way. Rather than relying on coverage coefficients or visitation ratio bounds, we use the \emph{robust} Bellman--Eluder dimension as the primary complexity measure governing learnability. This choice is deliberate: the BE dimension is directly tied to Bellman residuals and value-function approximation, making it naturally compatible with robust Bellman operators and amenable to extension under model uncertainty.

Importantly, guarantees developed for non-robust RL do not transfer directly to the distributionally robust setting. DR-RL replaces a single transition kernel with an uncertainty set and optimizes against a worst-case Bellman operator, introducing new technical challenges. These include the inherent nonlinearity of robust Bellman updates, the need to jointly control statistical estimation error and adversarial model shift, and the mismatch between nominal data collection and worst-case evaluation. Addressing these issues requires new analytical tools beyond those used in standard RL.

Our work introduces such tools by combining dual-based robust Bellman residuals with optimism-driven fitted learning, and by formalizing the robust BE dimension as the intrinsic measure of complexity in online DR-RL. This yields an algorithmic and theoretical framework that is genuinely distinct from, and not reducible to, existing results in either non-robust RL or offline/hybrid DR-RL.

\section{Preliminaries and Problem Formulation}
\label{sec:Problem_Setup}
\subsection{Distributionally Robust Markov Decision Process}
DR-RL can be formulated as an episodic finite-horizon distributionally robust Markov decision process (RMDP) \cite{INFORM2005_RobusDP_Iyengar}, represented by $\mathcal{M}:=(\mathcal{S},\mathcal{A},H, \mathcal{P}, r)$, where the set $\mathcal{S}=\{1,\dots,S\}$ is the finite state space, $\mathcal{A}=\{1,\dots,A\}$ is the finite action space, $H$ is the horizon length,
$r = \{r_h: \mathcal{S}\times \mathcal{A}\rightarrow [0, 1]\}_{h=1}^H$ is the collection of reward functions, and $\mathcal{P}=\{\mathcal{P}_h\}_{h=1}^H$ is an uncertainty set of transition kernels. At step $h$, the agent is at state $s_h$ and takes an action $a_h$, receives the reward $r_h(s_h,a_h)$, and is transited to the next state $s_{h+1}$ following an arbitrary transition kernel $P_h(\cdot|s_h,a_h)\in\mathcal{P}_h$. 

We consider the standard $(s,a)$-rectangular uncertainty set with divergence ball-structure \cite{INFORM2013_RMDP_Wiesemann}. Specifically, let  $P^{\star} = \{P^{\star}_h\}_{h=1}^H$ be the \textit{nominal} transition kernel, where each $P^\star_h: \mathcal{S}\times \mathcal{A}\rightarrow \Delta(\mathcal{S})$\footnote{$\Delta(\cdot)$ denotes the probability simplex over the space.}. The uncertainty set, centered around $P^{\star}$, is defined as $\mathcal{P}=\mathcal{U}^{\phi,\sigma}(P^{\star})= \bigotimes_{(h,s,a)\in [H]\times \mathcal{S} \times \mathcal{A}}\mathcal{U}^{\phi,\sigma}_h(s,a)$, and $\mathcal{U}^{\phi,\sigma}_h(s,a)\triangleq\big\{P\in\Delta(\mathcal{S}): D(P,P^\star_h(\cdot|s,a))\leq \sigma\big\}$, containing all the transition kernels that differ from $P^\star$ up to some uncertainty level $\sigma \geq 0$, under some  probability divergence functions $D$
 \cite{INFORM2005_RobusDP_Iyengar, PMLR2022_SampleComplexityRORL_Panaganti, AnnalsStat2022_TheoreticalUnderstandingRMDP_Yang}. 
Specifically, in this paper, we mainly focus on the standard $\phi$-divergence uncertainty set \cite{ghosh2025provably,sason2016f}.
\begin{defi}[$\phi$-Divergence Uncertainty Set]
\label{def:f_divergence_uncertainty}
   For each \((s,a)\) pair, the uncertainty set is defined as:
\begin{align*}
\mathcal{U}^{\phi, \sigma}_h(s,a)  \triangleq \left\{ P \in \Delta(\mathcal{S}): D_{\phi}\Big(P,P^{\star}_h(\cdot|s,a)\Big)\leq \sigma \right\},
\end{align*}
where $D_{\phi}\big(P,P^{\star}_h(\cdot|s,a)\big) =\ \sum\limits_{s' \in \mathcal{S}} \phi\big( \frac{ P(s')}{P^{\star}_h(s'|s,a)} \big) P^{\star}_h(s'|s,a)$ is the $\phi$-divergence \cite{sason2016f}. 
\end{defi}
Without loss of generality, we assume that $\text{Supp}(P)\subseteq \text{Supp}(P^\star)$ for all $P\in \mathcal{U}^{\phi, \sigma}$. Among the three uncertainty sets we considered later, the KL and $\chi^2$ divergence naturally satisfy this assumption, and the total variation also satisfies it under an additional standard assumption (\Cref{ass:vanisin_minimal}). Notably, without this assumption, the sample complexity in online setting can be exponentially large due to the information-bottleneck issue \cite{Arxiv2024_DRORLwithInteractiveData_Lu}.

\subsection{Policy and Robust Value Function} The agent's strategy of taking actions is captured by a Markov policy $\pi := \{\pi_h\}_{h=1}^H$, with $\pi_h: \mathcal{S}\rightarrow \Delta(\mathcal{A})$ for each step $h \in [H]$, where $\pi_h(\cdot|s)$ is the probability of taking actions at the state $s$ in step $h$. In RMDPs, the performance of a policy is captured by the worst-case performance, defined as the robust value functions. Specifically,  given any policy $\pi$ and for each step $h \in [H]$, the \textit{robust value function} and the \textit{robust state-action value function} are defined as the expected accumulative reward under the worst possible transition kernel within the uncertainty set:
\begin{align}
    &V^{\pi,\sigma}_{h}(s) = \inf_{P \in \mathcal{U}^{\phi,\sigma}}\mathbb{E}_{\pi,P}\bigg[\sum\limits_{t=h}^H r_t(s_t,a_t)\Big | s_h=s\bigg], \label{eq:robust_Q_fn}\\
    &Q^{\pi,\sigma}_{h}(s,a) = \inf_{P \in \mathcal{U}^{\phi,\sigma}}\mathbb{E}_{\pi,P}\bigg[\sum\limits_{t=h}^H r_t(s_t,a_t)\Big | s_h=s,a_h=a\bigg], \nonumber 
\end{align} 
where the expectation is taken with respect to the state-action trajectories induced by policy $\pi$ under transition $P$.

 The goal of DR-RL is to find the optimal robust policy $\pi^{\star}:=\{\pi^{\star}_h\}$ that maximizes the robust value function, for some initial state $s_1$:
\begin{align}
\label{eq:optimal_policy}
\pi^\star \triangleq \argmax_{\pi\in \Pi} V^{\pi,\sigma}_{1}(s_1), 
\end{align}
where $\Pi$ is the set of policies. Such an optimal policy exists and can be a deterministic policy \cite{INFORM2005_RobusDP_Iyengar,NeuRIPS2023_DoublePessimismDROfflineRL_Blanchet}. Moreover, the optimal robust value functions (denoted by $Q^{\star,\sigma}_{h},
V^{\star,\sigma}_{h}$), which are the corresponding robust value functions of the optimal policy $\pi^\star$, are shown to be the unique solution to the robust Bellman equations: 
\begin{align}
\label{eq:optimal_V_S_value}
Q^{\star,\sigma}_{h}(s,a)  &= r_h(s,a) + \mathbb{E}_{\mathcal{U}^{\phi,\sigma}_h(s,a)}\left[V^{\star,\sigma}_{h+1}\right], \nonumber\\ V^{\star,\sigma}_{h}(s) &= \max_{a \in \mathcal{A}} Q^{\star,\sigma}_{h}(s,a),
\end{align} 
where $\mathbb{E}_{\mathcal{U}^{\phi,\sigma}_h(s,a)}\left[V\right]\triangleq \inf\limits_{P_h \in \mathcal{U}^{\phi,\sigma}_h(s,a)}\mathbb{E}_{s'\sim P_h(\cdot|s,a)}\left[V(s')\right]$.

On the other hand, for any policy $\pi$, the corresponding robust value functions also satisfy the following robust Bellman equation for $\pi$ (\cite{NeuRIPS2023_DoublePessimismDROfflineRL_Blanchet}(Prop. 2.3)):
\begin{align}
    Q^{\pi,\sigma}_{h}(s,a) &= r_h(s,a) + \mathbb{E}_{\mathcal{U}^{\phi,\sigma}_h(s,a)}\left[V^{\pi,\sigma}_{h+1}\right], \nonumber\\
    V^{\pi,\sigma}_{h}(s) &= \mathbb{E}_{a \sim \pi_h(\cdot|s)} \left[ Q^{\pi,\sigma}_{h}(s,a) \right].\label{eq:Robust_bellman_V_fn} 
\end{align}

\subsection{Online Distributionally Robust RL}
In this work, we study distributionally robust RL in the online setting, where the agent aims to learn the robust-optimal policy \( \pi^\star \) defined in eq. \ref{eq:optimal_policy} through interaction with the nominal environment \( P^\star \) over \( K \in \mathbb{N} \) episodes. At the beginning of episode \( k \), the agent observes the initial state \( s_1^k \), chooses a policy \( \pi^k \) and executes it in \( P^\star \) to generate a trajectory, and then updates its policy before the next episode. In the online setting, the agent cannot explore arbitrarily and must instead control the risk of worst-case outcomes during learning. Accordingly, the objective is to minimize the \emph{cumulative robust regret} over $K$ episodes:
\begin{align}
\label{eq:Regret_K}
\text{Regret}(K) \triangleq \sum_{k=1}^K \left[V^{\star,\sigma}_{1}(s^k_1) - V^{\pi^k,\sigma}_{1}(s^k_1)\right].
\end{align}

We additionally assess performance via \emph{sample complexity}, defined as the minimum number of samples $T = KH$ required to learn an $\varepsilon$-optimal robust policy $\widehat{\pi}$ satisfying
\begin{align}
V^{\star,\sigma}_1(s_1) - V^{\widehat{\pi},\sigma}_1(s_1) \leq \varepsilon.
\end{align}

\section{Robust Bellman Operator with Function Approximation}
\label{sec:Dual_Reformulation}
In this section, we describe the structural and computational challenges of \emph{online} DR-RL with general function approximation under $\phi$-divergence uncertainty sets, and present our approach to overcoming them. Our presentation follows the Bellman--Eluder (BE) dimension framework for general RL with function approximation \cite{russo2013eluder,jiang2017contextual,sun2019model,NeurIPS2021_BellmanEluderDim_Jin}, adapted to robust Bellman updates.

\paragraph{Function approximation.} When the state–action space is large, learning robust policies from interaction alone is computationally challenging. To address this, we adopt the function approximation technique, where we use a general function class $\mathcal{F}=\{ \mathcal{F}_h\}^H_{h=1} $ where $\mathcal{F}_h$ contains some functions $f : \mathcal{S}\times\mathcal{A}\to[0,H]$, to approximate the robust value function $Q^{\star,\sigma}_h$. This function class can be a parametric class with low-dimension parameters, e.g., neural network, to significantly reduce the computation and improve sample efficiency. To ensure effective learning with these function classes, prior work has identified structural conditions that they must satisfy \cite{russo2013eluder,jiang2017contextual,sun2019model,wang_reinforcement_2020,NeurIPS2021_BellmanEluderDim_Jin,NeurIPS2022_RobustRLOffline_Panaganti}. These conditions regulate how the functional class \( \mathcal{F} \) interacts with the RMDP dynamics. The most commonly used assumptions are the \textbf{\textit{representation conditions}}, which require that \( \mathcal{F} \) is expressive enough to capture the robust value functions of interest. More specifically, the optimal robust Q-function \( Q^{\star,\sigma} \in \mathcal{F} \) (known as realizability) and closure under the robust Bellman operator, namely \( \mathcal{T}^{\phi,\sigma}_h \mathcal{F}_{h+1} \subseteq \mathcal{F}_h \) (known as completeness). Following standard studies of function approximation in RL \cite{NeurIPS2021_BellmanEluderDim_Jin, xie2022role, NeurIPS2022_RobustRLOffline_Panaganti, wang2019optimism}, we adopt the following completeness assumption. 
\begin{assu}[Completeness]
\label{ass:completeness}
    For all $h \in [H]$, we have $\mathcal{T}^{\phi,\sigma}_h f_{h+1} \in \mathcal{F}_h$ for all $f_{h+1} \in \mathcal{F}_{h+1}$.
\end{assu}
Per \Cref{ass:completeness}, $\mathcal{F}$ is closed under the robust Bellman operator $\mathcal{T}^{\phi,\sigma}$. Unlike standard function-approximation RL, we do not assume realizability ($Q^{\star,\sigma}\in\mathcal{F}$). Instead, realizability can be implied together with our assumption on the duality function class \cite{Arxiv2022_OfflineDRRLLinearFunctionApprox_Ma,PMLR2024_LinearFunctionDRMDP_Liu,NeurIPS2024_MinimaxOptimalOfflineRL_Liu,Arxiv2024_UpperLowerDRRL_Liu,Arxiv2024_SampleComplexityOfflineLinearDRMDP_Wang}.

\paragraph{Distribution families and robust BE dimension.} Let $\mathcal X$ be a domain and $\Phi \subseteq (\mathcal X \to \mathbb R)$ be a function class. Let $\Pi$ be a family of probability distributions over $\mathcal X$.

\begin{defi}[Distributional $\varepsilon$-dependence \cite{NeurIPS2021_BellmanEluderDim_Jin}] 
A distribution $\mu \in \Pi$ is said to be \emph{$\varepsilon$-dependent} on a set
$\{\mu_1,\dots,\mu_{n}\}\subseteq \Pi$ with respect to $\Xi$ if for all $\xi \in \Xi$,
\begin{equation}\label{eq:dist-eluder-dependence}
\Big|\mathbb{E}_{\mu}[\xi]\Big|
\;\le\;
\varepsilon + \Big(\sum_{i=1}^{n} \mathbb{E}_{\mu_i}[\xi]^2\Big)^{1/2}.
\end{equation}
Otherwise, $\mu$ is said to be \emph{$\varepsilon$-independent} of $\{\mu_1,\dots,\mu_n\}$.
\end{defi}

\begin{defi}[Distributional Eluder (DE) dimension \cite{NeurIPS2021_BellmanEluderDim_Jin}]
\label{def:DE_dim}
The \emph{distributional Eluder dimension} of $\Xi$ with respect to $\Pi$ at scale $\varepsilon>0$,
denoted $\dimDE(\Xi,\Pi,\varepsilon)$, is the length of the longest sequence
$x_1,\dots,x_m \in \Pi$ such that for each $i\in[m]$ there exists $\varepsilon'>\varepsilon$ such that $x_i \text{ is } \varepsilon'\text{-independent of } \{x_1,\dots,x_{i-1}\}$.
\end{defi}

\begin{rem}
When $\Pi$ is taken to be the set of point masses $\{\delta_x(\cdot):= \text{Dirac Distribution at $x$}: x\in\mathcal X\}$,
this definition reduces to the standard (pointwise) Eluder dimension.
\end{rem}

In robust setting, we define the robust Bellman residual class at stage $h$ as
\begin{align}
\label{eq:robust_residual_class}
(\mathcal{I}-\mathcal{T}^{\phi,\sigma}_h)\mathcal{F}
:= \Big\{ f_h - \mathcal{T}^{\phi,\sigma}_h f_{h+1} \,:\, f\in \mathcal{F}\Big\},
\end{align}
where $\mathcal{T}^{\phi,\sigma}_hf(s,a)=r_h(s,a)+\mathbb{E}_{\mathcal{U}^{\phi,\sigma}_h(s,a)}\left[f_{\max}\right]$ is the robust Bellman operator, and $f_{\max}(s)=\max_a f(s,a)$. 
This class captures all possible robust Bellman errors that can arise from candidate value functions in $\mathcal{F}$. We then define the robust Bellman-Eluder (BE) dimension as follows.
\begin{defi}
\label{def:robust_BE_dim}
Let $\dim_{\mathrm{DE}}(\Xi,\Pi,\varepsilon)$ denote the DE dimension of a function class $\Xi$ with respect to a distribution family $\Pi$ at scale $\varepsilon$ \cite{NeurIPS2021_BellmanEluderDim_Jin}. We then define the \emph{robust Bellman--Eluder dimension} of $\mathcal{F}$ by
\begin{align*}
\dim^{\mathrm{rob}}_{\mathrm{BE}}(\mathcal{F},\Pi,\varepsilon)
:= \max_{h\in[H]} \dim_{\mathrm{DE}}\!\Big( (\mathcal{I}-\mathcal{T}^{\phi,\sigma}_h)\mathcal{F},\, \Pi_h,\, \varepsilon \Big),
\end{align*}
where $\mathcal{I}-\mathcal{T}^{\phi,\sigma}_h$ be the robust Bellman residual class. 
\end{defi}
This quantity measures the intrinsic statistical complexity of learning robust value functions under function approximation. Smaller robust BE dimension implies fewer distinguishable robust Bellman errors along interaction trajectories, enabling tighter confidence sets and improved sample efficiency.

\begin{rem}[Distribution families $D_{\mathcal F}$ and $D_\Delta$]
Following \cite{NeurIPS2021_BellmanEluderDim_Jin}, we consider two distribution families:
\begin{itemize}
\item $D_{\mathcal F}=\{D_{\mathcal F,h}\}_{h=1}^H$, where $D_{\mathcal F,h}$ is the collection of step-$h$ state-action occupancy distributions induced by rolling in with $\pi_f$ for some $f\in\mathcal F$.
\item $D_\Delta=\{D_{\Delta,h}\}_{h=1}^H$, where $D_{\Delta,h}$ is the set of point masses on $\mathcal S \times \mathcal A$.
\end{itemize}
\end{rem}

\begin{rem}[Robust Bellman Rank and Relations with known RL Problems]
Many structured RL models—e.g., tabular, linear, reactive POMDPs—admit efficient learning guarantees. In the non-robust literature, two generic tractability notions---{\it low Bellman rank} and {\it low Eluder dimension}--cover many examples but are incomparable. For DR-RL, we instead use a single notion—{\it robust BE dimension}—defined via the robust Bellman residual class (eq. \ref{eq:robust_residual_class}). As shown in Appendix \ref{sec:Robust_Bellman_Rank}, low robust BE dimension subsumes the robust analogs of both low robust Bellman rank and low Eluder dimension, providing a unified structural condition for tractable robust RL with function approximation. 
\end{rem}

\paragraph{Empirical robust Bellman operator and functional optimization.}
We first provide the dual formulation of $\phi$-divergence uncertainty set support functions. Given a nominal kernel $P^\star$, there is a dual formulation of $\mathbb{E}_{\mathcal{U}^{\phi,\sigma}_h(s,a)}[\cdot]$: 
\begin{align}
\label{eq:phi-ball-dual}
\mathbb{E}_{\mathcal{U}^{\sigma}_h(s,a)}[V]
&=
-\inf_{\eta > 0,\ \nu \in \mathbb{R}}
\Big\{\eta\sigma - \nu +\nonumber\\
&\quad \eta \mathbb{E}_{s'\sim P_h^\star(s'|s,a)}\left[\phi^\star\!\left(\nu - V(s')/\eta\right)\right]
\Big\},
\end{align}
where $\phi^\star$ is the convex conjugate of $\phi$ restricted to $[0,\infty)$. We will use this representation as a unifying template for deriving robust bonuses. This formula, as given in eq. \ref{eq:phi-ball-dual}, follows from standard strong duality arguments; see, e.g., \cite{AnnalsStat2022_TheoreticalUnderstandingRMDP_Yang}[Lemma B.1].

We now denote $l_{\phi}(f;s,a,s';\eta,\nu)$ as the pointwise dual integrand for a single next-state $s'$ and it is defined as 
\begin{align}
\label{eq:pointiwse_l}
  l_{\phi}(f;s,a,s';\eta,\nu)&\triangleq \eta\sigma-\nu \,\\
&\quad + \eta\phi^\star\!\left(-\frac{\max_{a'}f(s',a')+\nu}{\eta}\right).\nonumber
\end{align}
The robust Bellman operator is then 
\begin{align}
&(\mathcal{T}^{\phi,\sigma}_h f_{h+1})(s,a) := r(s,a) \nonumber\\
&\quad -\inf_{\eta > 0,\ \nu \in \mathbb{R}}\mathbb{E}_{s'\sim P^\star_h(s,a)}\left[l_{\phi}(f_{h+1};s,a,s';\eta,\nu\right].\label{eq:Robust_Bellman_Operator_TV}
\end{align}

From eq. \ref{eq:Robust_Bellman_Operator_TV} we can recall that the robust value function is the fixed point of the robust Bellman operator. Therefore, finding an optimal robust policy reduces to computing this fixed point. As a result, applying the operator exactly is generally impractical: for every $(s,a)$, the term $\mathbb{E}_{\mathcal{U}^{\phi,\sigma}_h(s,a)}[\cdot]$ entails solving an optimization problem over an $S$-dimensional $\phi$-divergence uncertainty set, which rapidly becomes computationally expensive.

To overcome this computational bottleneck, we develop an efficient empirical procedure, inspired by \cite{NeurIPS2022_RobustRLOffline_Panaganti}, which eliminates pointwise scalar optimizations by reformulating the computation as a \emph{single} functional optimization problem. Specifically, consider the probability space $(\mathcal{S}\times \mathcal{A}, \Sigma(\mathcal{S}\times \mathcal{A}), \mu)$, where $\mu$ is a distribution over $\mathcal{S}\times\mathcal{A}$ (typically the stage-$h$ visitation distribution). Let $\mathcal{L}^1(\mu;\mathbb R^2)\triangleq \{g=(g_{\eta},g_{\nu}): \mathcal{S}\times \mathcal{A} \to \mathbb R^2\big|g_{\eta},g_{\nu}\in \mathcal{L}^1(\mu)\}$ denote the space of absolutely integrable dual functions and define the dual loss $\mathrm{DualLoss}(g; f)$ as 
\begin{align}
\label{eq:loss_dual}
\mathbb{E}_{(s,a)\sim\mu}\big[\mathbb{E}_{s'\sim P_{s,a}^{\star}}[l_{\phi}(f;s,a,s';g)\big],
\end{align}
where $g(s,a)=\big(g_{\eta}(s,a),g_{\nu}(s,a)\big)$ and $l_{\phi}(f;s,a,s';g):=g_{\eta}(s,a)\sigma-g_{\nu}(s,a) + g_{\eta}(s,a)\,\phi^\star\!\left(-\frac{\max_{a'}f(s',a')+g_{\nu}(s,a)}{g_{\eta}(s,a)}\right)$. We then extend the results in \cite{NeurIPS2022_RobustRLOffline_Panaganti} (which considers the total variation set) and \cite{Arxiv2024_ModelFreeRobustRL_Panaganti} (which considers regularized MDPs), and show that for general $\phi$-divergence RMDP, minimization of the dual loss is equivalent to the support function under some distribution.
\begin{lem}[Dual loss minimization]
\label{lem:equiv_loss_dual}
Let $\mathrm{DualLoss}$ be defined as the dual loss function as in eq. \ref{eq:loss_dual}. Then, for any function $f:\mathcal S \times \mathcal A \to [0,H]$, we have 
\begin{align*}
\inf_{g\in \mathcal{L}^1(\mu;\mathbb R^2)} \mathrm{DualLoss}(g; f)=
\mathbb{E}_{(s,a)\sim \mu}\left[\mathbb{E}_{\mathcal{U}^{\phi,\sigma}_h(s,a)}[f]\right].
\end{align*}
\end{lem}
This enables us to form an empirical counterpart of the dual objective, denoted by $\widehat{\mathrm{DualLoss}}(g; f)$, and compute an approximate dual minimizer via
$\inf\limits_{g\in\mathcal{L}^1(\mu;\mathbb R^2)} \widehat{\mathrm{DualLoss}}(g; f)$. We further introduce a function class  $\mathcal{G}=\{g=(g_{\eta},g_{\nu})|g_{\eta},g_{\nu} :\mathcal{S}\times\mathcal{A}\to\Theta_{\phi}\}$ to approximate the dual variables as well. We adopt the following realizability condition from \cite{NeurIPS2022_RobustRLOffline_Panaganti,Arxiv2024_ModelFreeRobustRL_Panaganti}.
\begin{assu}
\label{ass:approx_dual_realizability}
    For all $f\in \mathcal{F}$ and any policy $\pi$, there exists a uniform constant $\varepsilon^{\mathrm{dual}}$ such that
    \[
      \inf_{g\in \mathcal{G}} \mathrm{DualLoss}(g; f)
      -
      \inf_{g\in \mathcal{L}^1(\mu^\pi;\mathbb R^2)} \mathrm{DualLoss}(g; f)
      \;\leq\; \varepsilon^{\mathrm{dual}},
    \]
    where $\mu^\pi$ is the distribution induced by $\pi$ under $P^\star$.
\end{assu}
This assumption is not restrictive. For instance, note that $\mathcal{L}^1$ can be approximated by deep/wide neural networks \citep{goodfellow2016deep}, which ensures \Cref{ass:approx_dual_realizability} with such neural network classes. Accordingly, for a fixed $f$ and dataset $\mathcal{D}$, we approximate the robust Bellman operator by first computing
$ \underline{g}_f = \arg\min\limits_{g\in\mathcal{G}} \widehat{\mathrm{DualLoss}}(g; f),$
and then plugging it in the robust Bellman operator as $(\mathcal{T}^{\phi,\sigma}_{\underline{g}_f} f)(s,a)$:
\begin{align}\label{eq:revised_dual_TV_g}
r(s,a) - \mathbb{E}_{s'\sim P^{\star}_{s,a}}\Big[l_{\phi}(f;s,a,s';\underline{g}_f)\Big].
\end{align}

\paragraph{Uniform approximation of the robust Bellman operator.}
The following lemma quantifies how accurately the empirical operator approximates the true robust Bellman operator in a global $\mathcal{L}^1$ sense. We first adopt the following assumption, which holds for all uncertainty sets we considered. 

\begin{assu}
\label{ass:multiplier-range-TV}
There exist a compact set $\Theta_{\phi}\subseteq \mathbb{R}_+$ and a constant $ B_{\phi}(\sigma)<\infty$ such that for all $h\in[H]$, all $f_{h+1}\in\mathcal{F}_{h+1}$, all $\eta,\nu\in\Theta_{\phi}$, and for all $(s,a,s')$, $\big|l_{\phi}(f_{h+1};s,a,s';\eta,\nu)\big|\le B_{\phi}(\sigma)$.
\end{assu}

\begin{lem}[Uniform approximation of robust Bellman operator]
\label{lem:tv_operator_approx_be}
Fix any policy $\pi$, let $\mu_h^\pi$ be the step-$h$ state-action distribution under $P^\star$, and let $\mathcal{D}$ be the dataset collected by executing $\pi$. Then, under \Cref{ass:completeness}-\ref{ass:multiplier-range-TV}, for any $\delta\in(0,1)$, with probability at least $1-\delta$, it holds that 
\begin{align}
    &\sup_{f\in\mathcal{F}_{h+1}}
\big\|\mathcal{T}^{\phi,\sigma}_h f-\mathcal{T}^{\phi,\sigma}_{\underline{g}_f} f\big\|_{1,\mu_h^\pi} \nonumber\\
&=
\mathcal{O}\!\bigg(
B_{\phi}(\sigma)\sqrt{\frac{\log\!\big(|\mathcal{F}_{h+1}||\mathcal{G}|/\delta\big)}{|\mathcal{D}|}}
+\varepsilon^{\mathrm{dual}}
\bigg).
\end{align}
\end{lem}

A similar result is derived for a fixed distribution (the offline dataset distribution) in \cite{NeurIPS2022_RobustRLOffline_Panaganti,Arxiv2024_ModelFreeRobustRL_Panaganti}, whereas we show it simultaneously hold for any policy and its induced distribution. 
Lemma \ref{lem:tv_operator_approx_be} shows that our empirical functional optimization yields a uniformly accurate approximation to the robust Bellman operator under the $\mathcal{L}^1(\mu^\pi;\mathbb R^2)$ norm. Crucially, the error is controlled \emph{globally} with respect to the distribution $\mu^\pi$, rather than pointwise in $(s,a)$. This global control is what we leverage later to define our robust confidence sets and the global error term that drives the design and analysis of our main algorithm.

\begin{rem}[Relation to $\varphi$-regularized RMDPs \cite{Arxiv2024_ModelFreeRobustRL_Panaganti}]
Assumption~\ref{ass:approx_dual_realizability} and Lemma~\ref{lem:tv_operator_approx_be} build on the dual functional machinery first developed by \cite{NeurIPS2022_RobustRLOffline_Panaganti} and subsequently employed by \cite{Arxiv2024_ModelFreeRobustRL_Panaganti} for $\varphi$-regularized RMDPs in a hybrid setting, where the policy value includes a Lagrangian penalty $\lambda$ with $\lambda>0$ and the guarantees scale with $(\lambda+H)$. Although the $\varphi$-regularized RMDPs recovers the standard RMDPs with $\lambda=0$, the studies in \cite{Arxiv2024_ModelFreeRobustRL_Panaganti} are developed for $\lambda>0$, hence our result cannot be obtained directly. 
\end{rem}

\section{Robust Fitting Learning Algorithm for $\phi$-Divergence Set (RFL-$\phi$)}
\label{sec:Algorithm}
We then utilize our previous constructions and propose our Robust Fitted Learning (RFL) algorithm. 
\begin{algorithm}[t]
\caption{{\Algo}}
\label{alg:robust_golf}
\begin{algorithmic}[1]
\State \textbf{Input:} Function class $\mathcal{F}$, Dual Function class $\mathcal{G}$,  $\beta > 0$,   $\sigma >0$.
\State \textbf{Initialize:} $\mathcal{F}^{(0)} \gets \mathcal{F}, \ \mathcal{D}^{(0)}_h \gets \emptyset \ \forall h \in [H]$
\For{episode $k = 1,2,\dots,K$}
\State $\pi^{(k)} \gets \pi^{f^{(k)}}$, where $f^{(k)} \leftarrow \arg\max_{f \in \mathcal{F}^{(k-1)}} f(s_1,\pi^f_1(s_1))$  
  \State {Execute} $\pi^{(k)}$ to collection trajectory $(s^{(k)}_1, a^{(k)}_1, r^{(k)}_1), \dots, (s^{(k)}_H, a^{(k)}_H, r^{(k)}_H)$  
  \State  $\mathcal{D}^{(k)}_h \gets \mathcal{D}^{(k-1)}_h \cup \{(s^{(k)}_h,a^{(k)}_h, s^{(k)}_{h+1})\}, \forall h$
\State {$\mathcal{F}^{(k)}_H\leftarrow\{0\}$}
\State For all $f_{h+1}\in \mathcal{F}^{(k)}_{h+1}$, update the confidence set according to eq. \ref{eq:hatg}:
  \begin{align*}
  \mathcal{F}^{(k)}_h\leftarrow \Bigg\{ f \in \mathcal{F}_h : L_{\mathcal{D}^{(k)}_h}(f_h,f_{h+1},\underline{g}_{f_{h+1}}) - \min_{f'_h \in \mathcal{F}_h} L_{\mathcal{D}^{(k)}_h}(f'_h,f_{h+1},\underline{g}_{f_{h+1}}) \leq \beta, \forall h\in [H] \Bigg\}  
  \end{align*}
\EndFor
\State \textbf{Output:} $\bar{\pi} = \text{unif}(\pi^{(1:K)})$. 
\end{algorithmic}
\end{algorithm}

Our algorithm follows the standard fitting learning structure. In each step $h$, we construct a confidence set $\mathcal{F}^{(k)}$ (Line 8) using the fitted error under the robust Bellman operator to ensure the inclusion of $Q^{\star,\sigma}\in\mathcal{F}^{(k)}$ (Lemma \ref{lem:robust-lemma39}). As discussed, we utilize our functional optimization based loss function and the error bound in Lemma \ref{lem:tv_operator_approx_be} to construct the set. Namely, given a function $f$, we first obtain the dual-variable approximation $(\underline{g}_{\eta,f},\underline{g}_{\nu,f})$ via minimizing empirical functional optimization loss, formulated as
\begin{align}
\underline{g}_f \triangleq \argmin\limits_{g\in \mathcal{G}}\sum_{(s,a,s')\in \mathcal{D}^{(k)}_h}\Big(l_{\phi}(f;s,a,s';g)\Big),\label{eq:hatg}
\end{align}
where we further capture the empirical robust Bellman error $L_{\mathcal{D}^{(k)}_h}(f',f,g)$ based on dataset $\mathcal{D}^{(k)}_h$ as $\sum_{(s,a,r,s')\in\mathcal{D}^{(k)}_h}\Big\{f'(s,a)-r-l_{\phi}(f;s,a,s';g)\Big\}^2$.

Notably, to handle large-scale problems we build \emph{global} confidence sets by optimizing over $\{f_h\}_{h=1}^H$ jointly \cite{zanette2020learning}, rather than certifying errors state--action-wise as in tabular UCB. Concretely, $\mathcal{F}^{(k)}$ contains functions that (i) achieve small squared robust Bellman error on the collected data $\{\mathcal{D}^{(k)}_h\}_{h=1}^H$ via the dual plug-in residual, and (ii) include any $f$ whose empirical loss is within a tolerance of the best loss in $\mathcal{F}$. We later choose the threshold $\beta$ so that $Q^{\star,\sigma}\in \mathcal{F}^{(k)}$ holds with high probability. Given this valid set, we apply optimism and select $\pi^{(k)}=\pi^{f^{(k)}}$, where $f^{(k)}\in\mathcal{F}^{(k)}$ maximizes the optimistic estimate $f^{(k)}_1\!\big(s_1,\pi^{(k)}_1(s_1)\big)$ of total return, thereby balancing exploration and performance.

\paragraph{Algorithmic novelties.}
While our algorithm follows the high-level paradigm of “optimism + fitted value iteration,” similar in spirit to GOLF~\citep{NeurIPS2021_BellmanEluderDim_Jin, xie2022role}, its design is fundamentally tailored to the robust BE dimension rather than non-robust Bellman-error control or coverage-based arguments \cite{xie2022role} and offline DR-RL \cite{NeurIPS2022_RobustRLOffline_Panaganti, Arxiv2024_ModelFreeRobustRL_Panaganti}. A central challenge is that the data ${\mathcal{D}_h^{(k)}}$ are collected under a sequence of evolving policies across episodes, so learning a single policy $\pi$ cannot rely on guarantees derived under a fixed sampling distribution ${\mathcal{D}_h^{(k)}} \sim \mu^{\pi}$. To address this, we develop an optimistic, dual-driven fitted learning procedure in which a value–dual pair $(f,g)$ is learned jointly online. The dual function $g$ serves a dual role: it provides a tractable functional approximation to the TV-robust Bellman operator, and it induces a global control of the robust Bellman residual class $(\mathcal{I}-\mathcal{T}_h^{\phi,\sigma})\mathcal{F}$, which is precisely the object governing the robust BE dimension (Definition \ref{def:robust_BE_dim}). This allows us to construct confidence sets directly in function space and implement optimism in a manner that is tightly coupled to the intrinsic complexity of the problem. In contrast to GOLF, which controls squared non-robust Bellman errors, and to offline RFQI-style robust methods~\citep{NeurIPS2022_RobustRLOffline_Panaganti}, where the dual is analyzed under a fixed data distribution and does not guide exploration, our approach leverages the dual to both approximate robustness and drive exploration. This alignment between algorithm design, statistical control, and complexity measure is what enables regret guarantees governed by the robust BE dimension.

\section{Theoretical Guarantees}
\label{sec:Analysis}
We then develop the theoretical guarantees of our algorithm.

\begin{thm}
\label{thm:regret_bound_RGOLF}
Assume \Cref{ass:completeness}-\ref{ass:multiplier-range-TV}. For any $\delta\in (0,1]$, we set $\beta=\mathcal{O}\Big(B_{\phi}(\sigma)\log\!\big(|\mathcal{F}||\mathcal{G}|KH/\delta\big)\Big)$ in \Cref{alg:robust_golf}. Then, with probability at least $1-\delta$,  \footnote{We assume for simplicity that $|\mathcal{F}|,|\mathcal{G}| < \infty$, but our result can be directly extended to the general infinite case with the standard finite coverage technique \cite{xie2022role,NeurIPS2022_RobustRLOffline_Panaganti}.}
\begin{align}
&\mathrm{Regret}(K)
\;\le\;
\tilde{\mathcal{O}}\!\Big(
\sqrt{dH^2B^2_{\phi}(\sigma)K}+ \;\varepsilon^{\mathrm{dual}}\Big),
\end{align}
\end{thm}
where $d \;:=\; \dimrobBE\bigl(\mathcal F, D_{\mathcal F}, 1/\sqrt{K}\bigr)$ as in \Cref{def:robust_BE_dim}.

\paragraph{Technical novelties and implications.}
Our work provides the first regret guarantees for online DR-RL with general function approximation based on an intrinsic complexity measure—the robust Bellman–Eluder (BE) dimension—without relying on coverage, density-ratio, or concentrability assumptions. This brings online robust RL into the modern complexity-theoretic framework of general function approximation \cite{NeurIPS2021_BellmanEluderDim_Jin}, while requiring new techniques to handle worst-case dynamics.

First, we introduce the robust BE dimension as the appropriate notion of intrinsic difficulty for online DR-RL. Our bounds depend only on the statistical complexity of the robust Bellman residual class $(\mathcal{I}-\mathcal{T}^{\phi,\sigma}_h)\mathcal{F}$ evaluated along on-policy trajectories. This yields an exploration theory for robust MDPs analogous to BE-based non-robust RL, avoids external coverage assumptions, and recovers sharp rates in structured settings such as tabular and linear RMDPs. Second, unlike existing BE analyses that assume a known Bellman operator, our setting requires learning the robust Bellman operator from data. We address this via a dual ERM plug-in backup, which introduces a new operator approximation error. We develop a new uniform control argument based on bounded $\phi$-divergence multipliers and approximate dual realizability, a component absent from prior BE-based theory. Finally, our regret decomposition cleanly separates (i) an exploration term governed by the robust BE dimension and (ii) an operator-approximation term governed by dual learning error. This parallels the structure of non-robust BE analyses but requires fundamentally new arguments due to the learned robust operator.

In contrast to offline/hybrid robust RL, which assumes fixed datasets and strong global coverage assumptions \cite{NeurIPS2022_RobustRLOffline_Panaganti,Arxiv2024_ModelFreeRobustRL_Panaganti}, we study fully online setting with evolving data distributions and provide the first intrinsic, complexity-theoretic characterization of exploration in online DR-RL with general function approximation.

As an immediate corollary, we obtain the sample complexity for learning an $\varepsilon$-optimal policy with {\Algoname} by applying a standard online-to-batch conversion \cite{NeuRIPS2001_OnlineLearningAlgo_Cesa} for each total variation (TV), $\chi^2$ and KL divergence uncertainty sets.

\begin{cor}[Sample Complexity: TV, $\chi^2$, KL]
\label{cor:Sample_Complexity_bound_TV}
Under the same setup in Theorem \ref{thm:regret_bound_RGOLF}, set $\varepsilon^{\mathrm{dual}}=0$ and additionally assume \Cref{ass:vanisin_minimal} 
for TV. Then, with probability at least $1 - \delta$, the sample-complexity of {\Algoname} to obtain an $\varepsilon$-optimal robust policy is $T=KH$ as 
\begin{align*}
\begin{cases}
        \mathcal{O}\bigg(\frac{H^5(\min\{H,1/\sigma\})^2d\log\big(|\mathcal F||\mathcal G|T/\delta\big)}{\varepsilon^2}\bigg), \hfill\text{{\RMDPTV}}\\
        \mathcal{O}\bigg(\frac{H^5(1+\sqrt{\sigma})^2d\log\big(|\mathcal F||\mathcal G|T/\delta\big)}{\varepsilon^2}\bigg), \qquad \text{{\RMDPchi}}\\
        \mathcal{O}\bigg(\frac{H^5\sigma^2d\log\big(|\mathcal F||\mathcal G|T/\delta\big)}{\varepsilon^2}\bigg), \qquad \text{{\RMDPKL}}
    \end{cases}
\end{align*}
\end{cor}

Our bound is independent of $S$ and $A$, indicating scalability to large state and action spaces. Moreover, as we shall discuss later, the dependences on other parameters, $H, \sigma, \varepsilon$, are also tight and near-optimal.

To validate the near optimality and sharpness of our results, we specialize them to two special cases: tabular case and linear RMDP \cite{Arxiv2022_OfflineDRRLLinearFunctionApprox_Ma,Arxiv2024_UpperLowerDRRL_Liu} setting. A more detailed comparison is provided in Appendix \ref{sec:com}. 
\begin{rem}[Tabular RMDPs] 
\label{rem:RFLTV_Tabular}
In the finite tabular case, we take $\mathcal{F}$ and $\mathcal{G}$ as the full classes of bounded functions $\mathcal{S}\!\times\!\mathcal{A}\to[0,H]$ and $\mathcal{S}\!\times\!\mathcal{A}\to\Theta_{\phi}$, so that $\log(|\mathcal{F}||\mathcal{G}|)=\tilde{\mathcal{O}}(SA)$ \cite{NeurIPS2021_BellmanEluderDim_Jin} and $d=\mathcal{O}(SA)$.  Plugging these into Theorem~\ref{thm:regret_bound_RGOLF} yields a tight tabular regret bound. For instance, for $\chi^2$ set, our sample complexity for an $\varepsilon$-optimal policy is $\tilde{\mathcal{O}}\!\big({H^5(1+\sqrt{\sigma})^2S^2A^2\varepsilon^{-2}}\big)$, which improves the prior results $\tilde{\mathcal{O}}\!\big(C_{\mathrm{vr}}{H^5(1+\sqrt{\sigma})^2S^3A\varepsilon^{-2}}\big)$ in \cite{ICML2025_OnlineDRMDPSampleComplexity_He} (note that $C_{\mathrm{vr}}$ is polynomial in $S,H,A$). 
\end{rem}

\begin{rem}[Linear RMDPs]
\label{rem:linear_RMDOP_BE}
As another special case, we instantiate our framework to $d_{\mathrm{lin}}$-rectangular linear RMDPs~\citep{Arxiv2022_OfflineDRRLLinearFunctionApprox_Ma,Arxiv2024_UpperLowerDRRL_Liu}. Note that for linear RMDPs, the uncertainty rectangularity is defined differently from our $(s,a)$-rectangularity. However, our method can be similarly extended (see \Cref{subsec:linear-RMDP}). When \(\mathcal F\) and the dual class \(\mathcal G\) are \(d_{\mathrm{lin}}\)-dimensional linear function classes, the robust BE complexity satisfies
\(\dimrobBE(\mathcal F, D_{\mathcal F}, 1/\sqrt{K})=\widetilde{\mathcal O}(d_{\mathrm{lin}})\) and
\(\log(|\mathcal F||\mathcal G|)=\widetilde{\mathcal O}(d_{\mathrm{lin}})\)
(cf.~\cite{wang_reinforcement_2020, NeurIPS2021_BellmanEluderDim_Jin}).
Substituting into \Cref{thm:regret_bound_RGOLF} with \(\varepsilon^{\mathrm{dual}}=0\) yields
\(
\mathrm{Regret}(K)=\widetilde{\mathcal O}\!\big(\sqrt{d^2_{\mathrm{lin}}\,H^2\,B^2_{\phi}(\sigma)K}\big)
\).
In particular, for TV uncertainty, we obtain
\(
\mathrm{Regret}(K)=\widetilde{\mathcal O}\!\Big(\sqrt{d^2_{\mathrm{lin}}\,H^{4}\big(\min\{H,1/\sigma\}\big)^2K}\Big),
\)
which matches the sharp dependencies on $(d_{\mathrm{lin}},\sigma,K)$ in
\cite{Arxiv2024_UpperLowerDRRL_Liu} except $H$. Hence, our general function-approximation theory recovers the near-optimal regret while strictly extending the scope beyond linear realizability.
\end{rem}

Our results hence are sharp and approximately match or improve the previous results under the two special cases, while strictly extending the generality.

\subsection{Comparision}\label{sec:com}

\begin{table*}[t]
\centering
\caption{Comparison of sample complexity in general-function, tabular and linear settings. Our contributions are highlighted.}
\label{tab:comparison}
\small
\setlength{\tabcolsep}{6pt}
\renewcommand{\arraystretch}{1.15}

\begin{tabular}{c|c|c|c|c}
\toprule
\textbf{Setting} & \textbf{Method} & \textbf{Robust} & \textbf{Divergence} & \textbf{Sample complexity} \\
\midrule

\multirow{8}{*}{\textbf{General}} 

& Online, \cite{xie2022role} & No & -- 
& $\tilde{\mathcal{O}}\!\big(C_{\mathrm{cov}}H^3\log(|\mathcal{F}|/\delta)\varepsilon^{-2}\big)$ \\

& Online, \cite{NeurIPS2021_BellmanEluderDim_Jin} & No & -- 
& $\tilde{\mathcal{O}}\!\big(H^3 \text{dim}_{\text{BE}}\log(|\mathcal{F}|/\delta)\varepsilon^{-2}\big)$ \\

\cmidrule(l){2-5}

& \multirow{3}{*}{\textbf{Online, {\Algoname} (Ours)}} 
& \multirow{3}{*}{Yes}
& TV 
& $\tilde{\mathcal{O}}\!\big(H^{5}(\min\{H,1/\sigma\})^2\dimrobBE\log(|\mathcal{F}||\mathcal{G}|/\delta)\varepsilon^{-2}\big)$ \\

& & & $\chi^2$ 
& $\tilde{\mathcal{O}}\!\big(H^{5}(1+\sqrt{\sigma})^2\dimrobBE\log(|\mathcal{F}||\mathcal{G}|/\delta)\varepsilon^{-2}\big)$ \\

& & & KL 
& $\tilde{\mathcal{O}}\!\big(H^{5}\sigma^2\dimrobBE\log(|\mathcal{F}||\mathcal{G}|/\delta)\varepsilon^{-2}\big)$ \\

\cmidrule(l){2-5}

& Lower Bound & -- & -- & N/A \\

\midrule

\multirow{10}{*}{\textbf{Tabular}} 

& Online, \cite{azar_minimax_2017} & No & -- 
& $\tilde{\mathcal{O}}\!\big(SAH^4\varepsilon^{-2}\big)$ \\

& Online, \cite{Arxiv2024_DRORLwithInteractiveData_Lu} & Yes & TV 
& $\tilde{\mathcal{O}}\!\big(H^3\min\{H,\sigma^{-1}\}SA\varepsilon^{-2}\big)$ \\

\cmidrule(l){2-5}

& \multirow{3}{*}{Online, \cite{ICML2025_OnlineDRMDPSampleComplexity_He}} & \multirow{3}{*}{Yes} & TV 
& $\tilde{\mathcal{O}}\!\big(C_{\mathrm{vr}}S^3AH^5\varepsilon^{-2}\big)$ \\
& & & $\chi^2$ & $\tilde{\mathcal{O}}\!\big(C_{\mathrm{vr}}(1+\sqrt{\sigma})^2S^3AH^5\varepsilon^{-2}\big)$ \\
& & & KL &  $\tilde{\mathcal{O}}\!\big(C_{\mathrm{vr}}\big(1+(H\sqrt{S}/\sigma P^\star_{\min})\big)^2 SAH^3\varepsilon^{-2}\big)$\\

\cmidrule(l){2-5}

& \multirow{3}{*}{\textbf{Online, {\Algoname} (Ours)}} 
& \multirow{3}{*}{Yes}
& TV 
& $\tilde{\mathcal{O}}\!\big(H^{5}(\min\{H,1/\sigma\})^2 S^2A^2\varepsilon^{-2}\big)$ \\

& & & $\chi^2$ 
& $\tilde{\mathcal{O}}\!\big(H^{5}(1+\sqrt{\sigma})^2 S^2A^2\varepsilon^{-2}\big)$ \\

& & & KL 
& $\tilde{\mathcal{O}}\!\big(H^{5}\sigma^2S^2A^2\varepsilon^{-2}\big)$ \\

\cmidrule(l){2-5}

& Lower Bound, \cite{Arxiv2024_DRORLwithInteractiveData_Lu} & Yes & TV 
& $\tilde{{\Omega}}\!\big(H^{3}\min\{H,1/\sigma\}SA\varepsilon^{-2}\big)$ \\

& \multirow{2}{*}{Lower Bound, \cite{ghosh2025provably}} 
& \multirow{2}{*}{Yes}
& $\chi^2$ 
&$\tilde{\Omega}\!\big(H^5(1+\sigma)SA\varepsilon^{-2}\big)$ \\

& & & KL 
& $\tilde{\Omega}\!\big(H^5SA\varepsilon^{-2}/(P^\star_{\min}\sigma^2)\big)$ \\

\midrule
\multirow{4}{*}{\textbf{Linear}} 
& Online, \citep{PMLR2023_NearlyMinimaxOptimalRLLinearMDP_He} & No &  --
& $\tilde{\mathcal{O}}\!\big(d^{2}_{\mathrm{lin}}H^{4}\varepsilon^{-2}\big)$ \\[0.2em]
& Online, \citep{Arxiv2024_UpperLowerDRRL_Liu} & Yes & TV
& $\tilde{\mathcal{O}}\!\big(d^{2}_{\mathrm{lin}}H^{3}(\min\{H,1/\sigma\})^{2}\varepsilon^{-2}\big)$ \\[0.2em]
& Hybrid, \citep{Arxiv2024_ModelFreeRobustRL_Panaganti} & Yes & TV
& $\tilde{\mathcal{O}}\!\big(\max\{C^{2}(\pi^\star),1\}\,d^{3}_{\mathrm{lin}}H^{3}(\lambda+H)^{2}\varepsilon^{-2}\big)$ \\[0.2em]
\cmidrule(l){2-5}

& \multirow{3}{*}{\textbf{Online, {\Algoname} (Ours)}} 
& \multirow{3}{*}{Yes}
& TV 
& $\tilde{\mathcal{O}}\!\big(d^{2}_{\mathrm{lin}}H^{5}(\min\{H,1/\sigma\})^{2}\varepsilon^{-2}\big)$ \\

& & & $\chi^2$ 
& $\tilde{\mathcal{O}}\!\big(d^2_{\mathrm{lin}}H^{5}(1+\sqrt{\sigma})^2\varepsilon^{-2}\big)$ \\

& & & KL 
& $\tilde{\mathcal{O}}\!\big(d^2_{\mathrm{lin}}H^{5}\sigma^2\varepsilon^{-2}\big)$ \\

\cmidrule(l){2-5}
& Lower Bound \citep{Arxiv2024_UpperLowerDRRL_Liu}& Yes & TV
& $\tilde{{\Omega}}\!\big( d^2d^2_{\mathrm{lin}}H^2(\min\{H,1/\sigma\})^2\varepsilon^{-2}\big)$ \\
\bottomrule

\end{tabular}

\end{table*}

Table~\ref{tab:comparison} summarizes our theoretical guarantees of {\Algoname} of $\phi$-divergence uncertainty set, and positions them relative to existing results across three regimes. First, in the \emph{general-function, purely online} setting, our bound is governed by an intrinsic BE-style complexity measure and makes the TV/$\chi^2$/KL uncertainty dependence explicit. Second, in \emph{tabular} RMDPs, our specialization recovers the same TV uncertainty scaling achieved by near-optimal tabular analyses, while providing parallel $\chi^2$/KL scalings in the same uncertainty-multiplier format. Third, in \emph{linear} RMDPs, the table should be read as a \emph{reduction}: under small robust BE dimension, our general theorem yields the stated linear rates, consistent with BE-dimension-based comparisons and without implying universal learnability of arbitrary online linear RMDPs \citep{Arxiv2024_UpperLowerDRRL_Liu,NeurIPS2021_BellmanEluderDim_Jin}.

\paragraph{General function approximation: robust extension of BE-style characterizations.}
In non-robust online RL with general function approximation, the sharpest comparisons are typically phrased in terms of intrinsic complexity (e.g., BE/DE dimension) rather than explicit $S,A$ dependence. In particular, the BE-dimension framework yields sample complexity of order $\widetilde{O}\!\big(H^3\,\dim_{\mathrm{BE}}\log(|\mathcal F|/\delta)\,\varepsilon^{-2}\big)$ for learning near-optimal policies in low-BE problems \citep{NeurIPS2021_BellmanEluderDim_Jin}. Motivated by BE-dimension analyses, we characterize exploration through the complexity of Bellman errors: we work with the \emph{robust} Bellman residual class and derive a regret bound that scales with the resulting robust BE dimension $\dimrobBE$ (\Cref{def:robust_BE_dim}).

A second natural comparator in the general-function online literature is work that incurs explicit \emph{coverage/concentrability} factors (often denoted $C_{\mathrm{cov}}$) in sample complexity. Table~\ref{tab:comparison} includes such a representative non-robust baseline with complexity $\widetilde{O}\!\big(C_{\mathrm{cov}}H^3\log(|\mathcal F|/\delta)\varepsilon^{-2}\big)$. Relative to these baselines, our contribution is not to ``beat'' non-robust learning in the worst case, but to show that robust learning remains sample-efficient under \emph{the same type of intrinsic characterization} (robust BE dimension), with an additional and interpretable $\sigma$-dependent multiplier that reflects the uncertainty set geometry (TV/$\chi^2$/KL).

\paragraph{Interpreting the $\sigma$-dependence: matching the qualitative ``robustness helps'' effect in tabular RMDPs.}
A central insight from recent tabular RMDP analyses under TV uncertainty is that robustness can fundamentally alter the statistical difficulty of learning: the minimax sample complexity depends explicitly on the uncertainty radius, and for constant uncertainty the problem can be strictly easier than standard (non-robust) MDP learning \cite{NeuRIPS2023_CuriousPriceDRRLGenerativeMdel_Shi, Arxiv2024_DRORLwithInteractiveData_Lu, Arxiv2024_UpperLowerDRRL_Liu}. This behavior is reflected in our TV specialization through the factor $\min\{H,1/\sigma\}$ in Table~\ref{tab:comparison}, where larger uncertainty induces contraction of the robust Bellman operator and leads to improved statistical rates. Recent work further establishes near-optimal rates for online tabular RMDPs with explicit uncertainty dependence \cite{Arxiv2024_DRORLwithInteractiveData_Lu, Arxiv2024_UpperLowerDRRL_Liu}, and our tabular reduction recovers the same qualitative scaling. The distinction lies in the analytical structure rather than the rate: existing results derive tabular guarantees through tabular-specific arguments, whereas our bound arises as a specialization of a unified robust BE-dimension framework by upper bounding $\dimrobBE = \mathcal O(S^2A^2)$, while preserving a formulation that extends naturally beyond the tabular setting.

For $\chi^2$ and KL uncertainty sets, existing tabular results often involve additional distribution-dependent quantities (e.g., variance or minimum transition mass). In contrast, we present bounds in a uniform and comparable form across divergences, isolating the uncertainty contribution as $(1+\sqrt{\sigma})$ for $\chi^2$ and $\sigma$ for KL. This representation is deliberate: it highlights how different $\phi$-divergences induce different degrees of smoothing in the dual formulation, and hence lead to quantitatively distinct statistical regimes, consistent with the divergence-sensitive comparisons emphasized in recent $\phi$-robust RL analyses \cite{Arxiv2024_ModelFreeRobustRL_Panaganti}.

\paragraph{Comparison with non-robust linear MDPs.}
In the standard (non-robust) online linear MDP setting, algorithms such as UCRL-VTR and its refinements \cite{PMLR2023_NearlyMinimaxOptimalRLLinearMDP_He} achieve the minimax regret rate $\widetilde{\mathcal O}(\sqrt{d^2_{\mathrm{lin}} H^3 K})$. In contrast, when specializing our general bound to linear {\RMDPf} for $\phi$-divergence set, we obtain the regret bound \(
\widetilde{\mathcal O}\!\bigl(d_{\mathrm{lin}}^2 H^2 B^2_{\phi}(\sigma)K\bigr),
\) (see Theorem \ref{thm:linear-regret-main}). For instance, the regret bound o TV-divergence is \(
\widetilde{\mathcal O}\!\bigl(d_{\mathrm{lin}}^2 H^4 (\min\{H,1/\sigma\})^2K\bigr)\) which matches the optimal dependence on the feature dimension $d_{\mathrm{lin}}$ and the uncertainty-dependent factor $\min\{H,1/\sigma\}$ that appears in recent robust linear analyses, while incurring a higher-order dependence on the horizon $H$ due to the need to control robust Bellman errors under function approximation. Accordingly, we do not claim minimax optimality in $H$, but emphasize that the dimension and uncertainty scalings are consistent with known robust lower bounds.

\paragraph{Comparison with linear RMDPs.}
Several recent works study DR-RL under linear structure but under settings that are not directly comparable to ours. Offline DR-RL with linear function approximation is studied in \cite{Arxiv2022_OfflineDRRLLinearFunctionApprox_Ma} and 
\cite{Arxiv2024_SampleComplexityOfflineLinearDRMDP_Wang} where value-estimation rates of order $\widetilde{\mathcal O}(\sqrt{d_{\mathrm{lin}}/N})$ or $\widetilde{\mathcal O}(\sqrt{d^3_{\mathrm{lin}}/N})$ are obtained depending on coverage assumptions and where $N$ denotes the number of trajectories. These results operate in an offline regime and do not address online exploration. In the \emph{online} setting, \citep{PMLR2024_LinearFunctionDRMDP_Liu} and \citep{Arxiv2024_UpperLowerDRRL_Liu} study $d_{\mathrm{lin}}$-rectangular linear RMDPs for TV-divergence set where the agent interacts online with a nominal (source) environment but the performance criterion is the worst-case value over a perturbed (target) environment, and attains regret rate $\widetilde{\mathcal{O}}\!\Bigl(\sqrt{d^2_{\mathrm{lin}}H^2(\min\{H,1/\sigma\})^2K}\Bigr)$ together with an information-theoretic lower bound that is optimal in $(d_{\mathrm{lin}},K,\sigma)$ up to a $\sqrt{H}$ factor~\citep{Arxiv2024_UpperLowerDRRL_Liu}. When specialized to the same $d_{\mathrm{lin}}$-rectangular linear setting, our framework yields the regret bound
\(
\widetilde{\mathcal O}\!\bigl(\sqrt{d_{\mathrm{lin}}^2 H^4 (\min\{H,1/\sigma\})^2 K}\bigr),
\) for the TV-divergence set (see Table~\ref{tab:comparison}). While this introduces an additional polynomial dependence on the horizon $H$ compared to the most refined linear-specific analyses, it recovers the same dependence on the feature dimension $d_{\mathrm{lin}}$ and the uncertainty-dependent contraction factor $\min\{H,1/\sigma\}$.

\begin{rem}
\cite{Arxiv2024_ModelFreeRobustRL_Panaganti} studies a hybrid $\varphi$-regularized RMDP under TV-divergence set that combines an offline dataset with online interactions. Under approximate value and dual realizability, a bilinear model of dimension $d_{\mathrm{lin}}$, and an offline concentratability coefficient $C(\pi^\star)$, they obtain a suboptimality bound of order $\mathcal{O}\!\Bigl(\max\{C(\pi^\star),1\}\,(\lambda + H)\sqrt{d_{\mathrm{lin}}^3H^2K}\Bigr)$. By contrast, for TV-divergences set we specialize our general theorem to a $d_{\mathrm{lin}}$-rectangular linear {\RMDPTV} and show that {\AlgonameTV} achieves robust regret $\widetilde{\mathcal{O}}\!\Bigl(\sqrt{H^4(\min\{H,1/\sigma\})^2\,d^2_{\mathrm{lin}}K}\Bigr)$, which is better than the one in \cite{Arxiv2024_ModelFreeRobustRL_Panaganti}. This comparison implies our algorithm achieves a tighter sample complexity, even without any prior collected offline dataset.

We also highlight that, conceptually, the two setups address different questions and are not directly comparable. \cite{Arxiv2024_ModelFreeRobustRL_Panaganti} analyze a \emph{regularized} robust objective in a hybrid offline--online regime, where the parameter $\lambda$ controls a trade-off induced by a $\varphi$-regularizer and the guarantees depend on offline coverage through $C(\pi^\star)$. In contrast, we study a \emph{constrained} {\RMDPf} in a purely online (off-dynamics) setting, where robustness is enforced via an explicit divergence ball of radius $\sigma$ around the nominal model and performance is measured by cumulative regret with respect to the unconstrained $\phi$-divergence robust value. Our general Theorem \ref{thm:regret_bound_RGOLF} applies to arbitrary parametric function classes.
\end{rem}

\section{Numerical Experiments}
\label{sec:numerical_exp}
\paragraph{Environment.}
We consider the standard \texttt{CartPole-v1} benchmark with a discrete action space.
The state $s \in \mathbb{R}^4$ contains the cart position, cart velocity, pole angle, and pole angular velocity, and the action space is $\mathcal{A}=\{0,1\}$, corresponding to applying a fixed horizontal force to the left or right.
Episodes terminate either when the pole falls beyond the allowed angle or when the time limit is reached (maximum horizon $H=500$).
Rewards are the standard per-step rewards from the environment, and the agent aims to maximize the undiscounted return over each episode.

\paragraph{Robustness evaluation.}
We are interested in how the learned policies behave under several kinds of mismatch between training and test conditions.
Policies are always \emph{trained on the nominal environment} and are evaluated under the following perturbation families, applied only at evaluation time:

\begin{itemize}
    \item \textbf{Action perturbation.}
    At each time step, with probability $\rho \in [0,1]$ the environment ignores the agent's action and instead executes a uniformly random action in $\mathcal{A}$.
    We evaluate over a grid of perturbation levels and, for the plots in the main text, we focus on
    \[
      \Gamma_{\text{act}} = \{0.3, 0.4, 0.5, 0.6, 0.7, 0.8, 0.9, 1.0\},
    \]
    where $\rho = 0$ corresponds to the nominal case (used internally for sanity checks but not always displayed in the figures).

    \item \textbf{Force-magnitude perturbation.}
    The horizontal push force applied in the dynamics is multiplied by a scalar factor $\eta_{\text{force}}$.
    We evaluate the learned policies on a finite set of scale values $\eta_{\text{force}} \in \Gamma_{\text{force}}$ that includes values 
     \[
      \Gamma_{\text{force}} = \{0.8, 0.7, 0.6, 0.5, 0.4, 0.3, 0.2, 0.1, 0.0\},
    \]
    where smaller values correspond to progressively weaker control inputs, and $\eta_{\text{force}} = 1.0$ is the nominal strength (used for training but not repeated in this sweep).

    \item \textbf{Pole-length perturbation.}
 The physical pole length is multiplied by a scalar factor $\eta_{\text{len}}$.
    At the configuration level, we specify the effective evaluation grid as
    \[
      \Gamma_{\text{len}} = \{0.25, 0.5, 0.75, 1.0, 1.25, 1.5, 1.75, 2.0\},
    \]
    covering shorter and longer poles relative to the nominal length.

\end{itemize}

In all settings, training is performed on the nominal environment $(\rho = 0, \eta_{\text{force}} = 1, \eta_{\text{len}} = 1)$, while robustness is measured by evaluating the final policy on perturbed environments from the above families.
Unless otherwise stated, each reported return corresponds to the average over $20$ evaluation episodes and $3$ independent random seeds $\{0,1,2\}$; we also plot $95\%$ confidence intervals computed across seeds and episodes.

\paragraph{Practical {\Algoname} agent.}
For CartPole we use a purely value-based implementation of {\Algoname} with a discrete action space.
The agent maintains two Q-networks $Q_1, Q_2$ (for Double Q-learning) and their target copies $\bar Q_1, \bar Q_2$, together with a dual network $g$ and its target copy $\bar g$. All networks are multilayer perceptrons with ReLU activations:

\begin{itemize}
 \item \textbf{Q-networks.}
We maintain two Q-networks $Q_1$ and $Q_2$. Each network takes the state $s \in \mathbb{R}^4$ as input and outputs a vector in $\mathbb{R}^{|\mathcal{A}|}$, one value per discrete action ($|\mathcal{A}| = 2$ for CartPole).
The architecture is a two-layer fully connected MLP with hidden sizes $(128,128)$ and ReLU activations, followed by a linear output layer. The scalar value $Q_i(s,a)$ is obtained by indexing the corresponding component of this output vector.

\item \textbf{Dual network.}
The dual function $g(s,a)$ is parameterized by a network with the same backbone as the Q-networks: it takes $s \in \mathbb{R}^4$ as input, passes it through two fully connected ReLU layers with $(128,128)$ units, and produces a vector in $\mathbb{R}^{|\mathcal{A}|}$, one value per action. The output is passed through a sigmoid and scaled so that $g(s,a) \in [0,10]$ for all $(s,a)$, enforcing non-negativity and preventing numerical blow-up in the dual updates.
\end{itemize}

\paragraph{Training protocol and robustness hyper-parameters.}
We have considered {\bf {\AlgonameTV}} to apply for the experiment. {\AlgonameTV} is trained off-policy on \texttt{CartPole-v1} using a replay buffer and an $\varepsilon$-greedy exploration strategy. Unless otherwise specified, we fix the discount factor to $\gamma = 0.99$ and use soft target updates with rate $\tau = 0.005$ for all target networks. Transitions are stored in a replay buffer of size $2\times 10^5$, from which we sample mini-batches of size $256$ and perform one gradient update per environment step. The Q-networks and dual network are optimized with Adam at a learning rate of $3\times 10^{-4}$. Exploration uses $\varepsilon$-greedy action selection, where the exploration rate is initialized at $\varepsilon_{\mathrm{start}} = 1.0$ and decayed linearly to $\varepsilon_{\mathrm{end}} = 0.05$ over the first $200$ episodes, and then held fixed at $0.05$ for the remainder of training. Each configuration is trained for $K = 500$ episodes, and we report performance statistics over three random seeds $\{0,1,2\}$. The robust {\AlgonameTV} backup is parameterized by a TV-radius $\sigma \in [0,1]$ and a slack parameter $\beta \ge 0$ that controls how strictly the dual constraint is enforced. On CartPole, we sweep $\sigma \in \{0.0,0.2,0.3,0.4,0.5,0.6\}$ and treat the slack
parameter $\beta$ as a scalar hyperparameter controlling how strictly we enforce
the dual Bellman constraint. After normalizing rewards and values so that the
dual residual has typical scale $\mathcal{O}(1)$, we sweep
$\beta \in \{0.0,0.5,1.0\}$, spanning hard ($\beta=0$) to moderately relaxed
($\beta=1$) constraints, and report the best-performing setting. In our
experiments, the best choice is $\beta=0.0$ under action perturbations and
$\beta=0.5$ under force-magnitude and pole-length perturbations.
The numerical values of all optimization hyper-parameters and network architectures are summarized in Tables~\ref{tab:cartpole-hparams} and~\ref{tab:cartpole-nets}.
\begin{table}[t]
\centering
\caption{Training hyper-parameters for {\AlgonameTV} on \texttt{CartPole-v1}.}
\label{tab:cartpole-hparams}
\begin{tabular}{lcc}
\toprule
\textbf{Parameter} & \textbf{Symbol} & \textbf{Value} \\
\midrule
Discount factor & $\gamma$ & $0.99$ \\
Target update rate & $\tau$ & $0.005$ \\
Replay buffer size & $|\mathcal{D}|$ & $2 \times 10^5$ transitions \\
Mini-batch size & $B$ & $256$ \\
Q-network learning rate & $lr_Q$ & $3 \times 10^{-4}$ \\
Dual-network learning rate & $lr_g$ & $3 \times 10^{-4}$ \\
Exploration start & $\varepsilon_{\text{start}}$ & $1.0$ \\
Exploration end & $\varepsilon_{\text{end}}$ & $0.05$ \\
Epsilon decay horizon & $T_{\varepsilon}$ & $200$ episodes \\
Gradient updates per step & -- & $1$ \\
Training episodes & $K$ & $500$ \\
Evaluation episodes per configuration & -- & $20$ \\
Random seeds & -- & $\{0,1,2\}$ \\
TV-robustness radii & $\sigma$ & $\{0.0,0.2,0.3,0.4,0.5,0.6\}$ \\
Slack parameter & $\beta$ & $\{0.0,0.5,1.0\}$ \\
\bottomrule
\end{tabular}
\end{table}

\begin{table}[t]
\centering
\caption{Network architectures for {\AlgonameTV} on CartPole.}
\label{tab:cartpole-nets}
\begin{tabular}{lc}
\toprule
\textbf{Network} & \textbf{Hidden layers} \\
\midrule
Q-network $Q_1,Q_2$  &
$(128, 128)$ (ReLU)\\
Dual network $g$ (default)  &
$(128, 128)$ (ReLU) \\
Dual network $g$ (capacity sweep) &
$(64,64)$/ $(128,128)$/ $(256,256)$ (ReLU) \\
\bottomrule
\end{tabular}
\end{table}

\paragraph{Practical {\AlgonameTV} update (CartPole, discrete).}
For completeness, Algorithm~\ref{alg:cartpole-rfltv} summarizes the training loop for the discrete practical {\AlgonameTV} agent used in the CartPole experiments.
The pseudocode follows our implementation: we use Double Q-learning with a dual network that approximates the robust inner optimization under total variation, and we incorporate the slack parameter $\beta$ by clipping the dual residual inside a quadratic penalty.

\subsection{{\AlgonameTV} vs. Functional Approximation Benchmarks: Gains Under Shift}

Figure~\ref{fig:RFL_TV_vs_FA} compares {\AlgonameTV} to three
function-approximation baselines: DQN, the value-function method GOLF~\cite{xie2022role}, and a dual-augmented variant GOLF-DUAL, which shares the same dual architecture as {\AlgonameTV} but is run
with $\sigma = 0$. All three baselines are trained without explicit distributional robustness and thus correspond to the non-robust ($\sigma=0$) setting. For {\AlgonameTV}, we fix the uncertainty radius
to the value that achieves the best nominal CartPole performance on our $\sigma$ grid, selecting $\sigma = 0.6$ for action perturbations, $\sigma = 0.5$ for force-magnitude perturbations, and $\sigma = 0.5$
for pole-length perturbations.

Across all three perturbation families, {\AlgonameTV} (with its best-performing $\sigma>0$) consistently dominates the non-robust functional approximation baselines.
Under {\it action perturbations}, for moderate noise levels $\rho \in [0.2,0.5]$, {\AlgonameTV} achieves roughly $30$–$60\%$ higher average return than DQN and about $15$–$30\%$ higher than the best non-robust value-based baseline, with performance at $\rho \approx 0.3$ nearly twice that of DQN. 
For {\it force-magnitude shifts} of $40$–$80\%$ from nominal, {\AlgonameTV} maintains average returns of roughly $150$–$400$, while DQN stays below about $260$ and GOLF/GOLF-DUAL lie mostly in the $60$–$380$ range, corresponding to roughly $\approx 1.5$–$3\times$ higher returns than DQN at severe shifts ($\ge 60\%$) and typically a $5$–$15\%$ gain over the GOLF baselines around the $40$–$50\%$ shift region.
For {\it pole-length changes} between $25\%$ and $200\%$ of nominal, {\AlgonameTV} stays near $500$ reward throughout, while the best non-robust baseline ranges between $\approx 330$ and $480$, yielding about 5–50\% higher return depending on the shift. Overall, for a fixed function class, turning on robustness in the Bellman update (via $\sigma>0$ and the dual term) yields substantially better robustness to both action noise and dynamics misspecification than any of the non-robust functional approximation baselines. These trends also highlight that robustness is inherently $\sigma$-dependent: for a fixed training robustness level, performance eventually degrades as the test-time perturbation grows, so maintaining high returns under stronger shifts typically requires training with a larger $\sigma$ and, in practice, possibly a more expressive function class.

\begin{figure*}[t]
\centering
\begin{subfigure}[b]{0.33\textwidth}
  \includegraphics[scale=0.352]{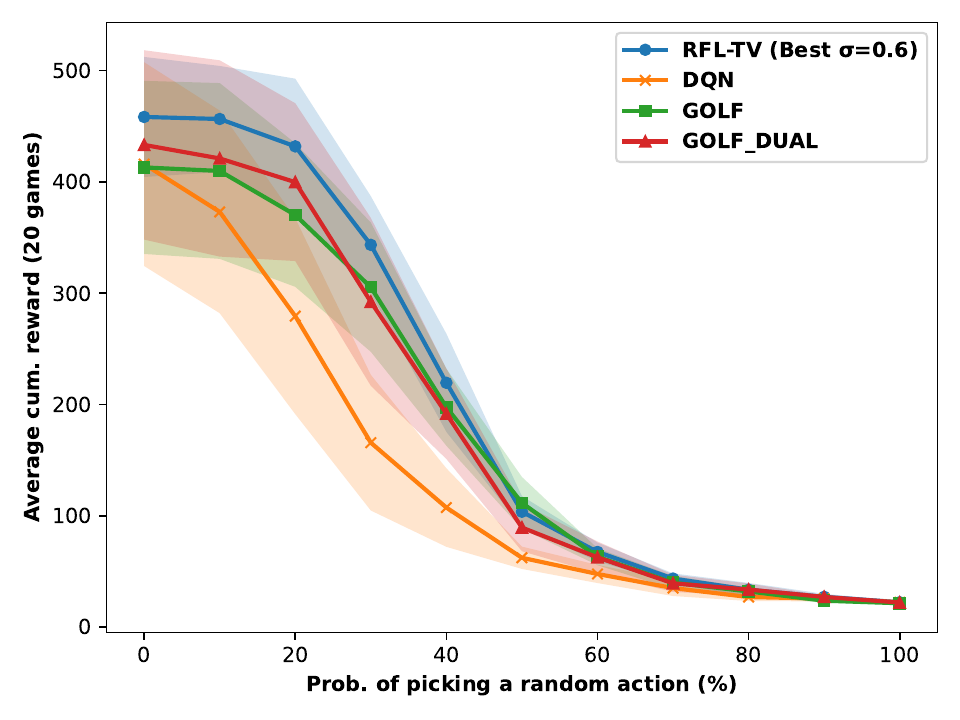}
  \caption{Action Perturbation}
  \label{fig:RFL_TV_FA_action}
\end{subfigure}
\hfill
\begin{subfigure}[b]{0.33\textwidth}
  \includegraphics[scale=0.352]{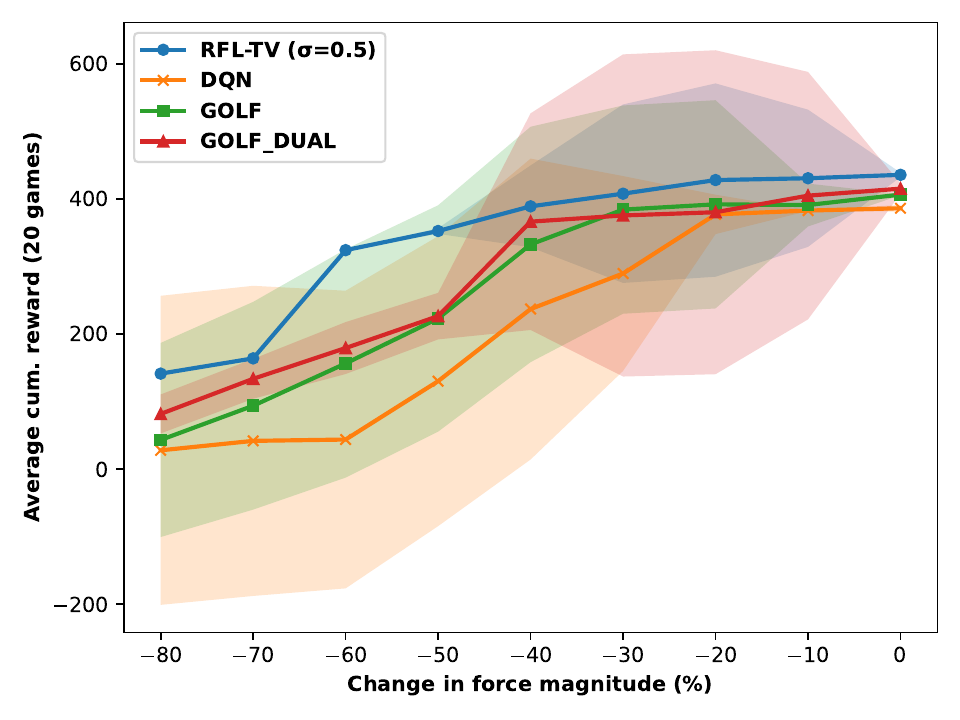}
  \caption{Force-magnitude Perturbation}
  \label{fig:RFL_TV_FA_force}
\end{subfigure}
\hfill
\begin{subfigure}[b]{0.33\textwidth}
  \includegraphics[scale=0.352]{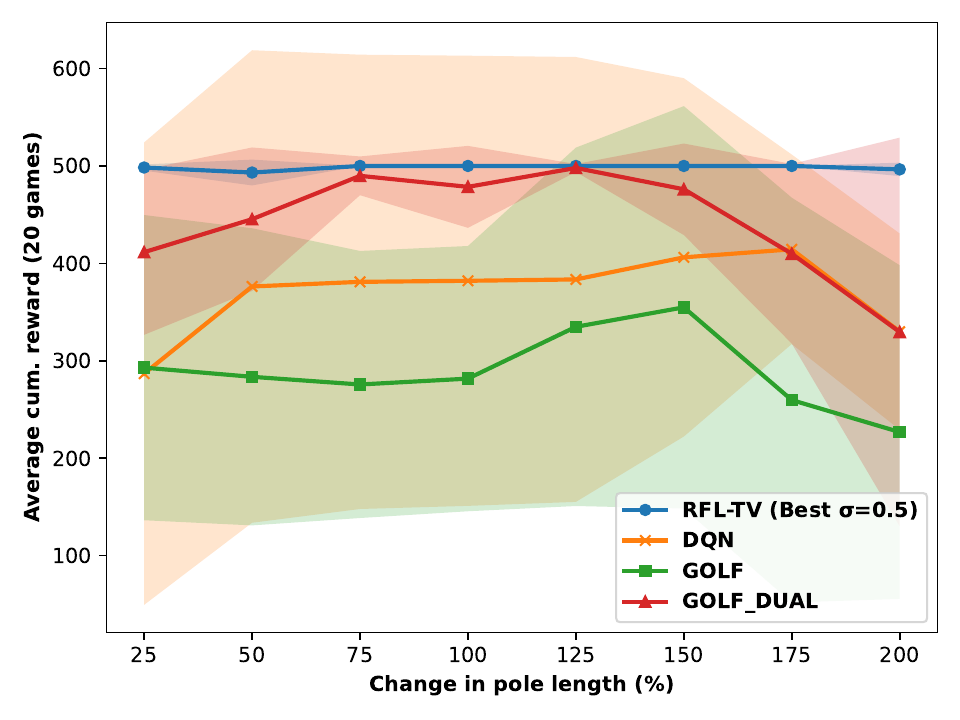}
  \caption{Pole-length Perturbation}
  \label{fig:RFL_TV_FA_length}
\end{subfigure}
\caption{{\AlgonameTV} vs. Functional Approximation Algorithms}
\label{fig:RFL_TV_vs_FA}
\end{figure*}

\subsection{{\AlgonameTV} vs. Online Tabular {\RMDPTV}}

Figure~\ref{fig:RFL_TV_vs_tabular} evaluates how closely our practical {\AlgonameTV} implementation matches an ideal TV-robust planner by comparing it to OPROVI-TV~\cite{Arxiv2024_DRORLwithInteractiveData_Lu}, a tabular algorithm that exactly solves the TV-robust Bellman equations for a given radius~$\sigma$. Although OPROVI-TV is restricted to small state spaces such as CartPole, it serves as a strong oracle-style baseline for TV-robust planning. In contrast, our practical {\AlgonameTV} implementation operates with neural function classes and sample-based updates, so its per-iteration computational cost depends on the network sizes, batch size, and action-space cardinality~$A$, but \emph{not} on the number of states~$S$, making it applicable to large-scale problems where typically $S \gg A$. Across action perturbations and dynamics perturbations (force magnitude and pole length), {\AlgonameTV} with $\sigma \in \{0.4,0.6\}$ consistently matches, and often exceeds the returns of OPROVI-TV at the same~$\sigma$.

For action perturbations (random-action probability $\rho \in [0.3,0.7]$), {\AlgonameTV} with $\sigma = 0.6$ achieves between roughly $100\%\text{ and }400\%$ higher average return than OPROVI-TV, while $\sigma = 0.4$ yields gains on the order of $30\%\text{--}200\%$ depending on the noise level; the two methods converge to similar near-random performance only as $\rho$ approaches~$1$. Under force-magnitude perturbations, {\AlgonameTV} with $\sigma = 0.6$ improves over OPROVI-TV by about $100\%\text{--}300\%$ at large changes ($40\%\text{--}80\%$ deviation from nominal), and $\sigma = 0.4$ still offers roughly $30\%\text{--}150\%$ gains. For pole-length perturbations, {\AlgonameTV} with $\sigma = 0.6$ maintains returns that are typically $150\%\text{--}300\%$ higher than the tabular baseline over most of the tested range, whereas $\sigma = 0.4$ yields about $30\%\text{--}150\%$ improvements. Overall, these trends indicate that a simple two-layer ReLU MLP (with $128\text{--}256$ hidden units for both Q and dual networks) can closely track—and often outperform—the robust value structure computed by an exact tabular TV-RMDP solver, while enjoying computational complexity that scales with network size and $A$ rather than $S$, which is particularly advantageous in regimes where $S \gg A$.

\begin{figure*}[!htb]
\centering
\begin{subfigure}[b]{0.33\textwidth}
  \includegraphics[scale=0.352]{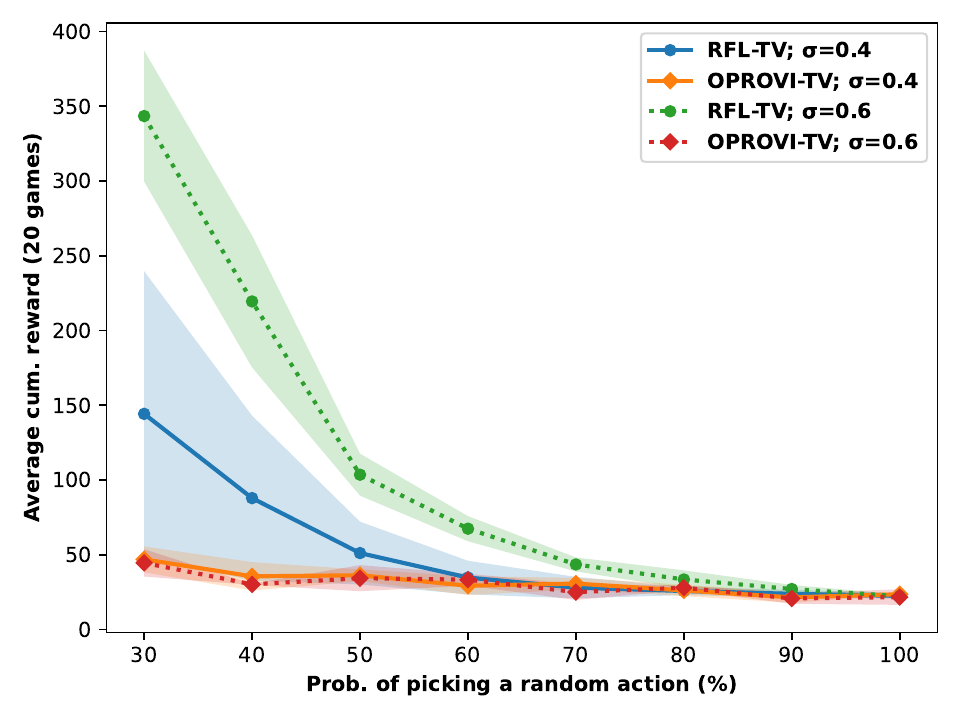}
  \caption{Action Perturbation}
  \label{fig:RFL_TV_tabular_action}
\end{subfigure}
\hfill
\begin{subfigure}[b]{0.33\textwidth}
  \includegraphics[scale=0.352]{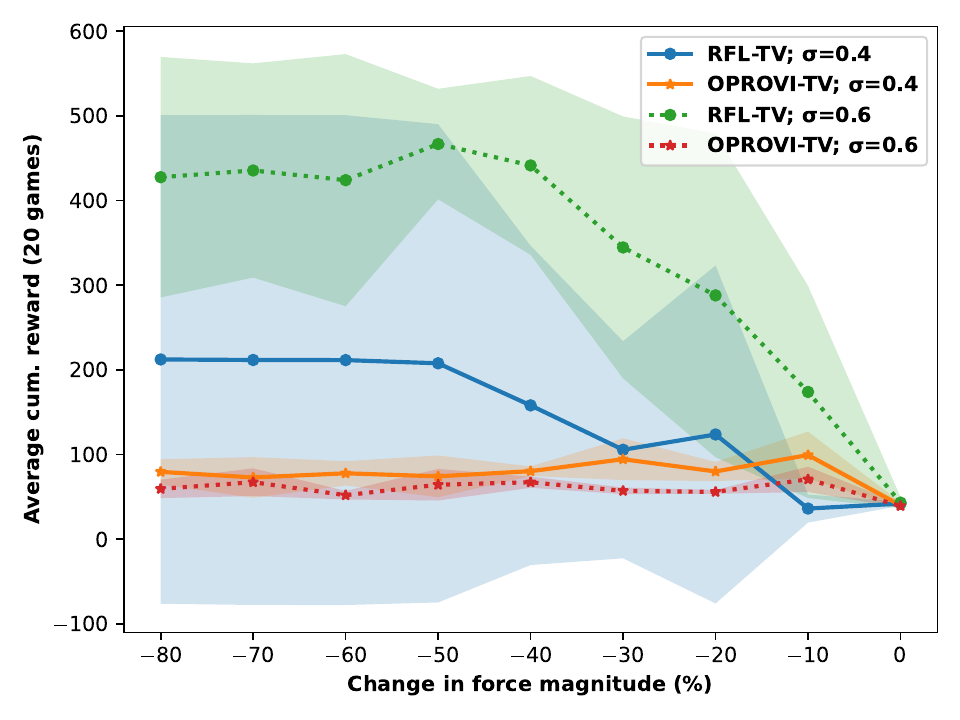}
  \caption{Force-magnitude Perturbation}
  \label{fig:RFL_TV_tabular_force}
\end{subfigure}
\hfill
\begin{subfigure}[b]{0.33\textwidth}
  \includegraphics[scale=0.352]{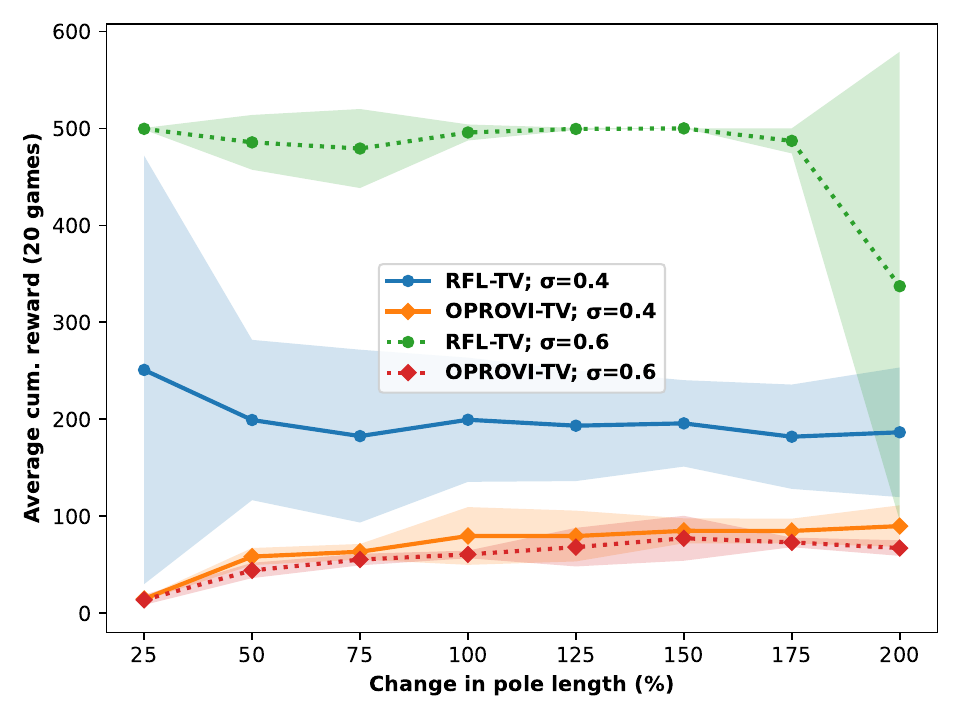}
  \caption{Pole-length Perturbation}
  \label{fig:RFL_TV_tabular_length}
\end{subfigure}
\caption{{\AlgonameTV} vs. OPROVI-TV (Tabular).}
\label{fig:RFL_TV_vs_tabular}
\end{figure*}

\subsection{Balancing Robustness Radius and Dual-Network Capacity}

Figure~\ref{fig:RFL_TV_vs_xi} examines how the TV robustness radius
$\sigma$ and the dual-network width $\xi_{\text{dual}}$ jointly shape the
performance of {\AlgonameTV}. For each perturbation family (action noise,
force–magnitude scaling, and pole–length scaling), we vary
$\xi_{\text{dual}}$ over two-layer MLPs with hidden sizes
$(64,64)$, $(128,128)$, and $(256,256)$ and evaluate {\AlgonameTV} for
$\sigma \in \{0.2,0.4,0.6\}$ at a representative perturbation level. Note that enlarging the dual hidden size can only decrease the approximation gap $\xi_{\mathrm{dual}}$ to the ideal dual optimizer; in other words, we can view the dual width as a structural knob that monotonically reduces the realizability constant $\xi_{\mathrm{dual}}$.
Across all three families, increasing the dual capacity markedly improves
robustness: moving from $(64,64)$ to $(256,256)$ yields roughly
$40\%\text{--}120\%$ higher average return under action perturbations,
about $50\%\text{--}180\%$ gains for force–magnitude shifts, and roughly
$100\%\text{--}250\%$ gains for pole–length perturbations. At any fixed
$\xi_{\text{dual}}$, larger robustness radii clearly help: compared to
$\sigma = 0.2$, using $\sigma = 0.6$ improves returns by about
$60\%\text{--}160\%$ under action noise, $30\%\text{--}80\%$ under
force–magnitude changes, and $50\%\text{--}150\%$ under pole–length
changes, with $\sigma = 0.4$ typically lying in between. This behaviour is natural: when $\sigma$ is too small, the uncertainty set remains close to the nominal dynamics and the dual term contributes less, so the policy tends to overfit to the unperturbed environment and degrades
sharply under shift. Larger radii ($\sigma \approx 0.4$–$0.6$), together with a sufficiently expressive dual network, force the optimizer to hedge against adversarial transitions, leading to policies that are more
conservative around failure modes yet still high-reward under the moderately perturbed environments we evaluate on. In practice, these results suggest a simple tuning recipe: increase $\xi_{\text{dual}}$
until the robust return curve flattens, and select $\sigma$ in a moderate range where performance gains saturate (here around $0.4$–$0.6$), thereby jointly controlling approximation quality and the
strength of robustness.

\begin{figure*}[t]
\centering
\begin{subfigure}[b]{0.33\textwidth}
  \includegraphics[scale=0.352]{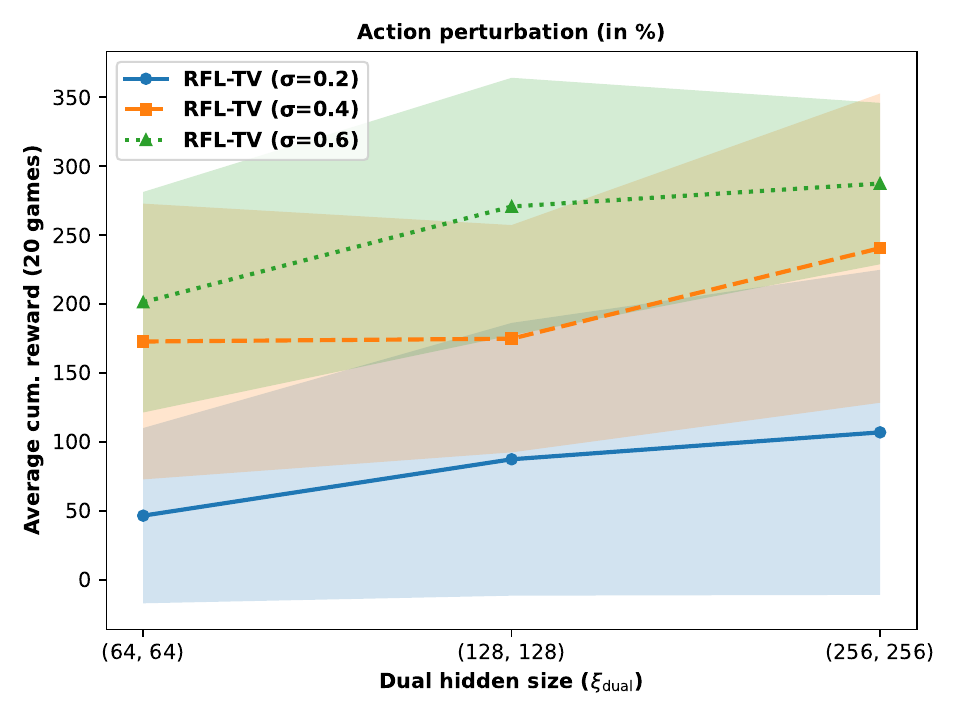}
  \caption{Action Perturbation}
  \label{fig:RFL_TV_xi_action}
\end{subfigure}
\hfill
\begin{subfigure}[b]{0.33\textwidth}
  \includegraphics[scale=0.352]{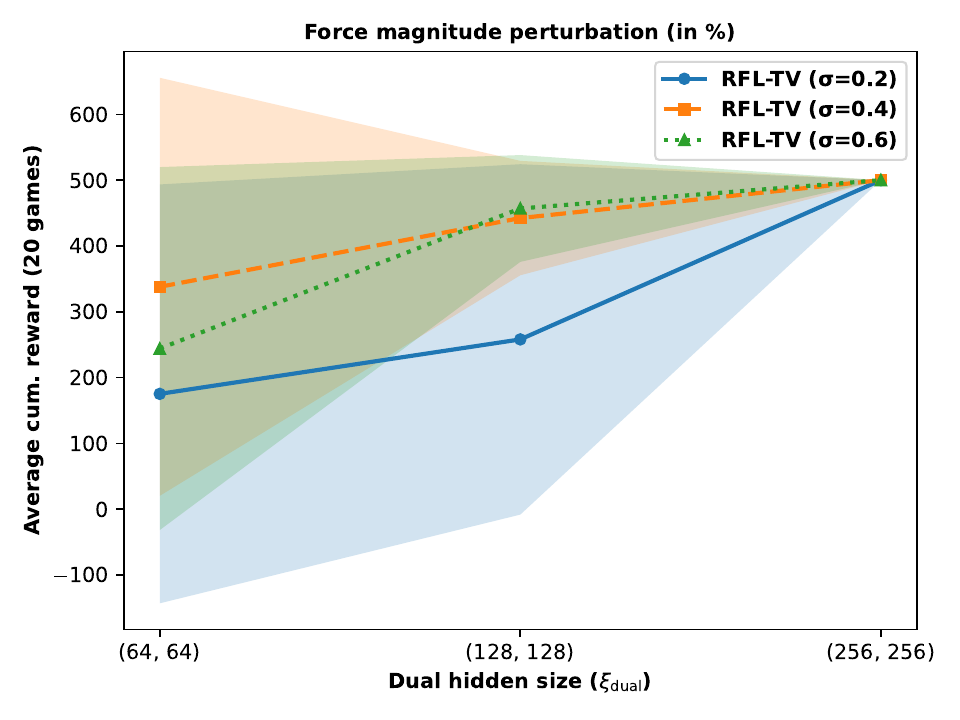}
  \caption{Force-magnitude Perturbation}
  \label{fig:RFL_TV_xi_force}
\end{subfigure}
\hfill
\begin{subfigure}[b]{0.33\textwidth}
  \includegraphics[scale=0.352]{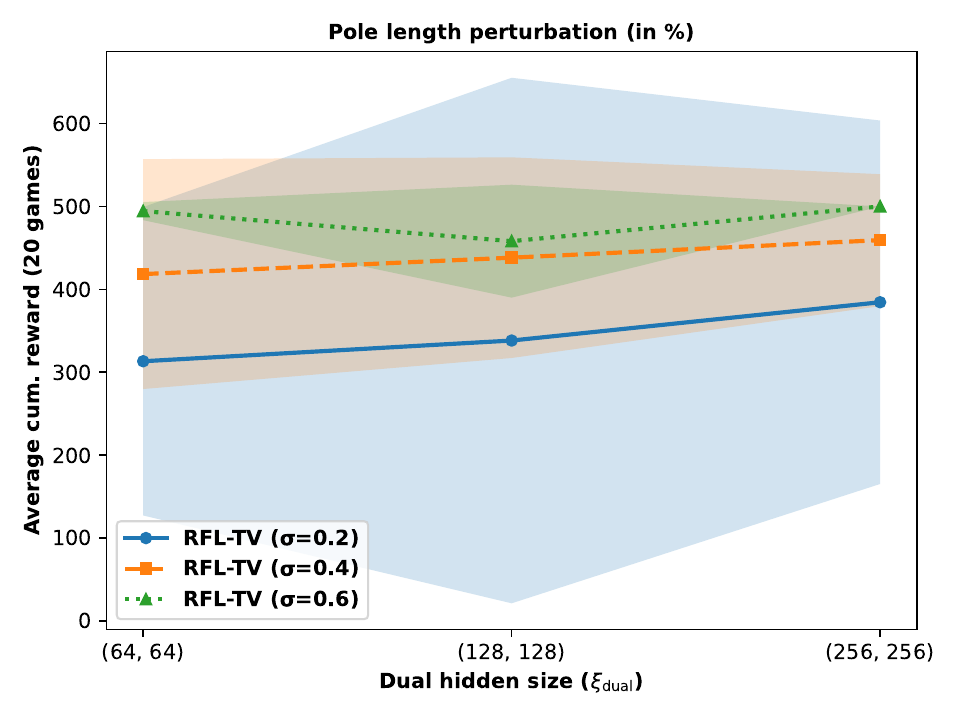}
  \caption{Pole-length Perturbation}
  \label{fig:RFL_TV_xi_length}
\end{subfigure}
\caption{{\AlgonameTV}: uncertainty level $\sigma$ vs. Uniform dual-approximation error $\xi_{\mathrm{dual}}$.}
\label{fig:RFL_TV_vs_xi}
\end{figure*}

\begin{algorithm}[H]
\caption{Practical {\AlgonameTV} for CartPole}
\label{alg:cartpole-rfltv}
\begin{algorithmic}[1]
\State \textbf{Inputs:} TV radius $\sigma$, slack $\beta$, discount $\gamma$, target rate $\tau$, batch size $B$, episodes $K$, horizon $H$, exploration schedule $(\varepsilon_{\text{start}}, \varepsilon_{\text{end}}, K_{\text{decay}})$.
\State Initialize replay buffer $\mathcal{D} \leftarrow \emptyset$.
\State Initialize Q-networks $Q_1,Q_2$ and dual network $g$; set target networks $\bar Q_i \gets Q_i$ for $i=1,2$ (and optionally $\bar g \gets g$).
\For{$k = 1,\dots,K$}
  \State Set $\varepsilon_k$ by linearly decaying from $\varepsilon_{\text{start}}$ to $\varepsilon_{\text{end}}$ over $K_{\text{decay}}$ episodes, then clamping.
  \State Reset environment and observe $s_0$.
  \For{$t = 0,\dots,H-1$}
    \State With prob.\ $\varepsilon_k$ sample $a_t$ uniformly; otherwise
    \(
      a_t = \arg\max_{a} \min\{ Q_1(s_t,a),\, Q_2(s_t,a) \}.
    \)
    \State Execute $a_t$, observe $(r_t,s_{t+1},d_t)$, and store $(s_t,a_t,r_t,s_{t+1},d_t)$ in $\mathcal{D}$.
    \If{$|\mathcal{D}| \ge B$}
      \State Sample minibatch $\{(s,a,r,s',d)\}_{j=1}^B$ from $\mathcal{D}$.\\
      \textbf{*** Target value (Double Q) ***}
      \State Compute $\bar Q_i(s',\cdot)$, $i=1,2$, and update
      \(
        v_{\text{next}}(s') = \max_{a'} \min\{ \bar Q_1(s',a'),\, \bar Q_2(s',a') \}.
      \)\\
      \textbf{*** Dual update with slack $\beta$ ***}
      \State Evaluate $g(s,a)$ and define
      \(
        \mathrm{dual\_term}(s,a)
        = \big(g(s,a) - v_{\text{next}}(s')\big)_+ - (1-\sigma)\,g(s,a).
      \)
      \State Compute residual
      \(
        r_{\text{dual}}(s,a) = \big|\mathrm{dual\_term}(s,a)\big| - \beta,\quad
        \tilde r_{\text{dual}}(s,a) = \max\{r_{\text{dual}}(s,a),0\},
      \)
      and minimize
      \(
        L_g = \mathbb{E}\big[\tilde r_{\text{dual}}(s,a)^2\big]
      \)
      w.r.t.\ the parameters of $g$ (one gradient step).\\
      \textbf{*** Q-update using updated $g$ ***}
      \State Recompute
      \(\mathrm{dual\_term}^{\text{new}}(s,a)
        = \big(g(s,a) - v_{\text{next}}(s')\big)_+ - (1-\sigma)\,g(s,a),
      \)
      and form targets
      \(
        y = r + (1-d)\,\gamma\big(v_{\text{next}}(s') + \mathrm{dual\_term}^{\text{new}}(s,a)\big).
      \)
      \State Compute $Q_i(s,a)$, $i=1,2$, and minimize
      \(
        L_Q = \mathbb{E}\big[(Q_1(s,a)-y)^2 + (Q_2(s,a)-y)^2\big]
      \)
      w.r.t.\ the parameters of $Q_1,Q_2$ (one gradient step).
      \State Soft-update targets: $\bar Q_i \gets (1-\tau)\,\bar Q_i + \tau\,Q_i,\quad i=1,2$,\\
      and optionally: $\bar g \gets (1-\tau)\,\bar g + \tau\,g$.
    \EndIf
    \If{$d_t = 1$}
      \State \textbf{break}
    \EndIf
  \EndFor
\EndFor
\State \textbf{Return} greedy policy: $\pi(s) = \arg\max_{a} \min\{ Q_1(s,a),\, Q_2(s,a) \}$.
\end{algorithmic}
\end{algorithm}

\section{Conclusion}
\label{sec:Conclusion}

In this work, we introduced {\Algoname}, a DR-RL algorithm with general function approximation for online settings. The algorithm implements a fitted robust Bellman update via a functional optimization and replaces state-action bonuses with a global uncertainty quantifier that more effectively guides exploration. Our theoretical analysis is grounded in the robust BE dimension, which we proposed to capture the intrinsic statistical complexity of learning robust value functions under function approximation. Our result yields strong sample-efficiency guarantees for large-scale problems, achieving sub-linear regret, and is independent of state/action spaces. When reduced to both tabular and linear RMDP cases, our results are both near-optimal against existing works and minimax lower bounds, which implies the tightness and efficiency of our algorithms. Our algorithm thus stands for the first purely online and sample-efficient algorithm for large-scale DR-RL, providing a theoretical foundation and a scalable algorithm for robust learning in high-dimensional environments.

\section{Appendix}
\label{sec:Appendix}
\subsection{Robust Bellman Rank and Relations with known tractable classes of Robust RL problems}
\label{sec:Robust_Bellman_Rank}
We define a robust analogue of Bellman rank (Q-type) to connect to known "low-rank" regimes. 

\begin{defi}[Robust Bellman rank (Q-type)]\label{def:rob-brank}
We say $\mathcal F$ has robust Bellman rank $d$ with normalization parameter $\gamma$ if for each $h$ there exist mappings $\varphi_h:\mathcal F\to\bbR^d$ and $\psi_h:\mathcal F\to\bbR^d$ such that for all
$f,f'\in\mathcal F$,
\begin{equation}
\label{eq:rob-brank}
\varepsilon_h^{\phi,\sigma}(f,f')
:=\bbE_{\pi_{f'}}\!\left[(f_h - \mathcal T_{h}^{\phi,\sigma}f_{h+1})(s_h,a_h)\right]
= \inner{\varphi_h(f)}{\psi_h(f')},
\end{equation}
and $\norm{\varphi_h(f)}_2\norm{\psi_h(f')}_2\le \gamma$.
The definition is similar to \cite{NeurIPS2021_BellmanEluderDim_Jin}[Definition~10] where we replace $\mathcal T_{h}$ by $\mathcal T_{h}^{\phi, \sigma}$.
\end{defi}

\begin{prop}[Low Robust Bellman Rank $\subset$ Low Robust BE Dimension]
\label{prop:rob-P11}
If an {\RMDPf} with function class $\mathcal F$ has robust Bellman rank $d$ with normalization parameter $\gamma$, then for all $\varepsilon>0$,
\begin{equation}
\label{eq:rob-P11}
\dimBE^{\mathrm{rob}}(\mathcal F,D_{\mathcal F},\varepsilon)
\;\le\; O\!\left(1+d\log\!\big(1+\gamma/\varepsilon\big)\right).
\end{equation}
\end{prop}

\noindent\textbf{Justification.} \Cref{prop:rob-P11} claims that the problems with low robust Bellman rank have low robust BE dimension to a multiplicative logarithmic factor in $\gamma$ and $\varepsilon^{-1}$. This proposition is the robust analogue of \cite{NeurIPS2021_BellmanEluderDim_Jin}[Proposition 11].

\noindent\textbf{Proof.}  The proof of \cite{NeurIPS2021_BellmanEluderDim_Jin}[Proposition~11]
uses only: (i) the bilinear factorization of the average Bellman error and (ii) norm bounds, and never uses any special structure of the nominal operator beyond appearing inside the residual.
Thus, the same determinant-growth proof applies verbatim after substituting
$\mathcal T_{h} \mapsto \mathcal T_{h}^{\phi,\sigma}$ and using Definition \ref{def:rob-brank}.
See \cite{NeurIPS2021_BellmanEluderDim_Jin}[Supp., App.\ D.1] for the full proof.

\begin{prop}[Low Eluder Dimension $\subset$ Low Robust BE Dimension]
\label{prop:rob-P12}
Under \Cref{ass:completeness} for all $h$, for all $\varepsilon>0$,
\begin{equation}
\label{eq:rob-P12}
\dimrobBE(\mathcal F,D_\Delta,\varepsilon)
\;\le\; \max_{h\in[H]} \dimE(\mathcal F_h,\varepsilon),
\end{equation}
where $\dimE(\mathcal G, \varepsilon)$ is the Eluder dimension which is the length of the longest sequence $\{x_1,\cdots,x_n\}\subset \mathcal X$ such that there exists $\varepsilon'>\varepsilon$
where $x_i$ is $\varepsilon'$-independent of $\{x_1,\cdots,x_{i-1}\}$ with respect to $\mathcal G$ for all $i \in [n]$.
\end{prop}

\noindent\textbf{Justification.} \Cref{prop:rob-P12} shows that any problem class with low Eluder dimension necessarily also has low robust BE dimension. This follows naturally from the completeness property and from the fact that the Eluder dimension is a special case of the DE dimension. This proposition is the robust analogue of \cite{NeurIPS2021_BellmanEluderDim_Jin}[Proposition 12].  

\noindent\textbf{Proof.}  The proof of \cite{NeurIPS2021_BellmanEluderDim_Jin}[Proposition~12]
reduces distributional independence for the residual class to Eluder independence within $\mathcal F_h$
by defining $g^i_{h}=\mathcal T_h f^i_{h+1}$ and invoking completeness to ensure $g^i_{h}\in\mathcal F_h$.
In the robust case, we define the residual class by eq. \ref{eq:robust_residual_class}, and  define $g^i_{h}=\mathcal T_{h}^{\phi,\sigma} f^i_{h+1}$ and invoke robust completeness (\Cref{ass:completeness}).
No other step changes. See \cite{NeurIPS2021_BellmanEluderDim_Jin}[Supp., App.\ D.2].

\begin{prop}[Low Robust BE Dimension $\nsubseteq$ Low Eluder Dimension $\cup$ Low Robust BEllman Rank]
\label{prop:rob-P13}
For any $m\in\bbN^+$, there exists an {\RMDPf} with function class $\mathcal F$ with $\phi$-divergence uncertainty set such that
for all $\varepsilon\in(0,1]$,
\[
\dimrobBE(\mathcal F,D_{\mathcal F},\varepsilon)=\dimBE^{\mathrm{rob}}(\mathcal F,D_{\Delta},\varepsilon)\le 5,
 \text{ but } \min\Big\{\min_h \dimE(\mathcal F_h,\varepsilon),\ \text{Robust Bellman Rank}\Big\}\ge m.
\]
\end{prop}

\noindent\textbf{Justification.} \Cref{prop:rob-P13} states that the class of problems with low robust Bellman–Eluder (BE) dimension is strictly broader than the union of (i) problems with low Eluder dimension and (ii) problems with low robust Bellman rank. In the non-robust setting, low BE dimension already captures additional models—such as kernel reactive POMDPs—that fall outside both the Bellman rank and Eluder dimension frameworks \cite{NeurIPS2021_BellmanEluderDim_Jin}[Appendix C]. The same inclusion continues to hold in the robust setting. This proposition is the robust analogue of \cite{NeurIPS2021_BellmanEluderDim_Jin}[Proposition 13].

\noindent\textbf{Proof.}  The separation construction in \cite{NeurIPS2021_BellmanEluderDim_Jin}[Proposition~13]
is a horizon-$1$ linear bandit instance. Since, $H=1$ removes any dependence on transition dynamics in the Bellman operator, adding a $\phi$-divergence ambiguity set does not affect the residual class or the BE-dimension calculation.
Thus, the same construction applies unchanged. See \cite{NeurIPS2021_BellmanEluderDim_Jin}[Supp., App.\ D.3].

\subsection{Proof of the main results} 
\label{sec:main_results}

Recall the robust Bellman operator as in eq. \ref{eq:Robust_Bellman_Operator_TV} as follows:
\begin{align}
\label{eq:revised_dual_Q_TV}
[\mathcal{T}^{\phi, \sigma}f](s,a) = r(s,a)-\inf_{\eta \ge 0,\ \nu \in \mathbb{R}}\mathbb{E}_{s'\sim P^\star_h(s,a)}\left[\eta\sigma-\nu + \eta\,\phi^\star\!\left(-\frac{\max_{a'}f(s',a')+\nu}{\eta}\right)\right]
\end{align}

And we define the empirical duality loss as:
\begin{align}
\label{eq:emp_loss_dual}
\widehat{\mathrm{DualLoss}}(g; f) = \sum_{(s,a,s')\sim\mathcal{D}} \bigg(g_{\eta}(s,a)\sigma-g_{\nu}(s,a) + g_{\eta}(s,a)\,\phi^\star\!\left(-\frac{\max_{a'}f(s',a')+g_{\nu}(s,a)}{g_{\eta}(s,a)}\right)\bigg),
\end{align} 

\paragraph{Special cases: TV, $\chi^2$, and KL}
For concreteness, we recall the resulting one-dimensional variational forms for three choices frequently used in robust RL; detailed derivations can be found in
\cite{AnnalsStat2022_TheoreticalUnderstandingRMDP_Yang, ICML2025_OnlineDRMDPSampleComplexity_He}. Under the $\mathcal{S} \times \mathcal{A}$-rectangularity assumption and eq. \ref{eq:revised_dual_Q_TV}, the robust expectation for any $V:\mathcal{S}\to[0,H]$ and $P^{\star}_h$ admits the following equivalent forms:

\begin{itemize}
\item \textbf{TV-divergence} ($\phi(t)=|t-1|$).
Under \Cref{ass:vanisin_minimal}, in this case, eq. \ref{eq:revised_dual_Q_TV} simplifies to 
\begin{align}
\label{eq:dual_TV}
\mathbb{E}_{\mathcal{U}^{TV,\sigma}_h(s,a)}[f]
=
-\inf_{\nu \in [0,2H/\sigma]}
\left\{\mathbb E_{s'\sim \mathbb{P}^\star_h(\cdot|s,a)}\big[\nu-\max_{a'}f(s',a')\big]_{+}
+ (1-\sigma)\nu\right\}.
\end{align}

\item \textbf{$\chi^2$-divergence} ($\phi(t)=(t-1)^2$).
One obtains a variance-sensitive form:
\begin{align}
\label{eq:dual_chi}
\mathbb{E}_{\mathcal{U}^{\chi^2,\sigma}_h(s,a)}[f]
=
-\inf_{\nu \in [0,H]}
\left\{\sqrt{\sigma\,\mathrm{Var}_{P_h^\star(\cdot|s,a)}\big(\nu-\max_{a'}f(s',a')\big)_{+}}
+ \mathbb E_{s'\sim \mathbb{P}^\star_h(\cdot|s,a)}\big[\max_{a'}f(s',a')-\nu\big]_{+}
\right\}.
\end{align}

\item \textbf{KL-divergence} ($\phi(t)=t\log t$).
The robust expectation can be written as
\begin{align}
\label{eq:dual_KL}
\mathbb{E}_{\mathcal{U}^{\sigma}_h(s,a)}[V]
=
-\inf_{\nu \in [\underline{\nu},\,H]}
\left\{\nu \log\!\Big( \mathbb E_{s'\sim \mathbb{P}^\star_h(\cdot|s,a)}\big[\exp\{-\max_{a'}f(s',a')/\nu\}\big]\Big)+\nu\sigma\right\},
\end{align}
where $\underline{\nu}>0$ is a regularity bound on the optimal dual variable, as commonly assumed in
\cite{NeuRIPS2023_DoublePessimismDROfflineRL_Blanchet, ICML2025_OnlineDRMDPSampleComplexity_He}.
\end{itemize}

\subsubsection{Proof of Lemma \ref{lem:equiv_loss_dual}}
\label{subsec:proof_Lem1}
\begin{proof}
Fix $f$. For notational convenience, define $\omega := (s,a)$ and $\Omega := \mathcal S\times\mathcal A$.
Define the pointwise integrand
\[
F(\omega;\eta,\nu)
:=
\mathbb E_{s'\sim P^\star(\cdot|\omega)}
\!\left[
\eta\sigma-\nu
+\eta\,\phi^\star\!\left(-\frac{\psi_f(s')+\nu}{\eta}\right)
\right],
\qquad \eta>0, \nu\in\mathbb R,
\]
and set $F(\omega;\eta,\nu)=+\infty$ whenever $\eta\le 0$.
Then we can rewrite
\begin{align}
\mathrm{DualLoss}(g;f)
&=\mathbb E_{\omega\sim\mu}\big[F(\omega; g_\eta(\omega),g_\nu(\omega))\big].
\label{eq:DL_integral_form}
\end{align}

\paragraph{Step 1: Applicability of Rockafellar--Wets Theorem.}
We apply Theorem~\ref{lem:rockafellar_Thm14.60} with $\mathcal X=\mathcal L^1(\mu;\mathbb R^2)$ and integrand
$f_{\mathrm{RW}}(\omega,x)=F(\omega;x)$ for $x=(\eta,\nu)\in\mathbb R^2$.

\emph{(i) Decomposability.}
By Remark~\ref{rem:examples_decomposable_sp}, $\mathcal L^1(\mu)$ is decomposable. Since
$\mathcal L^1(\mu;\mathbb R^2)=\mathcal L^1(\mu)\times \mathcal L^1(\mu)$, it is decomposable relative to $(\Omega,\Sigma(\Omega),\mu)$ in the sense of Definition~\ref{def:decomposable_sp}.

\emph{(ii) Normal integrand.}
For each fixed $(\eta,\nu)$ with $\eta>0$, the map
\[
(s',\omega)\mapsto \eta\sigma-\nu+\eta\,\phi^\star\!\left(-\frac{\psi_f(s')+\nu}{\eta}\right)
\]
is measurable since $\psi_f$ is measurable and $\phi^\star$ is lower semicontinuous.
Thus, $\omega\mapsto F(\omega;\eta,\nu)$ is measurable.

For each fixed $\omega$, continuity of $(\eta,\nu)\mapsto F(\omega;\eta,\nu)$ on $\eta>0$ follows from continuity of the perspective transform of $\phi^\star$ and dominated convergence, since $\psi_f\in[0,H]$ and the integrand is finite under standard boundedness assumptions on multipliers. Hence, $F$ is a normal integrand in the sense of Definition~\ref{def:decomposable_sp}. Therefore, Theorem~\ref{lem:rockafellar_Thm14.60} applies and yields
\begin{align}
\inf_{g\in\mathcal L^1(\mu;\mathbb R^2)} \mathrm{DualLoss}(g;f) =
\inf_{g}\mathbb E_{\omega\sim\mu}[F(\omega;g(\omega))] =
\mathbb E_{\omega\sim\mu}\Big[\inf_{(\eta,\nu)\in\mathbb R^2}F(\omega;\eta,\nu)\Big].
\label{eq:RW_interchange}
\end{align}

\paragraph{Step 2: Identification of the pointwise infimum.}
Fix $\omega=(s,a)$. By classical $\phi$-divergence duality (e.g., \cite{shapiro2017distributionally,duchi2021learning}), for any bounded measurable $V$,
\[
\inf_{P:\,D_\phi(P\|P^\star)\le \sigma}\mathbb E_P[V]
=
\inf_{\eta>0,\nu\in\mathbb R}
\left\{
\eta\sigma-\nu
+\eta\,\mathbb E_{P^\star}\!\left[\phi^\star\!\left(-\frac{V+\nu}{\eta}\right)\right]
\right\}.
\]
Applying this identity with $V=\psi_f$ and $P^\star=P^\star(\cdot|s,a)$ gives
\[
\inf_{(\eta,\nu)\in\mathbb R^2}F((s,a);\eta,\nu)
=
\inf_{P\in\mathcal U_h^{\phi,\sigma}(s,a)}\mathbb E_{s'\sim P(\cdot|s,a)}[\psi_f(s')].
\]

\paragraph{Step 3: Conclusion.}
Substituting the above into \eqref{eq:RW_interchange}, we conclude
\[
\inf_{g\in \mathcal{L}^1(\mu;\mathbb R^2)} \mathrm{DualLoss}(g; f)
=
\mathbb{E}_{(s,a)\sim \mu}
\left[
\inf_{P\in\mathcal{U}^{\phi,\sigma}_h(s,a)}\mathbb{E}_{s'\sim P(\cdot|s,a)}[\psi_f(s')]
\right],
\]
as claimed.
\end{proof}

\subsubsection{Proof of Theorem \ref{thm:regret_bound_RGOLF}}
\label{subsec:proof_Thm1}
\begin{proof}
We prove Theorem~\ref{thm:regret_bound_RGOLF} in the robust BE dimension (\Cref{def:robust_BE_dim}) framework.
Throughout, all expectations are taken under the distributions induced by the
algorithm's policies in \emph{nominal} $P^\star$, consistent with the definition of $\Pi_h$ and the DE dimension
\cite{NeurIPS2021_BellmanEluderDim_Jin} as given in \Cref{def:DE_dim}.

\paragraph{Step 1: Reduce robust regret to robust Bellman residuals under nominal transition kernel.} According to Assumption \ref{ass:completeness} and Lemma~\ref{lem:RVF_error_decomposition_BE}, we can guarantee $f^{(k)}$ is optimistic. Based on this optimistic algorithm, we will now relate the regret to the robust average Bellman error under the learner's sequence of policies. By Lemma~\ref{lem:RVF_error_decomposition_BE}, under Assumption~\ref{ass:completeness} and the confidence-set optimism
(which ensures $f^{(k)}$ is optimistic in each episode), we have
\begin{equation}
\label{eq:thm_step1_regret_to_eps}
\mathrm{Regret}(K)
\;\le\;
\sum_{k=1}^K\sum_{h=1}^H
\varepsilon^{\phi,\sigma}_{h}\!\big(f^{(k)},\pi^{(k)}\big),
\qquad \pi^{(k)}:=\pi^{f^{(k)}}.
\end{equation}
Recalling the definition of $\varepsilon^{\phi,\sigma}_{h}(f^{(k)},\pi^{(k)})$ as in eq. \ref{eq:robust_avg_Bellman_error_BE},
\[
\varepsilon^{\phi,\sigma}_{h}(f^{(k)},\pi^{(k)})
=
\mathbb{E}_{\pi^{(k)}}
\Big[f_h^{(k)}(s_h,a_h)-(\mathcal T_h^{\phi,\sigma} f_{h+1}^{(k)})(s_h,a_h)\Big],
\]
where state--actions at the step-$h$ is induced by executing $\pi^{(k)}$
in the nominal MDP $P^\star$.

\paragraph{Step 2: Add and subtract the dual-based operator.}
For each $(k,h)$, add and subtract the dual-based empirical robust Bellman backup
$\mathcal T_{h,\underline g_{f^{(k)}_{h+1}}}^{\phi,\sigma}$:
\begin{align}
\label{eq:thm_step2_decomp}
f_h^{(k)}(s,a)-(\mathcal T_h^{\phi,\sigma} f_{h+1}^{(k)})(s,a) =&
\underbrace{
\Big(f_h^{(k)}(s,a)-(\mathcal T_{h,\underline g_{f^{(k)}_{h+1}}}^{\phi,\sigma}f_{h+1}^{(k)})(s,a)\Big)
}_{\text{dual-based Bellman residual}}\nonumber\\
&+
\underbrace{
\Big((\mathcal T_{h,\underline g_{f^{(k)}_{h+1}}}^{\phi,\sigma}f_{h+1}^{(k)})(s,a)-(\mathcal T_h^{\phi,\sigma} f_{h+1}^{(k)})(s,a)\Big)
}_{\text{dual approximation error}} .
\end{align}
Plugging eq. \ref{eq:thm_step2_decomp} into eq. \ref{eq:thm_step1_regret_to_eps} yields
\begin{equation}
\label{eq:thm_step2_I_II}
\mathrm{Regret}(K) \;\le\; \mathrm{I} + \mathrm{II},
\end{equation}
where
\begin{align}
\label{eq:thm_def_I}
\mathrm{I}
&:=
\sum_{k=1}^K\sum_{h=1}^H
\mathbb{E}_{\pi^{(k)}}
\Big[
f_h^{(k)}(s_h,a_h)-(\mathcal T_{h,\underline g_{f^{(k)}_{h+1}}}^{\phi,\sigma}f_{h+1}^{(k)})(s_h,a_h)
\Big],\\
\label{eq:thm_def_II}
\mathrm{II}
&:=
\sum_{k=1}^K\sum_{h=1}^H
\mathbb{E}_{\pi^{(k)}}
\Big[
(\mathcal T_{h,\underline g_{f^{(k)}_{h+1}}}^{\phi,\sigma}f_{h+1}^{(k)})(s_h,a_h)-(\mathcal T_h^{\phi,\sigma} f_{h+1}^{(k)})(s_h,a_h)
\Big].
\end{align}

\paragraph{Step 3: Bound $\mathrm{I}$ via BE dimension (distributional Eluder dimension).}
Fix any step $h\in[H]$ and define the robust Bellman residual class
\[
\Xi_h^{\xi} := (\mathcal I-\mathcal T_h^{\phi,\sigma})\mathcal F
= \{ f_h-\mathcal T_h^{\phi,\sigma} f_{h+1} : f\in\mathcal F\},
\]
consistent with eq. \ref{eq:robust_residual_class}.
Let $\mu_k$ denote the distribution undr policy $\pi^{(k)}$ step-$h$ roll-in distribution under $P^\star$.

Define for each episode $k$ the robust residual function
\[
\xi_k(\cdot,\cdot)
:= f_h^{(k)}(\cdot,\cdot)-\mathcal T_{h,\underline g_{f^{(k)}_{h+1}}}^{\phi,\sigma} f_{h+1}(\cdot,\cdot)
\in \Xi_h^{\xi}.
\]
Then the step-$h$ contribution to $\mathrm{I}$ can be written as
\begin{equation}
\label{eq:I_as_phi_sum}
\mathrm{I}_h
:=\sum_{k=1}^K \mathbb E_{\mu_k}\big[\xi_k\big]
\quad\text{and}\quad
\mathrm{I}=\sum_{h=1}^H \mathrm{I}_h.
\end{equation}

We now invoke Lemma~\ref{lem:robust_lemma17}, which is a direct application of the distributional Eluder dimension summation argument of \cite{NeurIPS2021_BellmanEluderDim_Jin}[Lemma 17].
To apply Lemma~\ref{lem:robust_lemma17}, we need a squared-control condition of the form
$\sum_{t<k}\mathbb E_{\mu_t}[\xi_k^2]\le \beta$.
This is exactly ensured by Lemma~\ref{lem:robust-lemma39}(b) (with the same $\beta$ choice as in the theorem statement),
which provides a uniform control on the cumulative squared (dual-based) Bellman residuals; combined with
Assumption~\ref{ass:completeness} and the confidence-set validity, this yields the required $\beta$-type bound at each stage
(see Lemma~\ref{lem:robust-lemma39} and its proof template following \cite{NeurIPS2021_BellmanEluderDim_Jin}).

Therefore, applying Lemma~\ref{lem:robust_lemma17} at each fixed $h$ gives: for all $k\in[K]$,
\[
\sum_{i=1}^{k}\big|\mathbb E_{\mu_i}[\xi_i]\big|
\;\le\;
\mathcal O\!\Big(\sqrt{\dim_{\mathrm{DE}}(\Xi_h^{\xi},\Pi_h,1/k)\,\beta\,k}\Big).
\]
In particular, taking $k=K$ and summing over $h\in[H]$ yields
\begin{equation}
\label{eq:I_bound_final}
\mathrm{I}
\;\le\;
\sum_{h=1}^H
\mathcal O\!\Big(\sqrt{\dim_{\mathrm{DE}}(\Xi_h^{\xi},\Pi_h,1/K)\,\beta\,K}\Big).
\end{equation}
Equivalently, in terms of the robust BE dimension
$\dim_{\mathrm{BE}}^{\mathrm{rob}}(\mathcal F,\Pi,1/K):=\max_{h\in[H]}\dim_{\mathrm{DE}}(\Xi_h^{\xi},\Pi_h,1/K)$,
we may write
\begin{equation}
\label{eq:I_bound_BE}
\mathrm{I}
\;=\;
\mathcal O\!\Big(H\sqrt{\dim_{\mathrm{BE}}^{\mathrm{rob}}(\mathcal F,\Pi,1/K)\,\beta\,K}\Big).
\end{equation}

\paragraph{Step 4: Bound $\mathrm{II}$ via the TV dual optimization error lemma.}
For $\mathrm{II}$ we control, for each $(k,h)$,
\[
\Delta_{k,h}(s,a)
:=
(\mathcal T_{h,\underline g_{f^{(k)}_{h+1}}}^{\phi,\sigma} f^{(k)}_{h+1})(s,a)-(\mathcal T_h^{\phi,\sigma} f^{(k)}_{h+1})(s,a).
\]
Then by definition eq. \ref{eq:thm_def_II},
\[
\mathrm{II}
=
\sum_{k=1}^K\sum_{h=1}^H
\mathbb E_{\pi^{(k)}}\big[\Delta_{k,h}(s^k_h,a^k_h)\big]
\;\le\;
\sum_{k=1}^K\sum_{h=1}^H
\|\Delta_{k,h}\|_{1,\mu_h^{\pi^{(k)}}},
\]
where $\|\cdot\|_{1,\mu}$ denotes the $\ell_1(\mu)$ seminorm.

Now apply Lemma~\ref{lem:dual_opt_error_bound_BE} with $\pi=\pi^{(k)}$, $\mu_h^\pi$ as the distribution induced by $\pi^{(k)}$,
$f_{h+1}=f_{h+1}^{(k)}$, and dataset size $|\mathcal D_h^{(k)}|\geq kH$ (one transition per episode per stage).
Using a union bound over $(k,h)$ and the standard choice of failure probabilities, we obtain that with probability at least $1-\delta$,
\begin{equation}
\label{eq:II_step_dual}
\|\Delta_{k,h}\|_{1,\mu_h^{\pi^{(k)}}}
=
\mathcal O\!\left(
B_{\phi}(\sigma)\sqrt{\frac{\log\!\big(|\mathcal{F}_{h+1}||\mathcal{G}|KH/\delta\big)}{kH}}
+\frac{\varepsilon^{\mathrm{dual}}}{KH}
\right).
\end{equation}
Summing eq. \ref{eq:II_step_dual} over $k\in[K]$ and $h\in[H]$, and using
$\sum_{k=1}^K k^{-1/2}\le 2\sqrt K$, gives
\begin{equation}
\label{eq:II_bound_final}
\mathrm{II}
=
\mathcal O\!\left(
B_{\phi}(\sigma)\sqrt{KH\log\!\bigg(\frac{|\mathcal{F}||\mathcal{G}|KH}{\delta}\bigg)}
+\varepsilon^{\mathrm{dual}}\right)=\mathcal O\!\left(
\sqrt{HB_{\phi}(\sigma)\beta K}
+\varepsilon^{\mathrm{dual}}\right).
\end{equation}

\paragraph{Step 5: Combine the bounds.}
Combining eq. \ref{eq:thm_step2_I_II}, eq. \ref{eq:I_bound_BE}, and eq. \ref{eq:II_bound_final} yields that with probability at least $1-\delta$,
\[
\mathrm{Regret}(K)
\;\le\;
\mathcal O\!\Big(
H\sqrt{\dim_{\mathrm{BE}}^{\mathrm{rob}}(\mathcal F,\Pi,1/K)\,\beta\,K}
\;+\;
\sqrt{HB_{\phi}(\sigma)\beta K}
\;+\;
\varepsilon^{\mathrm{dual}}
\Big).
\]
Finally, by setting  $\beta=\mathcal{O}\bigg(B_{\phi}(\sigma)\log\!\big(|\mathcal{F}||\mathcal{G}|KH/\delta\big)\bigg)$ completes the proof.
\end{proof}

\subsubsection{Proof of Corollary \ref{cor:Sample_Complexity_bound_TV}}
For TV, we adopt the standard assumption.
\begin{assu}[Failure States]\label{ass:vanisin_minimal}
    For a {\RMDPTV}, there exists a set of failure states $\mathcal{S}_F\subseteq\mathcal{S}$, such that $r_h(s,a)=0$,  and $P^\star_h(s'|s,a)=0$, $\forall a\in\mathcal{A},\forall s\in\mathcal{S}_F,\forall s'\notin\mathcal{S}_F$. 
\end{assu}

\label{subsec:proof_cor1}
\begin{proof}
    Under \Cref{ass:multiplier-range-TV}, we will first find out the value of $B_{\phi}(\sigma)$ for each divergence. Then, applying the value of $B_{\phi}(\sigma)$ and $\varepsilon^{dual}=0$ in \Cref{thm:regret_bound_RGOLF}, we will find the sample-complexity bound for each case, as follows:
    \begin{itemize}
        \item {\bf TV-Divergence Case:} According to eq. \ref{eq:dual_TV}, we have
        \begin{align}
        \label{eq:TV_B_phi}
            l_{\mathrm{TV}}(f;s,a,s';\nu) \leq B_{\mathrm{TV}} := \mathcal{O}\big(H\min\{H,1/\sigma\}\big).
        \end{align}
       Applying eq. \ref{eq:TV_B_phi}, $\varepsilon^{dual}=0$ and $\dimrobBE\geq 1$ in \Cref{thm:regret_bound_RGOLF}, we get the sample-complexity bound as 
       \begin{align}
        \label{eq:sample_complexity_TV}
              T=KH= \mathcal{O}\bigg(\frac{H^5\big(\min\{H,1/\sigma\}\big)^2d\log\big(|\mathcal F||\mathcal G|T/\delta\big)}{\varepsilon^2}\bigg).
       \end{align}

               \item {\bf $\chi^2$-Divergence Case:} According to eq. \ref{eq:dual_chi}, we have
        \begin{align}
        \label{eq:chi_B_phi}
            l_{\chi^2}(f;s,a,s';\nu) \leq B_{\chi^2} := \mathcal{O}\big(H(1+\sqrt{\sigma}\big).
        \end{align}
       Applying eq. \ref{eq:chi_B_phi}, $\varepsilon^{dual}=0$ and $\dimrobBE\geq 1$ in \Cref{thm:regret_bound_RGOLF}, we get the sample-complexity bound as 
       \begin{align}
        \label{eq:sample_complexity_chi}
              T=KH= \mathcal{O}\bigg(\frac{H^5(1+\sqrt{\sigma})^2d\log\big(|\mathcal F||\mathcal G|T/\delta\big)}{\varepsilon^2}\bigg).
       \end{align}

        \item {\bf KL-Divergence Case:} According to eq. \ref{eq:dual_KL}, we have
        \begin{align}
        \label{eq:KL_B_phi}
            l_{\mathrm{KL}}(f;s,a,s';\nu) \leq B_{\mathrm{KL}} := \mathcal{O}\big(H\sigma\big).
        \end{align}
       Applying eq. \ref{eq:KL_B_phi}, $\varepsilon^{dual}=0$ and $\dimrobBE\geq 1$ in \Cref{thm:regret_bound_RGOLF}, we get the sample-complexity bound as 
       \begin{align}
        \label{eq:sample_complexity_KL}
              T=KH= \mathcal{O}\bigg(\frac{H^5\sigma^2d\log\big(|\mathcal F||\mathcal G|T/\delta\big)}{\varepsilon^2}\bigg).
       \end{align}
    \end{itemize}
    
\end{proof}

\subsection{Specialization to Linear {\RMDPf}}
\label{subsec:linear-RMDP}
We now show that our regret bound for general functional approximation specializes to a near--dimension-optimal
bound when the robust value function admits a linear representation, in the spirit of
the $d_{\mathrm{lin}}$-rectangular linear RMDP framework of \cite{Arxiv2022_OfflineDRRLLinearFunctionApprox_Ma} and
\cite{Arxiv2024_UpperLowerDRRL_Liu}.

\begin{assu}[$d_{\mathrm{lin}}$-Rectangular Linear {\RMDPf}]
\label{ass:linear-rmdp}
There exists a known feature map
$\boldsymbol{\kappa}_h : \mathcal S \times \mathcal A \to \mathbb R^d$ for each $h\in[H]$
with $\sum\limits_{i=1}^d\kappa_{h,i}(s,a) = 1$ and $\kappa_{h,i}(s,a)\geq 0$ for any $(i,s,a)\in [d]\times \mathcal{S}\times \mathcal{A}$ such that:
\begin{enumerate}
\item (Linear nominal model.)  
The reward and nominal kernel are linear:
\[
r_h(s,a) = {\boldsymbol\kappa}_h(s,a)^\top \boldsymbol\Omega_h, \qquad
P^\star_{h}(\cdot\mid s,a) = {\boldsymbol\kappa}_{h}(s,a)^\top\,{\boldsymbol\lambda}^\star_{h}(\cdot),
\]
for some unknown probability measures $\{{\boldsymbol\lambda}^\star_h\}_{h=1}^H$ over $\mathcal{S}$ and known
vectors $\{{\boldsymbol\Omega}_h\}_{h=1}^H$ with $\|{\boldsymbol\Omega}_h\|_2 \le \sqrt d_{\mathrm{lin}}$.

\item ($d_{\mathrm{lin}}$-rectangular $\phi$-divergence uncertainty set.)  
For each step $h$ and feature index $i\in[d_{\mathrm{lin}}]$ we can parameterize our uncertainty set $\mathcal{P}$ by $\{{\boldsymbol\lambda}^\star_h\}_{h=1}^H$, and thereby, can be defined as $\mathcal{P}=\mathcal{U}^{\phi,\sigma}(P^{\star})= \bigotimes_{(h,s,a)\in [H]\times \mathcal{S} \times \mathcal{A}}\mathcal{U}^{\phi,\sigma}_h(s,a; {\boldsymbol\lambda}^\star_h)$, where $\mathcal{U}^{\phi,\sigma}_h(s,a; {\boldsymbol\lambda}^\star_h)$ is defined as 
\begin{align*}
\mathcal{U}^{\phi,\sigma}_h(s,a; {\boldsymbol\lambda}^\star_h)\triangleq\bigg\{\sum\limits_{i=1}^d\kappa_{h,i}(s,a)\lambda_{h,i}(\cdot): \nu_{h,i}\in\Delta(\mathcal{S}) \text{ and } D_{\phi}(\lambda_{h,i},\lambda^\star_{h,i}(\cdot|s,a))\leq \sigma\bigg\}.
\end{align*}
\end{enumerate}
\end{assu}

This is the $\phi$-divergence analogue of the 
$d_{\mathrm{lin}}$-rectangular linear RMDP of
\cite{Arxiv2024_UpperLowerDRRL_Liu}[Sec. 3.2]. 

\paragraph{Linear function classes induced by the $d$-Rectangular linear {\RMDPf}.}
Under the linear {\RMDPf} structure in Assumption \ref{ass:linear-rmdp}, we specialize our general functional class $\mathcal{F}$ and dual functional class $\mathcal{G}$
used by {\Algoname} as linear function classes with a common feature map
$\kappa_h : \mathcal S\times\mathcal A \to \mathbb R^{d_{\mathrm{lin}}}$, and denote them as follows:
\begin{align}
\mathcal F^{lin} 
&:= \{\mathcal{F}^{lin}_h\}_{h=1}^H, \text{ where } \mathcal{F}^{lin}_h := \Bigl\{f_h : f_h(s,a) = {\boldsymbol\kappa}_h(s,a)^\top {\boldsymbol w}_h,\; {\boldsymbol w}_h\in\mathbb R^{d_{\mathrm{lin}}}
\Bigr\}, \label{eq:F-lin-def}\\
\mathcal G^{lin} &:= \{\mathcal{G}^{lin}_h\}_{h=1}^H, \text{ where } \mathcal{G}^{lin}_h := \Bigl\{g_h=(g_{\eta,h},g_{\nu,h}) : g_{\eta,h}(s,a)=\boldsymbol{\kappa}_h(s,a)^\top \boldsymbol u_{\eta,h},\;\nonumber\\
&\qquad \qquad \qquad g_{\nu,h}(s,a)=\boldsymbol{\kappa}_h(s,a)^\top \boldsymbol u_{\nu,h},\;
\boldsymbol u_{\eta,h},\boldsymbol u_{\nu,h}\in\mathbb R^{d_{\mathrm{lin}}}
\Bigr\}. \label{eq:G-lin-def}
\end{align}
The class $\mathcal F^{lin}$ is used to approximate robust $Q$--functions, while $\mathcal G^{lin}$ parameterizes the dual variables $(\eta,\nu)$ appearing in the $\phi$--robust Bellman operator (via the functional dual loss in Eq. \ref{eq:loss_dual})[See Sec.~\ref{sec:Dual_Reformulation} for the definition of the dual loss and its empirical counterpart].

\begin{lem}[Linear realizability and completeness]
\label{lem:linear-real}
Suppose the linear RMDP satisfies Assumption \ref{ass:linear-rmdp}. Then:
\begin{enumerate}
\item {\bf Linear realizability of $Q^{\pi,\sigma}$ and $Q^{\star,\sigma}$.} For any Markov policy $\pi$ and any $\sigma\ge 0$, there exist
vectors ${\boldsymbol w}^{\pi,\sigma}_1,\dots,{\boldsymbol w}^{\pi,\sigma}_H\in\mathbb R^{d_{\mathrm{lin}}}$ such that for all $h\in[H]$,
\begin{align}
Q^{\pi,\sigma}_h(s,a)
&= {\boldsymbol \kappa}_h(s,a)^\top \boldsymbol w^{\pi,\sigma}_h,
\qquad\forall (s,a)\in\mathcal S\times\mathcal A, \label{eq:Q-pi-linear}
\end{align}
and, in particular, for the robust-optimal policy $\pi^\star$ there exist
$\boldsymbol w^{\star,\sigma}_1,\dots,\boldsymbol w^{\star,\sigma}_H$ with
\begin{align}
Q^{\star,\sigma}_h(s,a)
&= \boldsymbol \kappa_h(s,a)^\top \boldsymbol w^{\star,\sigma}_h,
\qquad\forall (s,a),\; h\in[H]. \label{eq:Q-star-linear}
\end{align}
Hence $Q^{\pi,\sigma},Q^{\star,\sigma}\in\mathcal F^{lin}$.

\item {\bf Closure under the robust Bellman operator.}
Let $f\in\mathcal F^{lin}$ with component functions
$f_h(s,a)=\boldsymbol \kappa_h(s,a)^\top \boldsymbol w_h$. Then, for each $h\in[H]$ there exists
$\boldsymbol w'_h\in\mathbb R^{d_{\mathrm{lin}}}$ such that the robust Bellman backup satisfies
\begin{align}
[\mathcal{T}^{\phi,\sigma}_h f_{h+1}](s,a)
&= r_h(s,a) + \mathbb E_{P\in\mathcal U^{\phi,\sigma}_h(s,a)}
\bigl[ V_{h+1}(s') \bigr] \nonumber\\
&= \boldsymbol\kappa_h(s,a)^\top \boldsymbol w'_h,
\qquad\forall (s,a), \label{eq:T-sigma-linear}
\end{align}
so that $\mathcal{T}^{\phi,\sigma}_h f_{h+1}\subseteq \mathcal{F}^{lin}_h$ for all $h$.

\item {\bf Linear dual representation.}
Fix any $f\in\mathcal F^{\mathrm{lin}}$. The dual minimizer $g^*_f=(g_{\eta},g_{\nu}(s,a))\in\mathbb R^2$ that attains the pointwise $\phi$-divergence dual can be cosen. in $\mathcal G^{\mathrm{lin}}$, i.e., there exist $\boldsymbol u_{\eta,h}^f,\boldsymbol u_{\nu,h}^f\in\mathbb R^{d_{\mathrm{lin}}}, \forall h \in [H]$ such that for all
$(s,a)$ and $h\in[H]$,
\begin{align}
g^\star_{f,h}(s,a)
&:=\big(g^\star_{\eta,f,h}(s,a),\,g^\star_{\nu,f,h}(s,a)\big)
=
\Big(\boldsymbol \kappa_h(s,a)^\top \boldsymbol u_{\eta,h}^f,\;
      \boldsymbol \kappa_h(s,a)^\top \boldsymbol u_{\nu,h}^f\Big),
\qquad \forall (s,a),\; h\in[H].
\label{eq:g-star-linear}
\end{align}
Consequently, the dual realizability error $\varepsilon^{\mathrm{dual}}$ in
Assumption~\ref{ass:approx_dual_realizability} is zero whenever the dual class used by the algorithm
satisfies $\mathcal G^{\mathrm{lin}} \equiv \mathcal{L}^1(\mu^\pi;\mathbb R^2$.
\end{enumerate}
\end{lem}

\begin{proof} We will now proof Lemma \ref{lem:linear-real}.
\emph{(i) Linear realizability of $Q^{\pi,\sigma}$ and $Q^{\star,\sigma}$.} The linear RMDP literature
(e.g., \cite{Arxiv2022_OfflineDRRLLinearFunctionApprox_Ma}[Prop.~3.2 and Lem.~4.1] and \cite{Arxiv2024_UpperLowerDRRL_Liu}[Sec.~3.2]) implies that both the robust Bellman operator and the robust value functions preserve linearity in
${\boldsymbol \kappa}_h$, yielding eq. \ref{eq:Q-pi-linear}--\ref{eq:T-sigma-linear}, the nominal kernel and all kernels in the $\phi$-divergence uncertainty set are linear mixtures of the base measures $\{\boldsymbol \lambda_h\}_{h=1}^H$, and the reward is linear in $\boldsymbol \kappa_h$. 

\smallskip\noindent
\emph{(ii) Closure under $\mathcal{T}^{\phi,\sigma}_h$.}
Let $f\in\mathcal F^{lin}$ with
$f_{h+1}(s,a)=\boldsymbol \kappa_{h+1}(s,a)^\top \boldsymbol w_{h+1}$.  Define the value
$V_{h+1}(s)=\max\limits_{a\in\mathcal A} f_{h+1}(s,a)$. By the $d_{\mathrm{lin}}$-rectangular structure, any $P\in\mathcal U^{\phi,\sigma}_h(s,a)$
can be written as
\(
P(\cdot\mid s,a) = \sum_{i=1}^d \kappa_{h,i}(s,a)\,\lambda_{h,i}(\cdot)
\)
with $\lambda_{h,i}\in\mathcal U^\sigma_h(s,a;\boldsymbol \lambda^\star_h)$.  Thus,
\begin{align}
\inf_{P_h \in \mathcal{U}^{\phi,\sigma}_h(s,a)}\mathbb{E}_{s'\sim P_h(\cdot|s,a)}\left[V^{\pi,\sigma}_{h+1}(s')\right]
&=
\inf_{\lambda_{h,1},\dots,\lambda_{h,d}}
\sum_{i=1}^d \kappa_{h,i}(s,a)\,
\mathbb E_{s'\sim\mu_{h,i}}[V_{h+1}(s')] \\
&=
\sum_{i=1}^d \kappa_{h,i}(s,a)
\inf_{\lambda_{h,i}\in\mathcal U^{\phi,\sigma}_h(s,a;\boldsymbol \lambda^\star_h)}
\mathbb E_{s'\sim\nu_{h,i}}[V_{h+1}(s')] \\
&=
\sum_{i=1}^d \kappa_{h,i}(s,a)\,\zeta_{h,i}(\boldsymbol w_{h+1}),
\end{align}
where each scalar $\zeta_{h,i}(\boldsymbol w_{h+1})$ depends only on $V_{h+1}$ (and
hence on $\boldsymbol w_{h+1}$) and the local $\phi$-divergence ball at index~$i$.  We therefore
obtain
\begin{align}
[\mathcal{T}^{\phi,\sigma}_h f_{h+1}](s,a)
&=
\boldsymbol \kappa_h(s,a)^\top \boldsymbol \Omega_h +
\boldsymbol \kappa_h(s,a)^\top \boldsymbol\zeta_h(\boldsymbol w_{h+1})
=
\boldsymbol \kappa_h(s,a)^\top \boldsymbol w'_h,
\end{align}
with $\boldsymbol w'_h := \boldsymbol \Omega_h + \boldsymbol \zeta_h(\boldsymbol w_{h+1})$.  This yields eq. 
\ref{eq:T-sigma-linear} and shows that
$\mathcal{T}^{\phi,\sigma}_h f_{h+1} \subseteq \mathcal F^{lin}_h$.

\smallskip\noindent
\emph{(iii) Linear dual representation.}
Fix any $f\in\mathcal F^{\mathrm{lin}}$ and $(s,a,h)$. 
The $\phi$-divergence dual form of the robust Bellman operator
(Eq.~\ref{eq:phi-ball-dual}) expresses the inner worst-case expectation
as a two-dimensional convex optimization problem over scalar dual variables
$(\eta,\nu)$. In our functional formulation, this corresponds to optimizing over
dual functions 
\(
g_{f,h}(s,a)=(g_{\eta,h}(s,a),g_{\nu,h}(s,a))\in\mathbb R^2.
\)

Under the linear RMDP structure, the nominal transition kernel admits the
representation
\[
P_h^\star(\cdot\mid s,a)=\sum_{i=1}^{d_{\mathrm{lin}}}\kappa_{h,i}(s,a)\lambda^\star_{h,i},
\]
so the dual objective decomposes as a weighted combination over the base measures
$\{\lambda^\star_{h,i}\}_{i=1}^{d_{\mathrm{lin}}}$. As a consequence, the pointwise dual
minimizer can be chosen to decompose coordinate-wise: there exist scalars
$\{\eta^\star_{h,i},\nu^\star_{h,i}\}_{i=1}^{d_{\mathrm{lin}}}$ such that the optimal
dual function admits the representation
\[
g^\star_{f,h}(s,a)
=
\big(g^\star_{\eta,f,h}(s,a),\,g^\star_{\nu,f,h}(s,a)\big)
=
\Big(\boldsymbol\kappa_h(s,a)^\top\boldsymbol u^f_{\eta,h},\;
      \boldsymbol\kappa_h(s,a)^\top\boldsymbol u^f_{\nu,h}\Big),
\]
for some vectors $\boldsymbol u^f_{\eta,h},\boldsymbol u^f_{\nu,h}\in\mathbb R^{d_{\mathrm{lin}}}$
(see, e.g., the TV dual derivation in \cite{Arxiv2024_UpperLowerDRRL_Liu}[Sec.~3.2]).
Equivalently, $g_f^\star\in\mathcal G^{\mathrm{lin}}$ as defined in
eq.~\ref{eq:g-star-linear}. 

Therefore, the infimum in the dual representation is attained within the class
$\mathcal G^{\mathrm{lin}}$, and the dual realizability error
$\varepsilon^{\mathrm{dual}}$ in Assumption~\ref{ass:approx_dual_realizability}
is zero whenever the dual class used by the algorithm satisfies
$\mathcal G^{\mathrm{lin}}\equiv \mathcal{L}^1(\mu^\pi;\mathbb R^2)$).
\end{proof}

\begin{assu}[Finite linear covering]
\label{ass:linear-cover}
For $\varepsilon_0 = 1/K$, the union class
$\mathcal H = \mathcal F^{lin} \cup \mathcal G^{lin}$ admits a finite
$\varepsilon_0$-cover in $\|\cdot\|_\infty$ such that
\begin{align}
\label{eq:linear-cover-again}
\log N_{\mathcal H}(\varepsilon_0)
\;\le\;
c_od_{\mathrm{lin}}H\,\log\big(c_o K\big)
\end{align}
for some absolute constant $c_o>0$.
\end{assu}
This bound follows from standard metric-entropy results for linear predictors on a bounded domain (see, e.g., \cite{shalev2014understanding}[Thm.~14.5]). In our setting, the feature vectors satisfy the simplex constraints $\sum_i \kappa_{h,i}(s,a)=1$ and $\kappa_{h,i}(s,a)\ge0$ for all $(s,a,h)$, which immediately implies $|\kappa_h(s,a)|2\le 1$. Together with the fact that the parameter vectors of $\mathcal F^{lin}$ and $\mathcal G^{lin}$ are restricted to a bounded ball, this ensures that every function in the union class $\mathcal H = \mathcal F^{lin} \cup \mathcal G^{lin}$ behaves as a linear predictor in an ambient space of dimension $d^{lin}=dH$, yielding a covering-number bound of the form $\log N_{\mathcal H}(\varepsilon_0) \le c_od_{\mathrm{lin}}H \log\bigl(c_o/\varepsilon_0\bigr)$ for some absolute constant $c_o$.

\begin{thm}[Regret of {\Algoname} in linear {\RMDPf}]
\label{thm:linear-regret-main}
For any $\delta\in(0,1]$ and all linear classes,
$\mathcal{F}^{lin},\mathcal{G}^{lin}$ as in eq. \ref{eq:F-lin-def}--\ref{eq:G-lin-def}, we set
\(
\beta
=\mathcal{O}\Bigl(B_{\phi}(\sigma)d_{\mathrm{lin}}H\log\bigl( KH/\delta\bigr)\bigr)\Bigr)
\)
in {\Algoname}. Then, under Assumption \ref{ass:completeness}--\ref{ass:linear-cover} and setting  $\varepsilon^{\mathrm{dual}}=0$, with
probability at least $1-\delta$, it holds that
\begin{align*}
\mathrm{Regret}(K) \le \widetilde{\mathcal O}\Big(\sqrt{H^2d^2_{\mathrm{lin}}B^2_{\phi}(\sigma)K}\Big).
\end{align*}
In particular, plugging in the divergence-dependent envelope constants yields:
\begin{align}
\textbf{(TV)}\qquad
\mathrm{Regret}(K)
&\le
\widetilde{\mathcal O}\!\Big(\sqrt{
H^4d^2_{\mathrm{lin}}(\min\{H,\sigma^{-1}\})^2K}
\Big),
&&\text{where } B_{\mathrm{TV}}(\sigma)=\mathcal{O}\!\big(H\min\{H,\sigma^{-1}\}\big),\label{eq:regret-linear-TV}\\[2mm]
\textbf{($\chi^2$)}\qquad
\mathrm{Regret}(K)
&\le
\widetilde{\mathcal O}\!\Big(\sqrt{
H^4d^2_{\mathrm{lin}}(1+\sqrt{\sigma})^2K}
\Big),
&&\text{where } B_{\chi^2}(\sigma)=\mathcal{O}\!\big(H(1+\sqrt{\sigma})\big),\label{eq:regret-linear-chi}\\[2mm]
\textbf{(KL)}\qquad
\mathrm{Regret}(K)
&\le
\widetilde{\mathcal O}\!\Big(
\sqrt{H^4d^2_{\mathrm{lin}}\sigma^2K}\Big),
&&\text{where } B_{\mathrm{KL}}(\sigma)=\mathcal{O}\big(H\sigma\big).\label{eq:regret-linear-KL}
\end{align}
\end{thm}

\begin{proof}
We adapt the proof of Theorem~\ref{thm:regret_bound_RGOLF} to the linear $\phi$--RMDP setting under Assumption~\ref{ass:linear-rmdp}, with
linear function classes $\mathcal F^{lin},\mathcal G^{lin}$ as defined in eq.
\ref{eq:F-lin-def}--\ref{eq:G-lin-def}. In this work, the exploration is controlled intrinsically by the robust BE dimension (Definition~\ref{def:robust_BE_dim}).

\paragraph{Step 1: Starting point from the general regret proof (I-II decomposition).}
By the definition of robust regret eq. \ref{eq:Regret_K} and the same decomposition
steps as in the proof eq. \ref{eq:thm_step1_regret_to_eps}--\ref{eq:thm_step2_I_II} of Theorem~\ref{thm:regret_bound_RGOLF}, we have
\begin{equation}
\label{eq:lin_phi_regret_decomp}
\mathrm{Regret}(K)\;\le\; I + II,
\end{equation}
where:
(i) $I$ is the \emph{exploration / Bellman-residual term} controlled by the robust BE dimension (eq. \ref{eq:thm_def_I});
(ii) $II$ is the \emph{robust-operator approximation term} due to learning the robust Bellman operator via the dual ERM plug-in (eq. \ref{eq:thm_def_II}).

By eq. \ref{eq:I_bound_BE} and eq. \ref{eq:II_bound_final}, we can bound from $\mathrm I$ and $\mathrm II$, respectively, as
\begin{align}
I &\leq \mathcal{O}\bigg(H\sqrt{ \dimrobBE(\mathcal F,\mathcal D_{\mathcal F},1/\sqrt K)B_\phi(\sigma)K\log\!\big(|\mathcal{F}^{lin}||\mathcal{G}^{lin}|KH/\delta\big)},\label{eq:I_general_phi_BE}\\
II &\leq \mathcal O\!\left(
\sqrt{HB^2_{\phi}(\sigma)\log\!\big(|\mathcal{F}^{lin}||\mathcal{G}^{lin}|KH/\delta\big)K}
+\varepsilon^{\mathrm{dual}}\right),\label{eq:II_general_phi_dual}
\end{align}
where $B_\phi(\sigma)$ is the uniform envelope constant from Assumption~\ref{ass:multiplier-range-TV}, and $\varepsilon^{\mathrm{dual}}$ is the dual realizability bias
(Assumption~\ref{ass:approx_dual_realizability}).

\paragraph{Step 2: Linear $\phi$--RMDP consequences (structural assumptions and complexity terms).} For better clarity, we work under the exact dual realizability condition, and we set $\varepsilon^{\mathrm{dual}} = 0$ for simplicity of proof \footnote{By \Cref{lem:linear-real}(iii), when we instantiate {\Algoname} with the linear dual class $\mathcal G^{lin}$, the dual minimizer of the TV robust Bellman operator is exactly realizable, so the dual approximation error in \Cref{ass:approx_dual_realizability} vanishes and we have $\varepsilon^{\mathrm{dual}} = 0$. For clarity, we therefore focus on this exact-realizability case in the sequel. If one instead works with a dual class that only approximately realizes the optimal dual (so
$\varepsilon^{\mathrm{dual}} > 0$), the same proof strategy goes through with an additional additive term of order $\varepsilon^{\mathrm{dual}}$
propagating from the bound on $II$ (cf. \ref{eq:II_bound_final}) into the final regret
bound; no other part of the argument needs to be modified, and the dependence on $(K,d,H,\sigma)$ remains unchanged.}. Under \Cref{ass:linear-rmdp}, the linear classes $\mathcal{F}^{lin},\mathcal{G}^{lin}$ together with \Cref{lem:linear-real} guarantee that all structural assumptions used in \Cref{thm:regret_bound_RGOLF} remain valid when we instantiate the analysis with the linear {\RMDPf}; the only resulting changes are as follows:
\begin{itemize}
\item The complexity term $\log(|\mathcal{F}^{lin}||\mathcal{G}^{lin}|)$ is replaced by a covering-number bound for the union class
$\mathcal{H}\triangleq\mathcal{F}^{lin}\cup\mathcal{G}^{lin}$. By \Cref{ass:linear-cover}, for $\varepsilon_0=1/KH$, the union class
$\mathcal{H}=\mathcal{F}^{lin}\cup\mathcal{G}^{lin}$ admits an $\varepsilon_0$-cover in
$\|\cdot\|_\infty$ with $\log N_{\mathcal{H}}(\varepsilon_0)\le
c_0\,d_{\mathrm{lin}}H\,\log(c_0 K)$, 
for some absolute constant $c_0>0$ \citep{shalev2014understanding}. Therefore, using thsi fact we have $\log\!\big(|\mathcal{F}^{lin}||\mathcal{G}^{lin}|KH/\delta\big) \leq c_0\,d_{\mathrm{lin}}H\,\log(c_0 K) + \log(c_1 KH/\delta) = \mathcal{O}\Big(d_{\mathrm{lin}}H\log(KH/\delta)\Big)$
\item The dual bias term $\varepsilon^{\mathrm{dual}}$ drops out.
\end{itemize}

\paragraph{Step 3: Bounding $\mathrm{II}$ (robust operator approximation) in the linear case.}
The derivation of the general bound eq. \ref{eq:II_general_phi_dual} for II (Lemma \ref{lem:dual_opt_error_bound_BE}) 
uses ERM generalization bound Lemma \ref{lem:ERM_gen_bound} and a union bound over all episodes, time
steps, and function pairs $(f,g)\in\mathcal{F}^{lin}\times\mathcal{G}^{lin}$. In the linear case, we instead apply the same argument to a finite
$\varepsilon_0$-net of $\mathcal{H}$. 

More precisely, fix $\varepsilon_0=1/KH$ and let
$\mathcal{H}_0\subset\mathcal{H}$ be a minimal $\varepsilon_0$-net under $\|\cdot\|_\infty$,
such that $|\mathcal{H}_0| = N_{\mathcal{H}}(\varepsilon_0)$.
We then repeat the concentration analysis of Lemma \ref{lem:ERM_gen_bound}, but take the union bound over the finite set
$(k,h,\varphi)\in[K]\times[H]\times\mathcal{H}_0$ instead of $(k,h,f,g)\in[K]\times[H]\times\mathcal{F}\times\mathcal{G}$.
The approximation error between any $f\in\mathcal{H}$ and its nearest neighbor
$f'\in\mathcal{H}_0$ is at most $\varepsilon_0$ in $\|\cdot\|_\infty$ and
hence contributes only an $o(1)$ term in $K$ to the final regret bound, which we absorb into the big--$\mathcal{O}$ notation. Therefore, following the same steps of the proof of Lemma \ref{lem:dual_opt_error_bound_BE} and setting $\varepsilon^{\mathrm{dual}}=0$, we conclude that in the linear case
eq. \ref{eq:II_general_phi_dual} becomes
\begin{equation}
\label{eq:linear-II-H}
II \leq \mathcal O\!\left(
\sqrt{HB^2_{\phi}(\sigma)d_{\mathrm{lin}}K\log\!\big(KH/\delta\big)}\right).
\end{equation}

\paragraph{Step 4: Bounding $\mathrm{I}$ via robust BE dimension (linear case).}
By Definition~\ref{def:robust_BE_dim}, the exploration term is controlled by
the DE dimension of the robust residual class $(I-\mathcal T^{\phi,\sigma}_h)\mathcal F$
under the on-policy family $\mathcal D_{\mathcal F}$.
The proof of Theorem~\ref{thm:regret_bound_RGOLF} shows that, on the event that all confidence sets are valid,
eq. \ref{eq:I_general_phi_BE} holds  with $\dimrobBE(\mathcal F,\mathcal D_{\mathcal F},1/\sqrt K)$. In the linear case, the robust BE dimension is finite and satisfies 
\begin{equation}
\label{eq:robBE_linear}
\dimrobBE(\mathcal F^{lin},\mathcal D_{\mathcal F^{lin}},1/\sqrt K)
\;=\; \widetilde O(d_{\mathrm{lin}}),
\end{equation}
by the same linear-DE/BE arguments as in the non-robust setting
\cite{NeurIPS2021_BellmanEluderDim_Jin, wang_reinforcement_2020} and as summarized in Remark~\ref{rem:linear_RMDOP_BE}. Hence,
\begin{equation}
\label{eq:I_linear_phi}
I \leq \mathcal{O}\bigg(H\sqrt{ d^2_{\mathrm{lin}}B_\phi(\sigma)K\log\!\big(KH/\delta\big)}\bigg).
\end{equation}

\paragraph{Step 5: Combine $\mathrm{I}$ and $\mathrm{II}$.}
Combining eq. \ref{eq:lin_phi_regret_decomp}, eq. \ref{eq:linear-II-H}, and eq. \ref{eq:I_linear_phi}, we obtain
\begin{align}
\mathrm{Regret}(K)
\;&\le\; \mathcal{O}\bigg(H\sqrt{ d^2_{\mathrm{lin}}B_\phi(\sigma)K\log\!\big(KH/\delta\big)}\bigg) + \mathcal O\!\left(
\sqrt{HB^2_{\phi}(\sigma)d_{\mathrm{lin}}K\log\!\big(KH/\delta\big)}\right)\nonumber\\
& \leq \widetilde{\mathcal{O}}\bigg(\sqrt{ H^2d^2_{\mathrm{lin}}B^2_\phi(\sigma)K}\bigg).\label{eq:regret_linear_phi_final}
\end{align}

\paragraph{Step 6: Plugging in divergence-specific envelopes.}
We now obtain the regret bound for each divergences as follows:
    \begin{itemize}
        \item {\bf TV-Divergence Case:} According to eq. \ref{eq:TV_B_phi}, we have $B_{\mathrm{TV}} := \mathcal{O}\big(H\min\{H,1/\sigma\}\big)$. Putting this in eq. \ref{eq:regret_linear_phi_final}, we get
        \begin{align}
        \label{eq:TV_Regret_linear}
        \mathrm{Regret}(K)  \leq \widetilde{\mathcal{O}}\bigg(\sqrt{ H^4d^2_{\mathrm{lin}}(\min\{H,1/\sigma\})^2K}\bigg).
        \end{align}
 
        \item {\bf $\chi^2$-Divergence Case:} According to eq. \ref{eq:chi_B_phi}, we have $B_{\chi^2} := \mathcal{O}\big(H(1+\sqrt{\sigma}\big)$. Putting this in eq. \ref{eq:regret_linear_phi_final}, we get
        \begin{align}
        \label{eq:Chi_Regret_linear}
        \mathrm{Regret}(K)  \leq \widetilde{\mathcal{O}}\bigg(\sqrt{ H^4d^2_{\mathrm{lin}}(1+\sqrt{\sigma})^2K}\bigg).
        \end{align}

        \item {\bf KL-Divergence Case:} According to eq. \ref{eq:KL_B_phi}, we have $B_{\mathrm{KL}} := \mathcal{O}\big(H\sigma\big)$. Putting this in eq. \ref{eq:regret_linear_phi_final}, we get
        \begin{align}
        \label{eq:KL_Regret_linear}
        \mathrm{Regret}(K)  \leq \widetilde{\mathcal{O}}\bigg(\sqrt{ H^4d^2_{\mathrm{lin}}\sigma^2K}\bigg).
        \end{align}
    \end{itemize}
This concludes the proof.
\end{proof}

\subsection{Key Lemmas}
\label{subsec:Key_Lemmas}
\begin{keylem}[Robust value function error decomposition]
\label{lem:RVF_error_decomposition_BE}
Consider an RMDP with a $\Xi$-divergence uncertainty set.
For any $f=\{f_h\}_{h=1}^H\in\mathcal{F}$, let $\pi^f$ be the greedy policy induced by $f$, and we define $V^f := \mathbb{E}\!\left[f_1(s_1,\pi^f_1(s_1))\right]$, and $V^{\pi,Q} := \mathbb{E}_{a_{1:H}\sim \pi,\; s_{h+1}\sim Q_h}\!\left[\sum_{h=1}^H r_h(s_h,a_h)\right]$.
Let $\psi^f_{h+1}(s') := \max_{a'\in\mathcal{A}} f_{h+1}(s',a')$.  For any policy $\pi$ and stage $h$, define the  robust average Bellman error
\begin{align}
\label{eq:robust_avg_Bellman_error_BE}
\varepsilon^{\phi,\sigma}_h(f,\pi)
\;:=\;
\mathbb{E}\!\left[
f_h(s_h,a_h)-(\mathcal{T}^{\phi,\sigma}_h f_{h+1})(s_h,a_h)
\;\middle|\; a_h\sim \pi_h(\cdot|s_h),\; s_{h+1}\sim P_h^\star(\cdot|s_h,a_h)
\right],
\end{align}
where the expectation is taken over the trajectory distribution induced by executing $\pi$ in the nominal environment $P^\star$.
Then, under Assumption~\ref{ass:completeness} and the optimism property of $\{f^{(k)}\}_{k=1}^K$ ensured by the confidence sets,
the cumulative robust regret in eq. \ref{eq:Regret_K} satisfies
\begin{align}
\label{eq:regret_to_BE_errors}
\mathrm{Regret}(K)
\;\le\;
\sum_{k=1}^K\sum_{h=1}^H
\varepsilon^{\phi,\sigma}_h\!\big(f^{(k)},\pi^{(k)}\big),
\end{align}
where $\pi^{(k)}:=\pi^{f^{(k)}}$.
\end{keylem}

\begin{proof}
Fix any kernel $Q\in\mathcal{P}$ and any $f\in\mathcal{F}$. Recall by eq. \ref{eq:Robust_Bellman_Operator_TV} we have
\[
(\mathcal{T}_h^{\phi,\sigma} f_{h+1})(s,a)
=
r_h(s,a)+\inf_{P\in\mathcal{U}^{\phi,\sigma}_h(s,a)} 
\mathbb{E}_{s'\sim P(\cdot|s,a)}\!\left[V^f_{h+1}(s')\right].
\]
Therefore, for every $(s,a)$,
\begin{align}
\label{eq:RVF_BE_step1}
(\mathcal{T}_h^{\phi,\sigma} f_{h+1})(s,a)
\;\le\;
r_h(s,a)+\mathbb{E}_{s'\sim Q_h(\cdot|s,a)}\!\left[V^f_{h+1}(s')\right],
\end{align}
which implies
\begin{align}
\label{eq:RVF_BE_step2}
f_h(s,a)-(\mathcal{T}_h^{\phi,\sigma} f_{h+1})(s,a)
\;\ge\;
f_h(s,a)-r_h(s,a)-\mathbb{E}_{s'\sim Q_h(\cdot|s,a)}\!\left[V^f_{h+1}(s')\right].
\end{align}

Now take expectation along a trajectory generated by executing $\pi^f$ in the environment with transition kernel $Q$:
\begin{align}
\label{eq:RVF_BE_step3}
\sum_{h=1}^H 
\mathbb{E}^{\pi^f,Q}\!\left[
f_h(s_h,a_h)-(\mathcal{T}_h^{\phi,\sigma} f_{h+1})(s_h,a_h)
\right]
\;\ge\;
\sum_{h=1}^H 
\mathbb{E}^{\pi^f,Q}\!\left[
f_h(s_h,a_h)-r_h(s_h,a_h)-\mathbb{E}_{Q_h}[V^f_{h+1}]
\right],
\end{align}
where $\mathbb{E}^{\pi^f,Q}$ denotes expectation over trajectories with $a_h\sim \pi^f_h(\cdot|s_h)$ and $s_{h+1}\sim Q_h(\cdot|s_h,a_h)$.

The right-hand side of eq. \ref{eq:RVF_BE_step3} admits a standard telescoping argument (cf.\ \cite{jiang2017contextual}[Lemma~1]):
\begin{align}
\label{eq:RVF_BE_step4}
\sum_{h=1}^H 
\mathbb{E}^{\pi^f,Q}\!\left[
f_h(s_h,a_h)-r_h(s_h,a_h)-\mathbb{E}_{Q_h}[V^f_{h+1}]
\right]
=
V^f - V^{\pi^f,Q}.
\end{align}
Combining eq. \ref{eq:RVF_BE_step3} and eq. \ref{eq:RVF_BE_step4} yields
\begin{align}
\label{eq:RVF_BE_step5}
\sum_{h=1}^H 
\mathbb{E}^{\pi^f,Q}\!\left[
f_h(s_h,a_h)-(\mathcal{T}_h^{\phi,\sigma} f_{h+1})(s_h,a_h)
\right]
\;\ge\;
V^f - V^{\pi^f,Q}.
\end{align}

Now specialize to the \emph{worst-case} kernel $Q=P^\omega(\pi^f)$ for policy $\pi^f$, i.e., for each $(s,a,h)$,
\[
\mathbb{E}_{s'\sim P^\omega_h(\cdot|s,a)}\!\left[V^f_{h+1}(s')\right]
=
\inf_{P\in\mathcal{U}_h^{\phi,\sigma}(s,a)} \mathbb{E}_{s'\sim P(\cdot|s,a)}\!\left[V^f_{h+1}(s')\right].
\]
In this case, eq. \ref{eq:RVF_BE_step1} holds with equality pointwise, and hence eq. \ref{eq:RVF_BE_step5} tightens to
\begin{align}
\label{eq:RVF_BE_step6}
V^f - V^{\pi^f,P^\omega}
=
\sum_{h=1}^H 
\mathbb{E}^{\pi^f,P^\omega}\!\left[
f_h(s_h,a_h)-(\mathcal{T}_h^{\phi,\sigma} f_{h+1})(s_h,a_h)
\right].
\end{align}

Finally, apply eq. \ref{eq:RVF_BE_step6} episode-by-episode with $f=f^{(k)}$ and $\pi^{(k)}=\pi^{f^{(k)}}$.
By optimism of $f^{(k)}$ (guaranteed by the validity of the confidence sets under Assumption~\ref{ass:completeness}),
we have $V^{\star,\sigma}_1(s_1^k)\le V^{f^{(k)}}_1(s_1^k)$.
Moreover, by definition of the robust value, $V^{\pi^{(k)},\sigma}_1(s_1^k)$ is the value of $\pi^{(k)}$ under the robust Bellman recursion induced by $\mathcal{T}^\sigma$.
Therefore,
\begin{align*}
\mathrm{Regret}(K)
&=\sum_{k=1}^K\Big(V^{\star,\sigma}_1(s_1^k)-V^{\pi^{(k)},\sigma}_1(s_1^k)\Big)\\
&\le \sum_{k=1}^K\Big(V^{f^{(k)}}_1(s_1^k)-V^{\pi^{(k)},\sigma}_1(s_1^k)\Big).
\end{align*}

It remains to relate $V^{f^{(k)}}_1(s_1^k)-V^{\pi^{(k)},\sigma}_1(s_1^k)$ to the robust Bellman residuals of $f^{(k)}$.
Using the identity eq. \ref{eq:RVF_BE_step6} (which holds for the robust Bellman operator $\mathcal{T}^{\phi,\sigma}$) and a standard telescoping argument along the trajectory generated by executing $\pi^{(k)}$, we obtain
\[
V^{f^{(k)}}_1(s_1^k)-V^{\pi^{(k)},\sigma}_1(s_1^k)
=
\sum_{h=1}^H
\mathbb{E}\!\left[
f^{(k)}_h(s_h^k,a_h^k)-(\mathcal{T}_h^{\phi,\sigma} f^{(k)}_{h+1})(s_h^k,a_h^k)
\;\middle|\; a_h^k\sim \pi_h^{(k)}(\cdot|s_h^k)
\right],
\]
and hence,
\begin{align*}
\mathrm{Regret}(K)
&\le
\sum_{k=1}^K\sum_{h=1}^H
\mathbb{E}\!\left[
f^{(k)}_h(s_h^k,a_h^k)-(\mathcal{T}_h^{\phi,\sigma} f^{(k)}_{h+1})(s_h^k,a_h^k)
\;\middle|\; a_h^k\sim \pi_h^{(k)}(\cdot|s_h^k)
\right]\\
&=
\sum_{k=1}^K\sum_{h=1}^H
\varepsilon^{\phi,\sigma}_{h}\!\big(f^{(k)},\pi^{(k)}\big),
\end{align*}
where $\varepsilon^{\phi,\sigma}_{h}(f,\pi)$ is defined in eq. \ref{eq:robust_avg_Bellman_error_BE}.
This completes the proof.
\end{proof}

\begin{keylem}
\label{lem:robust-lemma39}
Suppose Assumption \ref{ass:completeness} holds. Then if $\beta>0$ is selected
as in Theorem \ref{thm:regret_bound_RGOLF}, then with probability at least $1-\delta$, for all $k\in[K]$, {\Algoname} satisfies
\begin{enumerate}
  \item[(a)] $Q^{\star,\sigma} \in \mathcal F^{(k)}$.
  \item[(b)] \(\sum_{t=1}^{k-1}\mathbb E_{(s,a)\sim \pi^t}\Bigg[\Bigg(f^{(k)}_h(s,a)-\left[\mathcal{T}^{\phi,\sigma}_{h,\underline{g}_{f^{(k)}_{h+1}}}f^{(k)}_{h+1}\right](s,a)\Bigg)^2\Bigg]\le \mathcal O(\beta).
    \)
\end{enumerate}
\end{keylem}

\begin{proof}
The proof follows the same structure as the non-robust argument \cite{NeurIPS2021_BellmanEluderDim_Jin}[Lemma 39 and 40] and \cite{xie2022role}[Lemma 15] (martingale concentration via Freedman's inequality plus a finite cover of the functional class), with two robust-specific ingredients: (i) the dual scalar representation of the TV worst-case expectation and (ii) the use of the dual pointwise integrand as a sample target. We derive the complete proof as follows.

\begin{itemize}
    \item[\ding{43}] {\it Proof of ineq. (b)} To show ineq. (b), we will focus on the proof-lines of \cite{NeurIPS2021_BellmanEluderDim_Jin}[Lemma 39] and \cite{xie2022role}[Lemma 15 (2)]. We first fix $(k,h,f)$ tuple, where an episode $k$ we consider a function $f^{(k)}=\{f^{(k)}_1,\dots,f^{(k})_H\}\in\mathcal F$. Let us denote $\psi^{f^{(k)}}_{h+1}(s):=\psi^f_{f_{h+1}^{(k)}}(s)$ such that $\psi^{f^{(k)}}_{h+1}(s_{h+1}):=f^{(k)}_{h+1}(s_{h+1},\pi^{(k)}_{h+1}(s_{h+1}))$, and we assume $\|f\|_\infty,\|\psi^f\|_\infty\le H$ (this is the boundedness assumption used throughout). We consider the filtration induced as
$$\mathcal H^{(k)}_h=\{s^i_1,a^i_1,r^i_1,\dots,s^i_H\}_{i=1}^{k-1}\bigcup \{s^k_1, a^k_1, r^k_1, \dots, s^k_h, a^k_h\}$$ 
as the filtration containing the history up to the episode $k$ at step $h$ and $\mathcal H^{(k)}_h$ is sampled by following $\pi^{(k)}$ in the $k^{th}$ episode.

We obtain $\underline{g}_{f^{(k)}}:=(\underline{g}_{\eta,f^{(k)}},\underline{g}_{\nu, f^{(k)}})$ such that $\underline{g}_{\eta,f^{(k)}}$, $\underline{g}_{\nu,f^{(k)}} \in [0,2H/\sigma]$ as a measurable minimizer of eq. \ref{eq:emp_loss_dual} that satisfies Assumption \ref{ass:approx_dual_realizability}. For the trajectory of episode $k$, we define
\begin{align} l_{\phi}\Big(\psi^{f^{(k)}}_{h+1};s^k_h, a^k_h, \pi^{(k)}_{h+1};\underline{g}_{f^{(k)}_{h+1}}\Big)&:=
    \underline{g}_{\eta,f^{(k)}_{h+1}}(s^k_h, a^k_h)\sigma-\underline{g}_{\nu,f^{(k)}_{h+1}}(s^k_h, a^k_h)\nonumber\\
    &\quad  + \underline{g}_{\eta,f^{(k)}_{h+1}}(s^k_h, a^k_h)\,\phi^\star\!\left(-\frac{\psi^{f^{(k)}}_{h+1}(s^k_h,a^k_h)+\underline{g}_{\nu,f^{(k)}_{h+1}}(s^k_h, a^k_h)}{\underline{g}_{\eta,f^{(k)}_{h+1}}(s^k_h, a^k_h)}\right), \label{eq:I_step3}
\end{align}
such that $\abs{l_{\phi}\Big(\psi^{f^{(k)}}_{h+1};s^k_h, a^k_h, \pi^{(k)}_{h+1};\underline{g}_{f^{(k)}_{h+1}}\Big)}\leq B_{\phi}(\sigma)$, where $C_1>0$ is an absolute constant, and 
\begin{align}
\label{eq:I_step4}
    \mathbb{E}\bigg [l_{\phi}\Big(\psi^{f^{(k)}}_{h+1};s^k_h, a^k_h, \pi^{(k)}_{h+1};\underline{g}_{f^{(k)}_{h+1}}\Big)\Big|\mathcal H^{(k)}_h\bigg] = \left[\mathcal{T}^{\phi,\sigma}_{h,\underline{g}_{f^{(k)}_{h+1}}}f^{(k)}_{h+1}\right](s^k_h,a^k_h) - r^{(k)}_h(s^k_h,a^k_h).
\end{align}
For each episode $k$ and step $h$, we define the martingale difference as
\begin{align}
\label{eq:I_step5}
X^{(k)}_h(f,\underline{g}_f)&:=\bigg(f_h^{(k)}(s^k_h,a^k_h)-r^{(k)}_h(s^k_h,a^k_h) - l_{\phi}\Big(\psi^{f^{(k)}}_{h+1};s^k_h, a^k_h, \pi^{(k)}_{h+1};\underline{g}_{f^{(k)}_{h+1}}\Big)\bigg)^2 \nonumber\\
&-\bigg(\left[\mathcal{T}^{\phi,\sigma}_{h,\underline{g}_{f^{(k)}_{h+1}}}f^{(k)}_{h+1}\right](s^k_h,a^k_h)-r^{(k)}_h(s^k_h,a^k_h) + l_{\phi}\Big(\psi^{f^{(k)}}_{h+1};s^k_h, a^k_h, \pi^{(k)}_{h+1};\underline{g}_{f^{(k)}_{h+1}}\Big)\bigg)^2,
\end{align}
such that we have $\abs{X^{(k)}_h(f,\underline{g}_f)}\leq C_2B_{\phi}(\sigma)^2$, where $C_2>0$ is an absolute constant. Moreover, 
\begin{align}
\label{eq:I_step6}
    \mathbb{E}\bigg [X^{(k)}_h(f,\underline{g}_f)\Big|\mathcal H^{(k)}_h\bigg] &= \Bigg(f^{(k)}_h(s,a)-\left[\mathcal{T}^{\phi,\sigma}_{h,\underline{g}_{f^{(k)}_{h+1}}}f^{(k)}_{h+1}\right](s,a)\Bigg)^2\nonumber\\
    \mathrm{Var}\bigg [X^{(k)}_h(f,\underline{g}_f)\Big|\mathcal H^{(k)}_h\bigg]&\leq C_3B_{\phi}(\sigma)^2\mathbb{E}\bigg [X^{(k)}_h(f,\underline{g}_f)\Big|\mathcal H^{(k)}_h\bigg],
\end{align}
where $C_2, C_3>0$ are absolute constants.

Therefore, by Freedman's inequality as given Lemma \ref{lem:Freedman}, we can write
\begin{align}
\label{eq:I_step7}
\abs{\sum_{k=1}^K\bigg(X^{(k)}_h(f,\underline{g}_f)-\mathbb{E}\Big [X^{(k)}_h(f,\underline{g}_f)\Big]\bigg)\big|\mathcal H^{(k)}_h}\leq \mathcal{O}
\bigg(\sqrt{\log(1/\delta)\sum_{k=1}^K \mathbb{E}\Big [X^{(k)}_h(f,\underline{g}_f)\Big|\mathcal H^{(k)}_h\Big]}+\log(1/\delta)\bigg).
\end{align}

Now, let us consider $\mathcal{X}_{\rho}$ be the $\rho$-cover of $\mathcal{F}\bigcup\mathcal{G}$. Now taking a union bound for all $(k,h,\Xi) \in [K] \times [H] \times \mathcal{X}_{\rho}$, and following the same proof-lines as in \cite{NeurIPS2021_BellmanEluderDim_Jin}[Lemma 39], we get
\begin{align}
\label{eq:I_step8}
       \displaystyle \sum_{t<k}\mathbb E\Bigg[\Bigg(f^{(k)}_h(s,a)-\left[\mathcal{T}^{\phi,\sigma}_{h,\underline{g}_{f^{(k)}_{h+1}}}f^{(k)}_{h+1}\right](s,a)\Bigg)^2\Bigg| \mathcal H^{(t)}_h\Bigg] \le \mathcal O(\beta),
\end{align}
where $\beta = \mathcal{O}\bigg(B_{\phi}(\sigma)\log(\mathcal N_{F\cup G}(\rho)\cdot KH/\delta)\Big) \bigg).$ Now, we set $\rho=1/K$. In addition, as $\mathcal F\cup\mathcal G$ is finite, then $\log\mathcal N_{F\cup G}(1/K)\le \log|\mathcal F\cup \mathcal G|$, therefore, we consider $\beta$ as $\beta = \mathcal{O}\bigg(B_{\phi}(\sigma)\log(|\mathcal F|\mathcal G|\cdot KH/\delta)\Big) \bigg).$

Therefore, eq. \ref{eq:I_step8} concludes that $\sum_{t<k}\mathbb E_{(s,a)\sim \pi^t}\Bigg[\Bigg(f^{(k)}_h(s,a)-\left[\mathcal{T}^{\phi,\sigma}_{h,\underline{g}_{f^{(k)}_{h+1}}}f^{(k)}_{h+1}\right](s,a)\Bigg)^2\Bigg]\;\le\; \mathcal O(\beta).$ 

\item[\ding{43}] {\it Proof of ineq. (a)} To show ineq. (a), we will focus on the proof-lines of \cite{NeurIPS2021_BellmanEluderDim_Jin}[Lemma 40] and \cite{xie2022role}[Lemma 15 (1)]. Fix $(k,h,f)$ and follow the same notation as mentioned in the proof lines of the inequality (b), we define
\begin{align*}
W^{(t)}_h(f,\underline{g}_f)&:=\bigg(f_h^{(t)}(s^t_h,a^t_h)-r^{(t)}_h(s^t_h,a^t_h) - l_{\phi}\Big(\psi^{f^{(k)}}_{h+1};s^k_h, a^k_h, \pi^{(k)}_{h+1};\underline{g}_{f^{(k)}_{h+1}}\Big)\bigg)^2 \nonumber\\
&-\bigg(Q_h^{\star,\sigma}(s^t_h,a^t_h)-r^{(t)}_h(s^t_h,a^t_h) + l_{\phi}\Big(\psi^{f^{(k)}}_{h+1};s^k_h, a^k_h, \pi^{(k)}_{h+1};\underline{g}_{f^{(k)}_{h+1}}\Big)\bigg)^2, \quad \text{ for $1\leq t\leq k$}.
\end{align*}
As in eq. \ref{eq:I_step6}, $\mathbb E\Big[W^{(t)}_h(f,\underline{g}_f)\mid\mathcal{H}^{(t)}_h\Big]=\Big(f_h^{(t)}(s^t_h,a^t_h) -Q_h^{\star,\sigma}(s^t_h,a^t_h)\Big)^2$ where $\mathcal{H}^{(t)}_h$ be the filtration induced by $\{s^i_1,a^i_1,r^i_1,\dots,s^i_H\}_{i=1}^{t-1}\bigcup \{s^t_1,a^t_1,r^t_1,\dots,s^t_h,a^t_h\}$. Similarly, we can verify that  $|W^{(t)}_h(f,\underline{g}_f)|\le C_4\Big(B_{\phi}(\sigma)(\sigma)\Big)^2$ and $\mathrm{Var}\Big[W^{(t)}_h(f,\underline{g}_f)\mid\mathcal{H}^{(t)}_h\Big]\leq C_5\Big(B_{\phi}(\sigma)(\sigma)\Big)^2E\Big[W^{(t)}_h(f,\underline{g}_f)\mid\mathcal{H}^{(t)}_h\Big]$. Now, following the proof-lines of \cite{NeurIPS2021_BellmanEluderDim_Jin}[Lemma 40], and applying Freedman's ineq. (Lemma \ref{lem:Freedman} and a cover of $\mathcal G$ yields, w.p. $1-\delta$, we get
\begin{align*}
&\sum_{t=1}^{k-1} 
\Big[ Q_h^{\star,\sigma}(s_h^t, a_h^t) - r_h^t(s_h^t, a_h^t) - Q_{h+1}^{\star,\sigma}(s_{h+1}^t, \pi^{Q^{\star,\sigma}}_{h+1}(s_{h+1}^t)) \Big]^2 \\[6pt]
&\leq \sum_{t=1}^{k-1} 
\Big[ f^{(t)}_h(s_h^t, a_h^t) - r_h^t(s_h^t, a_h^t) - Q_{h+1}^{\star,\sigma}(s_{h+1}^t, \pi^{Q^{\star,\sigma}}_{h+1}(s_{h+1}^t)) \Big]^2 
+ \mathcal{O}(\beta).
\end{align*}
Finally, by recalling the definition of $\mathcal{F}^{(k)}$, we conclude that with probability at least 
$1-\delta$, $Q^{\star,\sigma} \in \mathcal{F}^{(k)}$ for all $k \in [K]$. 
\end{itemize}
This concludes the proof of Lemma \ref{lem:robust-lemma39}.\qedhere
\end{proof}

\begin{keylem}[Robust Bellman--Eluder Dimension Bound \cite{NeurIPS2021_BellmanEluderDim_Jin}]
\label{lem:robust_lemma17}
Fix $h\in[H]$ and let $\Xi_h^{\xi}=(\mathcal I-\mathcal T_h^{\phi,\sigma})\mathcal F$.
Let $\Pi$ be the family of distributions under $P^\star$.
Assume $\sup_{\xi\in\Xi_h^{\xi}}\|\xi\|_\infty \le C_{\phi}$.
Let $\{(f^{(k)},\pi^{(k)})\}_{k=1}^K$ be the sequence produced by the algorithm and define
\[
\xi_k := f_h^{(k)}-(\mathcal T_h^{\phi,\sigma} f_{h+1}^{(k)})\in\Xi_h^{\xi},
\qquad
\mu := \{\mu_k\}_{k=1}^K \in\Pi.
\]
Suppose there exists $\beta>0$ such that for all $k\in[K]$,
\[
\sum_{t=1}^{k-1}\mathbb E_{\mu_t}\big[\xi_t^2\big]\le \beta.
\]
Then for all $k\in[K]$,
\[
\sum_{t=1}^{k}\big|\mathbb E_{\mu_t}[\xi_t]\big|
\;\le\;
\mathcal O\!\Big(\sqrt{\dim_{\mathrm{DE}}(\Xi_h^{\xi},\Pi_h,1/k)\,\beta\,k}\Big).
\]
\end{keylem}

\begin{proof}
The proof is the same as in \citep{NeurIPS2021_BellmanEluderDim_Jin}[Lemma 41]. Fix $h\in[H]$ and let $\Xi=\Xi_h^{\xi}$ and $\Pi=\Pi_h$.
Assume $|\xi(x)|\le C_{\phi}$ for all $(\xi,x)\in\Xi\times\mathcal X$.
Suppose a sequence $\{\xi_k\}_{k=1}^K\subseteq\Xi$ and $\{\mu_k\}_{k=1}^K\subseteq\Pi$
satisfies for all $k\in[K]$,
\[
\sum_{t=1}^{k-1}\big(\mathbb E_{\mu_t}[\xi_t]\big)^2 \le \beta.
\]

Let $e_t:=|\mathbb E_{\mu_t}[\xi_t]|$ and sort $(e_1,\dots,e_K)$ in non-increasing order:
$e_{(1)}\ge e_{(2)}\ge \cdots \ge e_{(K)}$.
Then
\[
\sum_{t=1}^K e_t
=\sum_{t=1}^K e_{(t)}
\le \sum_{t=1}^K e_{(t)}\mathbf 1\{e_{(t)}\le \omega\}
+\sum_{t=1}^K e_{(t)}\mathbf 1\{e_{(t)}>\omega\}
\le K\omega + \sum_{t=1}^K e_{(t)}\mathbf 1\{e_{(t)}>\omega\}.
\]

Let $d:=\dim_{\mathrm{DE}}(\Xi,\Pi,\omega)$.
Fix any index $t$ with $e_{(t)}>\omega$.
Then there exists $\alpha$ such that $e_{(t)}>\alpha\ge \omega$.
By Proposition~\ref{prop:robust_prop43} applied at threshold $\alpha$,
\[
t \;\le\; \sum_{i=1}^K \mathbf 1\{e_i>\alpha\}
\;\le\; \Big(\frac{\beta}{\alpha^2}+1\Big)\dim_{\mathrm{DE}}(\Xi,\Pi,\alpha)
\;\le\; \Big(\frac{\beta}{\alpha^2}+1\Big)d,
\]
where the last inequality uses that $\dim_{\mathrm{DE}}(\Xi,\Pi,\alpha)$ is non-increasing in $\alpha$
and $\alpha\ge\omega$.
Rearranging gives $\alpha^2 \le d\beta/(t-d)$ (for $t>d$), hence
\[
e_{(t)} \le \min\Big\{C_{\phi},\sqrt{\frac{d\beta}{t-d}}\Big\}.
\]
Therefore,
\[
\sum_{t=1}^K e_{(t)}\mathbf 1\{e_{(t)}>\omega\}
\le \min\{d,k\}C_{\phi} + \sum_{t=d+1}^K \sqrt{\frac{d\beta}{t-d}}
\le \min\{d,k\}C_{\phi} + \sqrt{d\beta}\int_0^K \frac{1}{\sqrt{t}}\,dt
\le \min\{d,k\}C_{\phi} + 2\sqrt{d\beta k}.
\]
Combining with the earlier decomposition yields
\begin{align}
\label{eq:robust_lemma41}
\sum_{t=1}^K |\mathbb E_{\mu_t}[\xi_t]|
\le K\omega + \min\{d,K\}C_{\phi} + 2\sqrt{d\beta K},
\end{align}
which is the claimed bound up to absolute constants.

Let $\xi_k$ and $\mu_k$ be defined as:
\[
\xi_k := f_h^{(k)}-(\mathcal T_h^{\phi,\sigma} f_{h+1}^{(k)})\in\Xi_h^{\xi_k},
\qquad
\mu_k \in\Pi.
\]
We set $\Xi=\Xi_h^{\xi}$, and  $\Pi=\Pi_h$.
By Jensen's inequality, for each $k$ and each $t<k$,
\[
\big(\mathbb E_{\mu_t}[\xi_k]\big)^2 \le \mathbb E_{\mu_t}[\xi_k^2].
\]
Hence, the assumed condition $\sum_{t<k}\mathbb E_{\mu_t}[\xi_k^2]\le \beta$ implies
$\sum_{t<k}(\mathbb E_{\mu_t}[\xi_k])^2\le \beta$ for all $k$.
Applying eq. ~\ref{eq:robust_lemma41} with $\omega=1/k$ then yields
\[
\sum_{i=1}^{k}|\mathbb E_{\mu_i}[\xi_i]|
\le
\mathcal O\!\Big(\sqrt{\dim_{\mathrm{DE}}(\Xi,\Pi,1/k)\,\beta\,k}
+\min\{k,\dim_{\mathrm{DE}}(\Xi,\Pi,1/k)\}C_{\phi} + 1\Big).
\]
Dropping lower-order terms gives the claimed bound. The key inequality eq. \ref{eq:robust_lemma41}
itself follows from Proposition~\ref{prop:robust_prop43}, whose proof is a direct adaptation of \citep{NeurIPS2021_BellmanEluderDim_Jin}[Proposition 43]. This completes the proof.
\end{proof}

\begin{keylem}[$\phi$-Dual optimization error bound (Lemma \ref{lem:equiv_loss_dual})]
\label{lem:dual_opt_error_bound_BE}
Fix $h\in[H]$ and a policy $\pi$. Let $\mu_h^\pi$ denote the step-$h$ state--action distribution induced by $\pi$ under the nominal kernel $P^\star_h$, and let $\mathcal{D}_h$ be a dataset of transitions $(s,a,s')$ collected by executing $\pi$ at step $h$.
For any $f_{h+1}\in\mathcal{F}_{h+1}$, let $\underline g_{f_{h+1}}$ denote the dual parameter obtained from the empirical optimization in eq.~\ref{eq:emp_loss_dual} for a given state--action value function $f$ as given in eq. \ref{eq:hatg}, and let $\mathcal{T}^{\phi, \sigma}_g$ be as defined in eq.~\ref{eq:revised_dual_TV_g}. 
Then for any $\delta\in(0,1)$, with probability at least $1-\delta$,
\begin{align}
\label{eq:dual_opt_err_BE}
\sup_{f_{h+1}\in\mathcal{F}_{h+1}}
\big\|\mathcal{T}^{\phi, \sigma}_h f_{h+1} - \mathcal{T}^{\phi, \sigma}_{h,\underline g_{f_{h+1}}} f_{h+1}\big\|_{1,\mu_h^\pi}
\;=\;
\mathcal{O}\!\left(
B_{\phi}(\sigma)\sqrt{\frac{\log\!\big(|\mathcal{F}_{h+1}||\mathcal{G}|/\delta\big)}{|\mathcal{D}_h|}}
+\varepsilon^{\mathrm{dual}}.
\right).
\end{align}
\end{keylem}

\begin{proof}
For a fixed $h\in [H]$, fix an arbitrary $f\in\mathcal{F}$ and recall that $\underline{g}_f$ as defined in eq. \ref{eq:emp_loss_dual}, where $\widehat{\mathrm{DualLoss}}$ is given in eq.~\ref{eq:emp_loss_dual}. For notational convenience, define the dual objective
\[
\zeta_{f}(g)
:= \mathbb{E}_{(s,a)\sim\mu_h^\pi,\, s'\sim P^\star_h(\cdot|s,a)}
\Big[l_{\phi}(f;s,a,s';g)\Big].
\]
where $l_{\phi}(f;s,a,s';g):=g_{\eta}(s,a)\sigma-g_{\nu}(s,a) + g_{\eta}(s,a)\,\phi^\star\!\left(-\frac{\max_{a'}f(s',a')+g_{\nu}(s,a)}{g_{\eta}(s,a)}\right)$ by eq. \ref{eq:pointiwse_l}.

Using the dual representation in eq.~\ref{eq:revised_dual_Q_TV}, the difference between the true robust Bellman operator and its empirical counterpart can be written as
\begin{align}
\label{eq:dual-bound-a}
\big\|\mathcal{T}^{\phi,\sigma}f - \mathcal{T}^{\phi,\sigma}_{\underline{g}_f}f\big\|_{1,\mu^{\pi}}
&= \zeta_f\big(\underline{g}_f\big)
  -
  \mathbb{E}_{(s,a)\sim\mu^{\pi}}\bigg[\inf_{\eta \geq 0, \nu \in \mathbb R} l_{\phi}\big(f;s,a,s';\eta,\nu\big)
\Bigg].
\end{align}

Next, we use the functional reformulation, which (by the interchange rule for integral functionals \cite{rockafellar1998variational}[Theorem~14.60]) (as given in Lemma \ref{lem:rockafellar_Thm14.60}) states that
\[
\mathbb{E}_{(s,a)\sim\mu^{\pi}}\bigg[\inf_{\eta \geq 0, \nu \in \mathbb R} l_{\phi}\big(f;s,a,s';\eta,\nu\big)
\Bigg]
=
\inf_{g\in\mathcal{L}^1(\mu^{\pi};\mathbb R^2)} \zeta_f(g).
\]
Substituting this into eq. \ref{eq:dual-bound-a} gives
\begin{align*}
\big\|\mathcal{T}^{\phi,\sigma}f - \mathcal{T}^{\phi,\sigma}_{\underline{g}_f}f\big\|_{1,\mu^{\pi}}
&=
\zeta_f(\underline{g}_f)
-
\inf_{g\in\mathcal{L}^1(\mu^{\pi};\mathbb R^2)} \zeta_f(g)\\
&=
\big[\zeta_f(\underline{g}_f) - \inf_{g\in\mathcal{G}} \zeta_f(g)\big]
+
\big[\inf_{g\in\mathcal{G}} \zeta_f(g)
    - \inf_{g\in\mathcal{L}^1(\mu^{\pi};\mathbb R^2)} \zeta_f(g)\big].
\end{align*}
The second bracket is controlled by the approximate dual realizability assumption (Assumption \ref{ass:approx_dual_realizability}), which gives
\[
\inf_{g\in\mathcal{G}} \zeta_f(g)
-
\inf_{g\in\mathcal{L}^1(\mu^{\pi};\mathbb R^2)} \zeta_f(g)
\;\le\;
\varepsilon^{\mathrm{dual}}.
\]
Hence,
\begin{align}
\label{eq:dual-bound-decomposed}
\big\|\mathcal{T}^{\phi,\sigma}f - \mathcal{T}^{\phi,\sigma}_{\underline{g}_f}f\big\|_{1,\mu^{\pi}}
\;\le\;
\zeta_f(\underline{g}_f) - \inf_{g\in\mathcal{G}} \zeta_f(g)
+ \varepsilon^{\mathrm{dual}}.
\end{align}

We now bound the optimization error term $\zeta_f(\underline{g}_f) - \inf_{g\in\mathcal{G}} \zeta_f(g)$. Consider the loss function as given in eq. \ref{eq:pointiwse_l} as
\[
l_{\phi}(f;s,a,s';g)
:= 
g_{\eta}(s,a)\sigma-g_{\nu}(s,a) + g_{\eta}(s,a)\,\phi^\star\!\left(-\frac{\max_{a'}f(s',a')+g_{\nu}(s,a)}{g_{\eta}(s,a)}\right),
\]
so that $\zeta_f(g) = \mathbb{E}_{(s,a,s')}\big[l_{\phi}(f;s,a,s';g)\big]$ and 
$\widehat{\mathrm{DualLoss}}(g; f)$ in eq.~\ref{eq:emp_loss_dual} is the empirical average of $\ell_{\phi}$ over $\mathcal{D}$. Since $f\in\mathcal{F}$ and $g_{\eta},g_{\nu}\in\mathcal{G}$ take values in $[0,H]$ and $\Theta_{\phi}$, respectively, and by \Cref{ass:multiplier-range-TV} we have $|l_{\phi}(f;s,a,s';g)|\le B_{\phi}(\sigma)$.

By applying the empirical risk minimization generalization bound (\cite{NeurIPS2022_RobustRLOffline_Panaganti}[Lemma 3]) together with the Lipschitz-based bound in eq.~\ref{eq:ERM_bound_lipschitz} of Lemma \ref{lem:ERM_gen_bound}, we obtain that, with probability at least $1-\delta$,
\begin{align}
\label{eq:ERM-dual-bound}
\zeta_f(\underline{g}_f) - \inf_{g\in\mathcal{G}} \zeta_f(g)
\;\le\;
C_6B_{\phi}(\sigma)\sqrt{\frac{2\log|\mathcal{G}|}{|\mathcal{D}|}}
+
C_7B_{\phi}(\sigma)\sqrt{\frac{2\log(8/\delta)}{|\mathcal{D}|}},
\end{align}
where $C_6$ and $C_7$ are absolute constants. Combining eq. \ref{eq:dual-bound-decomposed} and eq. \ref{eq:ERM-dual-bound}, and then taking a union bound over $f\in\mathcal{F}$, we conclude that, with probability at least $1-\delta$,
\begin{align*}
    \sup_{f\in \mathcal{F}} \big\|\mathcal{T}^{\phi,\sigma}f - \mathcal{T}^{\phi,\sigma}_{\underline{g}_f}f\big\|_{1,\mu^\pi} &\le
    C\,B_{\phi}(\sigma)\sqrt{\frac{2\log\big(8|\mathcal{G}||\mathcal{F}|/\delta\big)}{|\mathcal{D}|}}
    + \varepsilon^{\mathrm{dual}},
\end{align*}
for some absolute constant $C>0$, which proves the claimed big-$\mathcal{O}$ bound. \qedhere
\end{proof}

\subsection{Technical Lemmas}
\label{subsec:Technical_Lemmas}
We now state a result for the generalization bounds on empirical risk minimization (ERM) problems.
This result is adapted from \cite{shalev2014understanding}[Theorem 26.5, Lemma 26.8, Lemma 26.9].

\begin{techlem}[ERM generalization bound \cite{NeurIPS2022_RobustRLOffline_Panaganti}, Lemma 3]
\label{lem:ERM_gen_bound}
Let $P$ be a distribution on $\mathcal{X}$ and let $\mathcal{H}$ be a hypothesis class of real-valued functions on $\mathcal{X}$. Assume the loss $\mathrm{Loss}:\mathcal{H}\times\mathcal{X}\to\mathbb{R}$ satisfies
\[
  |\mathrm{Loss}(h,x)| \le c_0, \quad \forall\, h\in\mathcal{H},\; x\in\mathcal{X}, \quad \text{ for some constant $c_0>0$.}
\]
Given an i.i.d.\ sample $\mathcal{D} = \{X_i\}_{i=1}^N$ from $P$, define the empirical risk minimizer $\widetilde{h} \in \arg\min_{h \in \mathcal{H}} \frac{1}{N}\sum_{i=1}^N \mathrm{Loss}(h,X_i)$. For any $\delta \in (0,1)$ and any population risk minimizer $h^\star \in \arg\min_{h \in \mathcal{H}} \mathbb{E}_{X\sim P}[\mathrm{Loss}(h,X)]$, the following holds with probability at least $1-\delta$:
\begin{align}
\label{eq:ERM_bound}
\mathbb{E}_{X\sim P}[\mathrm{Loss}(\widetilde{h},X)] - \mathbb{E}_{X\sim P}[\mathrm{Loss}(h^\star,X)]
\leq 2R(\mathrm{Loss}\circ \mathcal{H}\circ \mathcal{D}) + 5c_0 \sqrt{\frac{2\log(8/\delta)}{N}},
\end{align}
where $R(loss\circ \mathcal{H}\circ \mathcal{D})$ is the empirical Rademacher complexity of the loss-composed class $loss\circ\mathcal{H}$, defined by
\[
R(\mathrm{Loss}\circ \mathcal{H}\circ \mathcal{D})
= \frac{1}{N}\mathbb{E}_{\{\sigma_i\}_{i=1}^N}
\left[ \sup_{g \in \mathrm{Loss}\circ \mathcal{H}} \sum_{i=1}^N \sigma_i g(X_i)\right],
\]
with $\{\sigma_i\}_{i=1}^N$ independent of $\{X_i\}_{i=1}^N$ and i.i.d.\ according to a Rademacher random variable $\sigma$ (i.e., $\mathbb{P}(\sigma=1)=\mathbb{P}(\sigma=-1)=0.5$). Moreover, if $\mathcal{H}$ is finite, $|\mathcal{H}|<\infty$, and there exist constants $c_1,c_2>0$ such that
\[
|h(x)| \le c_0 \quad \forall\, h\in\mathcal{H},\, x\in\mathcal{X}, 
\qquad
\text{and}\qquad
\mathrm{Loss}(h,x)\ \text{is $c_1$-Lipschitz in $h$},
\]
then with probability at least $1-\delta$ we further have
\begin{align}
\label{eq:ERM_bound_lipschitz}
\mathbb{E}_{X\sim P}[\mathrm{Loss}(\widetilde{h},X)] - \mathbb{E}_{X\sim P}[\mathrm{Loss}(h^\star,X)]
\leq 2c_1c_2 \sqrt{\frac{2\log(|\mathcal{H}|)}{N}}
+ 5c_0 \sqrt{\frac{2\log(8/\delta)}{N}}.
\end{align}
\end{techlem}

We now mention two important concepts from variational analysis \cite{rockafellar1998variational} literature that is useful to relate minimization of integrals and the integrals of pointwise minimization under special class of functions.

\begin{defi}[Decomposable spaces and Normal integrands \cite{rockafellar1998variational}(Definition 14.59, Example 14.29)]
\label{def:decomposable_sp}
    A space $\mathcal{X}$ of measurable functions is a decomposable space relative to an underlying measure space $(\Omega,\mathcal{A},\mu)$, if for every function $x_0 \in \mathcal{X}$, every set $A \in \mathcal{A}$ with $\mu(A)<\infty$, and any bounded measurable function $x_1: A \to \mathbb{R}$, the function 
\[
x(\omega) = x_0(\omega)\mathbf{1}(\omega \notin A) + x_1(\omega)\mathbf{1}(\omega \in A)
\]
belongs to $\mathcal{X}$. A function $f: \Omega \times \mathbb{R} \to \mathbb{R}$ (finite-valued) is a normal integrand, if and only if $f(\omega,x)$ is $\mathcal{A}$-measurable in $\omega$ for each $x$ and is continuous in $x$ for each $\omega$.
\end{defi}

\begin{rem}
\label{rem:examples_decomposable_sp}
A few examples of decomposable spaces are $\mathcal{L}^p(\mathcal{S}\times \mathcal{A}, \Sigma(\mathcal{S}\times \mathcal{A}), \mu)$ for any $p \geq 1$ and $\mathcal{M}(\mathcal{S}\times \mathcal{A}, \Sigma(\mathcal{S}\times \mathcal{A}))$, the space of all $\Sigma(\mathcal{S}\times \mathcal{A})$-measurable functions.
\end{rem}

\begin{techlem}[\cite{rockafellar1998variational}, Theorem 14.60]
\label{lem:rockafellar_Thm14.60}
    Let $\mathcal{X}$ be a space of measurable functions from $\Omega$ to $\mathbb{R}$ that is decomposable relative to a $\sigma$-finite measure $\mu$ on the $\sigma$-algebra $\mathcal{A}$. Let $f:\Omega \times \mathbb{R} \to \mathbb{R}$ (finite-valued) be a normal integrand. Then, we have
\[
\inf_{x \in \mathcal{X}} \int_{\omega \in \Omega} f(\omega, x(\omega)) \mu(d\omega) 
= \int_{\omega \in \Omega} \left(\inf_{x \in \mathcal{X}} f(\omega, x)\right)\mu(d\omega).
\]

Moreover, as long as the above infimum is not $-\infty$, we have that
\[
x' \in \arg\min_{x \in \mathcal{X}} \int_{\omega \in \Omega} f(\omega, x(\omega)) \mu(d\omega),
\]
if and only if $x'(\omega) \in \arg\min_{x \in \mathbb{R}} f(\omega, x)\mu$ almost surely.
\end{techlem}

\begin{techlem}[Freedman’s inequality (e.g., \cite{agarwal2014taming})]
\label{lem:Freedman}
Let $\{M_t\}_{t\leq T}$ be a real-valued martingale difference sequence w.r.t.\ filtration $\{\mathcal G_t\}$ with $|M_t|\le b$ a.s.\ and let $S_T=\sum_{t=1}^T \mathbb E[M_t^2\mid\mathcal G_{t-1}]$. Then for any $\delta\in(0,1)$,
\[
\Pr\Big(\sum_{t=1}^T M_t \ge \sqrt{2 S_T \ln(1/\delta)} + \tfrac{b}{3}\ln(1/\delta)\Big)\le \delta.
\]
\end{techlem}

\begin{techlem}[Robust DE counting bound ($\phi$-residual class)]
\label{prop:robust_prop43}
Fix $h\in[H]$. Let $\Xi=\Xi_h^{\xi}$ be a function class on $\mathcal X:=\mathcal S\times\mathcal A$
and let $\Pi=\Pi_h$ be a family of probability measures on $\mathcal X$.
Suppose a sequence $\{\xi_k\}_{k=1}^K\subseteq \Xi$ and $\{\mu_k\}_{k=1}^K\subseteq \Pi$
satisfies that for all $k\in[K]$, $\sum_{t=1}^{k-1}\Big(\mathbb E_{\mu_t}[\xi_t]\Big)^2 \le \beta.$ Then for any $\varepsilon>0$ and any $k\in[K]$,
\begin{align}
\label{eq:robust_prop43_count}
\sum_{t=1}^{k}\mathbf 1\big\{\,|\mathbb E_{\mu_t}[\xi_t]|>\varepsilon\,\big\}
\;\le\;
\Big(\frac{\beta}{\varepsilon^2}+1\Big)\,\dim_{\mathrm{DE}}(\Xi,\Pi,\varepsilon).
\end{align}
\end{techlem}

\begin{proof}
The proof follows the same argument as in \citep{NeurIPS2021_BellmanEluderDim_Jin}[Proposition 43].
We include it for completeness.

Fix $\varepsilon>0$.

\textbf{Step 1: Disjoint dependence.}
Fix any $k$ such that $|\mathbb E_{\mu_k}[\xi_k]|>\varepsilon$.
By the definition of DE dimension in Definition \ref{def:DE_dim}, if $\mu_k$ is $\varepsilon$-dependent on a subsequence
$\{\nu_1,\dots,\nu_\ell\}$ of $\{\mu_1,\dots,\mu_{k-1}\}$ (with respect to $\Xi$),
then there exists some $\xi\in\Xi$ (in particular we may take $\xi=\xi_k$) such that
\[
\sum_{t=1}^{\ell}\Big(\mathbb E_{\nu_t}[\xi_t]\Big)^2 \;\ge\; \varepsilon^2.
\]
Therefore, if $\mu_k$ is $\varepsilon$-dependent on $M$ \emph{disjoint} subsequences of $\{\mu_1,\dots,\mu_{k-1}\}$,
then summing the above inequality over these $M$ disjoint subsequences yields
\[
\sum_{t=1}^{k-1}\Big(\mathbb E_{\mu_t}[\xi_k]\Big)^2 \;\ge\; M\,\varepsilon^2.
\]
Combining with $\sum_{t=1}^{k-1}\Big(\mathbb E_{\mu_t}[\xi_k]\Big)^2 \le \beta$ gives $M\le \beta/\varepsilon^2$.
Hence, for any $k$ with $|\mathbb E_{\mu_k}[\xi_k]|>\varepsilon$, the measure $\mu_k$ can be $\varepsilon$-dependent on
\emph{at most} $\beta/\varepsilon^2$ disjoint subsequences of $\{\mu_1,\dots,\mu_{k-1}\}$.

\textbf{Step 2: Pigeonhole argument using $\dim_{\mathrm{DE}}$.}
Now consider any sequence $\{\nu_1,\dots,\nu_\kappa\}\subseteq \Pi$.
We show that there exists some index $j\in[\kappa]$ such that $\nu_j$ is $\varepsilon$-dependent on at least
\[
M := \Big\lceil \frac{\kappa-1}{\dim_{\mathrm{DE}}(\Xi,\Pi,\varepsilon)} \Big\rceil
\]
disjoint subsequences of $\{\nu_1,\dots,\nu_{j-1}\}$. To see this, run the following procedure:
\begin{itemize}
\item Initialize $M$ disjoint blocks $B_1=\{\nu_1\},\dots,B_M=\{\nu_M\}$ and set $j=M+1$.
\item For each $j$, if $\nu_j$ is $\varepsilon$-dependent on all blocks $B_1,\dots,B_M$, we stop.
\item Otherwise, pick a block $B_i$ such that $\nu_j$ is $\varepsilon$-independent of $B_i$ and update $B_i\leftarrow B_i\cup\{\nu_j\}$.
\item Increase $j$ and continue.
\end{itemize}

By the definition of $\dim_{\mathrm{DE}}(\Xi,\Pi,\varepsilon)$,
each block $B_i$ can be extended by $\varepsilon$-independent insertions at most
$\dim_{\mathrm{DE}}(\Xi,\Pi,\varepsilon)$ times.
Therefore, the procedure must stop by time $j\le M\cdot \dim_{\mathrm{DE}}(\Xi,\Pi,\varepsilon)+1\le \kappa$,
which implies the claimed existence of such a $j$.

\textbf{Step 3: Combine Steps 1--2.}
Fix any $k\in[K]$ and let $\{\nu_1,\dots,\nu_\kappa\}$ be the subsequence of $\{\mu_1,\dots,\mu_k\}$
consisting of those indices $t\le k$ for which $|\mathbb E_{\mu_t}[\xi_t]|>\varepsilon$.
Applying Step 2 to this subsequence implies there exists some $j$ such that $\nu_j$ is $\varepsilon$-dependent
on at least $\lceil(\kappa-1)/\dim_{\mathrm{DE}}(\Xi,\Pi,\varepsilon)\rceil$ disjoint subsequences of its predecessors.
By Step 1 (applied to that $j$), the number of such disjoint dependent subsequences is at most $\beta/\varepsilon^2$.
Hence,
\[
\frac{\kappa-1}{\dim_{\mathrm{DE}}(\Xi,\Pi,\varepsilon)} \;\le\; \frac{\beta}{\varepsilon^2},
\]
which yields
\[
\kappa \;\le\; \Big(\frac{\beta}{\varepsilon^2}+1\Big)\,\dim_{\mathrm{DE}}(\Xi,\Pi,\varepsilon).
\]
This is exactly eq. \ref{eq:robust_prop43_count}.
\end{proof}

\bibliographystyle{abbrv}  
\bibliography{ref}

\end{document}